\documentclass[11pt]{article}

\usepackage{float}
\usepackage[all]{xy}

\usepackage{amsmath,amssymb,amsthm}
\usepackage{latexsym}
\usepackage[dvipdfmx]{graphicx}
\usepackage[ruled]{algorithm2e}
\usepackage{algorithmic} 
\usepackage{authblk}
\usepackage{color}
\usepackage{relsize}
\usepackage{url}
\usepackage[colorlinks,linkcolor=blue,anchorcolor=green,citecolor=green]{hyperref}
\usepackage{caption} 
\usepackage{tikz}
\usepackage{yhmath}
\usepackage{mathrsfs,booktabs}

\usepackage[T1]{fontenc}
\usepackage[utf8]{inputenc}
\usepackage[font=small,labelfont=bf,tableposition=top]{caption}
\usepackage{booktabs}
\usepackage{threeparttable} 
\usepackage{multirow}

\makeatletter

\newcommand{\Rmnum}[1]{\expandafter\@slowromancap\romannumeral #1@}
\makeatother

\newtheorem{theorem}{Theorem}

\newtheorem{lemma}{Lemma}

\newtheorem{proposition}{Proposition}
\newtheorem{definition}{Definition}

\newtheorem{assumption}{Assumption}
\usepackage{color}

\usepackage{tikz,pgfplots}
\pgfplotsset{compat=newest}
\pgfplotsset{plot coordinates/math parser=false,trim axis left}
\newlength\figureheight
\newlength\figurewidth

\newcommand{\eins}{\boldsymbol{1}}

\newcommand{\argmax}{\operatornamewithlimits{arg \, max}}
\newcommand{\argmin}{\operatornamewithlimits{arg \, min}}

\usepackage{siunitx}
\usepackage{subfig}

\author[]{Hanyuan Hang}

\date{\today}

\affil[]{Department of Applied Mathematics \\ 
University of Twente, The Netherlands \\
{\tt h.hang@utwente.nl}
}

\begin{document}

\title{Local Adaptivity of Gradient Boosting in Histogram Transform Ensemble Learning}



\maketitle

%

\allowdisplaybreaks

\begin{abstract}In this paper, we propose a gradient boosting algorithm called \textit{adaptive boosting histogram transform} (\textit{ABHT}) for regression to illustrate the local adaptivity of gradient boosting algorithms in histogram transform ensemble learning. From the theoretical perspective, when the target function lies in a locally H\"older continuous space, we show that our ABHT can filter out the regions with different orders of smoothness. Consequently, we are able to prove that the upper bound of the convergence rates of ABHT is strictly smaller than the lower bound of \textit{parallel ensemble histogram transform} (\textit{PEHT}). In the experiments, both synthetic and real-world data experiments empirically validate the theoretical results, which demonstrates the advantageous performance and local adaptivity of our ABHT. 
\end{abstract}


\section{Introduction} \label{sec::Introduction}
Ensemble learning is an important framework that has been explored since 1970s \cite{tukey1977exploratory,dasarathy1979composite} and is still regarded as the state-of-the-art algorithms \cite{hang2021histogram,tian2021rase,cui2021gbht}. The study of ensemble learning was initially motivated by the incompetence and the lack of stability of one single learner encountering complex data. To deal with the problems, researchers raised the idea of combining results from various base learners to form a more powerful one, which could obtain higher accuracy and lower variance. Consequently, ensemble learning attracted great attention and has been utilized on diverse real-world problems with satisfactory performances \cite{fernandez2014we,yu2017hybrid}.

In the meantime, new ensemble-based algorithms spring up due to the flexible structure and mild requirements of the ensemble framework. Generally, according to how the base learners integrate, ensemble-based algorithms can be categorized into two major classes, i.e., sequential ensemble methods and parallel ensemble methods \cite{zhou2019ensemble}.

As the name suggests, the parallel ensembles train the base learners independently and combine them with certain aggregating methods. The base learners of parallel ensemble methods can be generated simultaneously. One representative of this kind is \textit{bagging}, short for \textit{b}ootstrap \textit{agg}regat\textit{ing}, which employs the bootstrap method to obtain different sample sets from the original training data set. Then, each base learner is trained on a corresponding sampled dataset and they are combined to form the final learner by methods like averaging or voting. Take \cite{Breiman96} for instance, the bagging classifier was determined by a plurality voting process of the base classifiers trained on bootstrap replicates of the original dataset and was also proved to be more accurate and show better resistance towards the perturbation of the data. It is worth noticing that different base learners lead to different bagging algorithms. Equipped with decision trees as base learners, the so-called random forest algorithm has been recognized as one of the most successful algorithms for classification and regression, leading to numerous algorithmic studies \cite{breiman2001random, biau2016random, meinshausen2006quantile, vasiloudis2019quantifying}, theoretical studies \cite{biau2008consistency, biau2012analysis, scornet2015consistency, mentch2016quantifying, athey2019generalized,
mentch2020randomization, mourtada2020minimax,gao2020towards}, and real-world applications \cite{pal2005random, diaz2006gene, ishwaran2008random, fanelli2013random, paul2018improved, wang2018hierarchical}. Alternatively, the bagged nearest neighbor algorithms also appeal plenty of attention \cite{hall2005properties, biau2010rate, samworth2012optimal, zhang2019novel}.

On the other hand, the base learners of sequential ensemble methods are generated sequentially. A major representative of these methods is \textit{boosting}. Instead of simultaneously training many base learners, boosting starts with only one weak learner, but iteratively piles new weak learners on the current one to improve its performance. In detail, for supervised learning tasks, a boosting algorithm trains a weak learner and records its empirical residuals; Next, the boosting algorithm trains the second weak learner targeting on the residuals, combines the two learners to form an integrated model, and again records the new residuals. By repeating the procedure, the residual of the model decreases, and the boosting algorithm can get promising performance by choosing a proper number of iterations. Based on such procedures, boosting-based algorithms \cite{Freund97, chen2015xgboost, parnell2020snapboost}, theories \cite{schapire2013boosting,bickel2006some}, and applications \cite{truong2020robust,ma2020diagnostic,taherkhani2020adaboost} emerge drastically.

In addition to the algorithmic studies, a wealth of literature concentrates on the theoretical properties of ensemble algorithms, exploring why boosting and bagging are effective \cite{domingos1997does,buhlmann2002analyzing,breiman2001using,cai2020boosted,cui2021gbht,hang2021histogram,lu2021a}. However, these analyses failed to distinguish between the sequential ensemble methods and the parallel ensemble methods. 
Since these works simply let each base learner has the same parameters and training areas, these theoretical results fail to explain why sequential ensembles usually outperform parallel ensembles in many real-world data experiments. Therefore, in this paper, we propose a sequential ensemble algorithm called \textit{Adaptive Boosting Histogram Transform} (\textit{ABHT}) for regression which allows the diversity of base learners and turn to examine an adaptive boosting algorithm that coincides better with many real-world applications. When the target function lies in an H\"{o}lder continuous space with different local H\"{o}lder exponents and thus the order of smoothness varies from area to area, the boosting algorithm can well identify the local properties of the target function, while the parallel ensemble cannot. In this case, we are able to theoretically show the benefits of sequential over parallel ensemble algorithms by means of convergence rates.

Our contributions made in this paper can be summarized as follows:

\textit{(i)} 
Compared with the \textit{Boosted Histogram Transform} (\textit{BHT}) in \cite{cai2020boosted}, our proposed ABHT algorithm allows different parameters for each base learner, and takes early stopping into consideration. We theoretically demonstrate the local adaptivity of ABHT. To be specific, for the regression problem where the target function has local H\"{o}lder exponents on different sub-regions, we show that ABHT can recognize the regions with different $\alpha$-H\"{o}lder exponents.

\textit{(ii)}
From the theoretical perspective, we show that with high probability, the upper bound for the excess risk of ABHT can be significantly smaller than the lower bound for that of the \textit{Parallel Ensemble Histogram Transforms} (PEHT) proposed in \cite{hang2021histogram}. More precisely, by deriving finite-sample bounds for both ABHT an PEHT, we prove that under the locally H\"{o}lder continuous assumption, the upper bound of ABHT turns out to be strictly smaller than the lower bound of PEHT. While ABHT is locally adaptive and assigns different optimal parameters when fitting on each region, PEHT assigns the same parameters for all regions. Thus, PEHT has larger excess risk since the selected parameters usually disagree with the optimal ones for the locally H\"{o}lder smooth regions. 

\textit{(iii)} 
In experiments, we verify the theoretical findings. Through synthetic experiments on target functions with different orders of smoothness on different regions, we illustrate that ABHT can filter out the regions with different smoothness, while PEHT selects the same parameters for all regions. 
We also verify through simulations the influence of sample size over the performance gap between ABHT and PEHT.
Moreover, on multiple synthetic and real datasets, we show that the MSE performance of ABHT is significantly better than that of PEHT, especially on the less smooth regions.

The paper is organized as follows. Section \ref{sec::preliminaries} is a warm-up section for the introduction of some basic notations, definitions, the preliminaries on histogram transform regressor, and assumptions that are related to the local smoothness of the regression function. The two histogram transform ensemble learning methods for regression, namely ABHT and PEHT, are presented in Section \ref{sec::Method}. We provide our main results on the local adaptivity of ABHT in Section \ref{sec::mainresults}. In addition, we establish the upper bound of ABHT and lower bound of PEHT in terms of convergence rates. Some comments and discussions on the comparison of ABHT and PEHT will be also provided in this section. In Section \ref{sec::ErrorAnalysis}, we present the error analysis for both ABHT and PEHT . We conduct synthetic and real data experiments in Section \ref{sec::experiments}. An illustrative example on the local adaptivity of ABHT will also be provided in this section. All the proofs of Section \ref{sec::mainresults} can be found in Section \ref{sec::proofs}.

\section{Preliminaries} \label{sec::preliminaries}

\subsection{Notations} \label{sub::notations}

We predict the value of an unobserved output variable $Y$ based on the observed input variable $X$, based on a dataset $D := \{ (x_1, y_1), \ldots, (x_n, y_n) \}$ consisting of i.i.d.~observations drawn from an unknown probability measure $\mathrm{P}$ on $\mathcal{X}\times \mathcal{Y}$. Throughout this paper, we assume that $\mathcal{X} = [0,1]^d \subset \mathbb{R}^d$, $\mathcal{Y} \subset \mathbb{R}$ is compact and non-empty. Moreover, let $\mu$ denote the Lebesgue measure.

We use the notation $a \vee b := \max \{ a, b \}$ and $a \wedge b := \min \{ a, b \}$. For any $x \in \mathbb{R}$, let $\lfloor x \rfloor$ denote the largest integer less than or equal to $x$. Recall that for $1 \leq p < \infty$, the $L_p$-norm of $x = (x_1, \ldots, x_d)$ is defined by $\|x\|_p := (|x_1|^p + \cdots + |x_d|^p)^{1/p}$, and the $L_{\infty}$-norm is defined by $\|x\|_{\infty} := \max_{i \in [d]} |x_i|$. For $N, N_1, N_2 \in \mathbb{N}$,  $[N]$ and $[N_1, N_2]$ refer to the index sets $\{ 1, \ldots, N \}$ and $\{ N_1, \ldots, N_2 \}$, respectively.

For a hypercube set $A:=\otimes_{i=1}^d [l_i,r_i] \subset \mathbb{R}^d$ and for any $h\in (0,\min_{i}(r_i-l_i)/2)$, we define $A \ominus h := \otimes_{i=1}^d [l_i-h,r_i-h]$ and $A\oplus h := \otimes_{i=1}^d [l_i+h,r_i+h]$. The cardinality of $A$ is denoted by $\#(A)$, the diameter of $A$ is denoted by $|A|$, and the indicator function on $A$ is denoted by $\eins_A$ or $\eins \{ A \}$. Moreover, for any function $f : \mathbb{R}^d \to \mathbb{R}$ and function set $\mathcal{F}$ consisting of such functions $f$, $f_{|A}$ and $\mathcal{F}_{|A}$ denote their restrictions on $A$, respectively, i.e., $f_{|A} := f \cdot \eins_A$ and $\mathcal{F}_{|A} := \{ f \cdot \eins_A : f \in \mathcal{F} \}$.

\subsection{Least Square Regression} \label{sub::LSRegression}

In this paper, we consider the regression model $Y_i = f(X_i) + \varepsilon_i$, where $f(x):[0,1]^d\to \mathbb{R}$ is a measurable function and $\varepsilon_i$ are i.i.d.~random variables with zero mean and variance $\sigma^2< \infty$. Moreover, we consider the least square loss $L : \mathcal{Y} \times \mathbb{R} \to [0, \infty)$ defined by $L(y, f(x)) := (y - f(x))^2$ for our target of regression. Then, for a measurable decision function $f : \mathcal{X} \to \mathbb{R}$, the risk is defined by $\mathcal{R}_{L,\mathrm{P}}(f) := \int_{\mathcal{X} \times \mathcal{Y}} L(y, f(x)) \, d\mathrm{P}(x,y)$ and the empirical risk is defined by $\mathcal{R}_{L,\mathrm{D}}(f) := \frac{1}{n} \sum^n_{i=1} L(y_i,f(x_i))$. The Bayes risk, which is the smallest possible risk with respect to $\mathrm{P}$ and $L$, is given by $\mathcal{R}_{L, \mathrm{P}}^* := \inf \{ \mathcal{R}_{L, \mathrm{P}}(f) | f : \mathcal{X} \to \mathbb{R} \text{ measurable} \}$. Then the excess risk is defined as $\mathcal{R}_{L,\mathrm{D}}(f) - \mathcal{R}_{L, \mathrm{P}}^*$. Moreover, for the set $A$, define the restricted least squared loss by $L_A(y,t) := L(y,t) \eins_A(x)$.

In what follows, it is sufficient to consider predictors with values in $[-M, M]$. To this end, we introduce the concept of \textit{clipping} for the decision function, see also Definition 2.22 in \cite{StCh08}. Let $\wideparen{t}$ be the \textit{clipped} value of $t\in \mathbb{R}$ at $\pm M$ defined by $- M$ if $t < - M$,  $t$ if $t \in [-M, M]$, and $M$ if $t > M$. Then, a loss is called \textit{clippable} at $M > 0$ if, for all $(y, t) \in \mathcal{Y} \times \mathbb {R}$, there holds $L(x, y, \wideparen{t}) \leq L(x, y, t)$. According to Example 2.26 in \cite{StCh08}, the least square loss $L$ is \textit{clippable} at $M$ with the risk reduced after clipping, i.e.~$\mathcal{R}_{L, \mathrm{P}}(\wideparen{f}) \leq \mathcal{R}_{L, \mathrm{P}}(f)$. Therefore, in the following, we only consider the clipped version $\wideparen{f}_{\mathrm{D}}$ of the decision function as well as the risk $\mathcal{R}_{L, \mathrm{P}}(\wideparen{f}_{\mathrm{D}})$.

\subsection{Histogram Transform (HT) for Regression} \label{sub::histogram}

In this section, we will introduce the histogram transform partition and its implementation method. Based on the partition, we present histogram transform (HT) regressors.

\subsubsection{Histogram Transform Partition}

To give a clear description of one possible construction procedure of histogram transforms, we introduce a random vector $(R,s,b)$ where each element represents the rotation matrix, stretching factor, and translation vector, respectively. To be specific, $R$ denotes the rotation matrix which is a real-valued $d \times d$ orthogonal square matrix with unit determinant, that is, $R^{\top} = R^{-1}$ and  $\det(R) = 1$. Then $s$ stands for the stretching factor which is positive real-valued. Then the bin width defined on the input space is given by $h = s^{-1}$. Finally, $b \in [0,1]^d$ is a $d$-dimensional vector named translation vector.

\begin{figure}[htbp]
\centering
\vskip 0.0in
\centerline{\includegraphics[height=0.36\columnwidth]{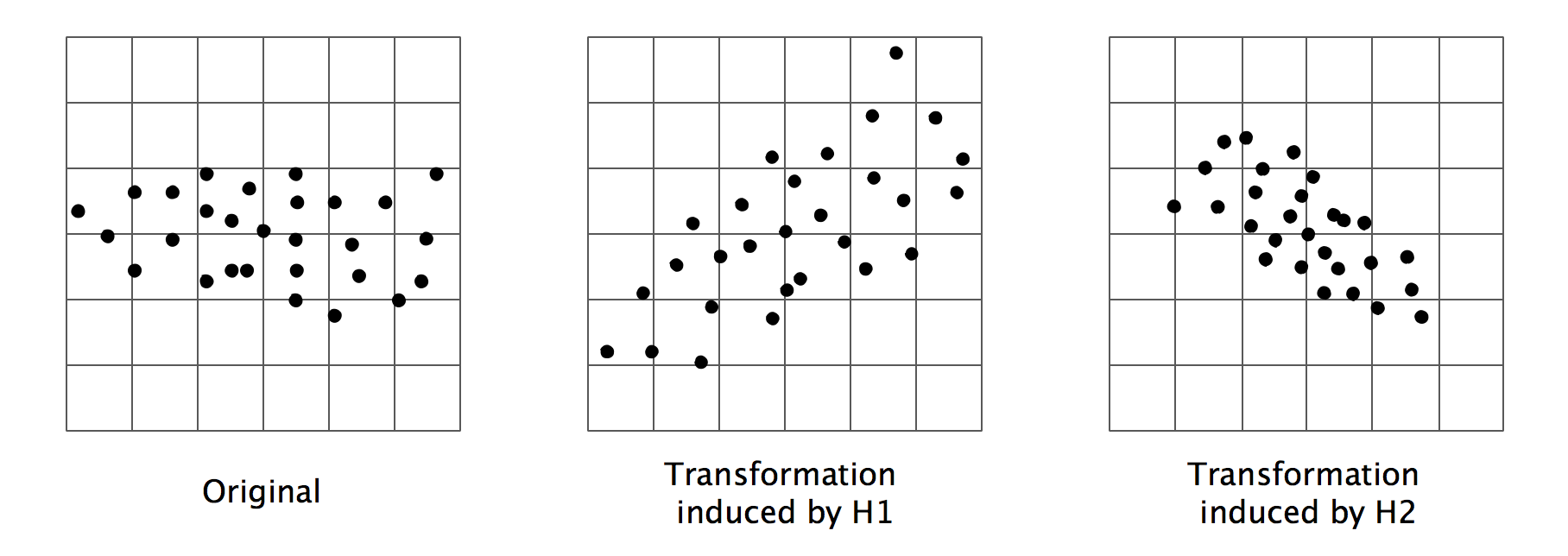}}
\vskip 0.0in
\caption{Two-dimensional examples of histogram transforms. The left subfigure is the original data and the other two subfigures are possible histogram transforms of the original sample space, with different rotating orientations and scales of stretching.}
\label{fig::RHT}
\vskip 0.0in
\end{figure}

Based on the above notation, we define the histogram transform $H:\mathcal{X}\to \mathcal{X}$ by 
\begin{align}\label{equ::HT}
H(x) := s R x + b. 
\end{align}
Here, it is worth pointing out that we adopt the isotropic bin width, i.e., the bin width of each dimension after transformation is $h$. It is important to note that we only consider the bin width equal to one. Otherwise, the same effect can be achieved by the scaling factor. We define the probability distribution of $R$, $s$, and $b$ as $\mathrm{P}_R$, $\mathrm{P}_s$, and $\mathrm{P}_b$, respectively. Then given bin width $h$, we let the three elements $(R,s,b) \sim (\mathrm{P}_R, \mathrm{P}_s, \mathrm{P}_b) =: \mathrm{P}_H$. Therefore, let $\lfloor H(x) \rfloor$ be the transformed bin indices, then the transformed bin is given by 
\begin{align}\label{TransBin}
A'_H(x) := \{ H(x') \ | \ \lfloor H(x') \rfloor = \lfloor H(x) \rfloor, x' \in \mathcal{X} \}.
\end{align}
The corresponding histogram bin containing $x \in \mathcal{X}$ in the input space is 
\begin{align}\label{equ::InputBin}
A_H(x) := \{ x' \ | \ H(x') \in A'_H(x), x' \in \mathcal{X} \}
\end{align}
and we further denote all the bins induced by $H$ as $\{ A_j' \}= \{ A_H(x) : x \in \mathcal{X} \}$ with the repetitive bin counted only once, and $\mathcal{I}_H$ as the index set for $H$ such that for $j\in \mathcal{I}_H$, we have $A_j' \cap \mathcal{X} \neq \emptyset$. As a result, the set $\pi_H := \{ A_j \}_{j \in \mathcal{I}_H} := \{ A_j' \cap \mathcal{X} \}_{j \in \mathcal{I}_H}$ forms a partition of partition of $\mathcal{X} = [0,1]^d$.

\subsubsection{A Practical Method for Constructing the Transform}

Here we describe a practical method for the construction of histogram transforms we are confined to in this study. Starting with a $d \times d$ square matrix $M$, consisting of $d^2$ independent univariate standard normal random variates, a Householder $Q R$ decomposition is applied to obtain a factorization of the form $M = R \cdot W$, with orthogonal matrix $R$ and upper triangular matrix $W$ with positive diagonal elements. The resulting matrix $R$ is orthogonal by construction and can be shown to be uniformly distributed. Unfortunately, if $R$ does not feature a positive determinant then it is not a proper rotation matrix. In this case, we can change the sign of the first column of $R$ to construct a new rotation matrix $R^+$. We let the scaling factor $s = h^{-1}$. Moreover, the translation vector $b$ is drawn from the uniform distribution over the hypercube $\mathcal{X} = [0,1]^d$.

\subsubsection{Histogram Transform (HT) Regressor}\label{sec::htr}

Given a histogram transform $H$, the set $\pi_{H} = \{ A_j \}_{j \in \mathcal{I}_H}$ forms a partition of $\mathcal{X} = [0,1]^d$. We consider the following function set $\mathcal{F}_H$ defined by
\begin{align}\label{equ::functionFn}
\mathcal{F}_H := \biggl\{ \sum_{j \in \mathcal{I}_H} c_j \eins_{A_j} : c_j \in [-M, M] \biggr\}.
\end{align}
In order to constrain the complexity of $\mathcal{F}_H$, we penalize on the bin width $h := (h_i)_{i=1}^d$ of the partition $\pi_H$. Then the histogram transform (HT) regressor can be produced by the regularized empirical risk minimization (RERM) over $\mathcal{F}_H$, i.e.
\begin{align*}
(f_{\mathrm{D}},h_*)
= \argmin_{f \in \mathcal{F}_H, \, h \in \mathbb{R}^d} \Omega(h) + \mathcal{R}_{L,\mathrm{D}}(f),
\end{align*}
where $\Omega(h) := \lambda h^{-2d}$. Since $h^{-d}$ is nearly equal to the number of cells in histogram partition, we use the regularization term $\Omega(h)$ to penalize the cell number in the histogram and thus to avoid overfitting.

\subsection{Local $\alpha$-H\"{o}lder Exponent}\label{sec::localalpha}

Existing literature considered the ordinary $\alpha$-H\"{o}lder continuous exponent, and showed that the parallel and sequential ensembles of HT regressors can achieve fast convergence rates \cite{hang2021histogram, cai2020boosted}.

\begin{definition}[$\alpha$-H\"{o}lder continuity]
A function $f : \mathcal{X} \to \mathbb{R}$ is $\alpha$-H\"{o}lder continuous, denoted as $f\in C^{\alpha}(\mathcal{X})$, $\alpha \in(0,1]$, if there exists a constant $c_L > 0$ such that for all $x, x' \in \mathcal{X}$, we have $|f(x) - f(x')| \leq c_L \|x-x'\|^{\alpha}$.
\end{definition}

However, in real-world datasets, the regression functions could have different orders of smoothness across the domain. Therefore, to investigate a larger variety of regression functions that appears in real-world data sets, we introduce the local H\"{o}lder exponent \cite{seuret2002local} to measure the local smoothness of an H\"{o}lder continuous target function.

\begin{definition}[Local H\"{o}lder exponent]\label{def::localholder}
Let $f: \mathcal{X}\to \mathbb{R}$ be a function, for an open subset $\Omega \subset \mathcal{X}$, the local H\"{o}lder exponent of $f$ is defined by $\alpha_{\mathrm{loc}}(\Omega; f)=\sup \{\alpha : f \cdot \eins_{\Omega} \in C^\alpha(\Omega)\}$.
\end{definition}

The local H\"{o}lder exponent is able to measure the local continuity on different subregions.
By Definition \ref{def::localholder}, there naturally holds that for $\Omega' \subset \Omega \subset \mathcal{X}$, $\alpha_{\mathrm{loc}}(\Omega') \geq \alpha_{\mathrm{loc}}(\Omega)$. Therefore, for any $\emptyset \subset B_K \subset \cdots \subset B_1 = \mathcal{X}$, we naturally have $\alpha_{\mathrm{loc}}(B_K) \geq \cdots \geq \alpha_{\mathrm{loc}}(B_1)$. If the local exponents of all subsets are the same, we could simply use the ordinary H\"{o}lder exponent to measure the smoothness of the target function. Therefore, to model the complex structure of the regression function of the real-world data sets, we naturally assume that the target function has different local H\"{o}lder exponents on different subsets.

\begin{assumption}\label{def::localholder2}
Assume that there exists a series of subsets, denoted as $B_k \subset \mathcal{X}$, $k \in [K]$, and $\emptyset \subsetneq B_K \subsetneq \cdots \subsetneq B_1 = \mathcal{X}$, such that $\alpha_{\mathrm{loc}}(B_K; f)> \cdots >\alpha_{\mathrm{loc}}(B_1; f)$.
\end{assumption}

A regression function $f$ is locally H\"{o}lder continuous with exponent $\alpha_k$ in $B_k$ if $f$ is uniformly H\"{o}lder continuous with exponent $\alpha_k$ on any compact subsets of $B_k$. When $k=1$, the local H\"{o}lder exponent coincides with the uniform H\"{o}lder exponent.

\section{Histogram Transform Ensemble Learning Methods for Regression} \label{sec::Method}

\subsection{Adaptive Boosting Histogram Transform (ABHT) for Regression} \label{sec::multiRERM}

Before we start, let us recall the boosted histogram transform (BHT) for regression proposed in \cite{cai2020boosted}, which is a gradient boosting algorithm using HT regressor as base learners (Algorithm \ref{alg::BHT}).

\begin{algorithm} 
\caption{Boosting Histogram Transform for Regression}
\label{alg::BHT}
\KwIn{
Training data $D := (x_i, y_i)_{i=1}^{n}$;
\\
\quad \quad \quad \quad Learning rate $\rho > 0$;
\\
\quad \quad \quad \quad Maximum iteration times $T$;
\\
\quad \quad \quad \quad Bin width $h$.
}
Initialization: For $i=1,\cdots, n$, $U_i = y_i$. Set $t=1$, $\epsilon_0 = 0$.
\\
\While{$t < T$}{
Set the bin width $h_{t} = h$ and generate random vector $(R,s,b)$;
\\
Generate histogram transform $H_t$ and apply data independent splitting to the transformed sample space;
\\
Apply constant functions to each cell, that is, fit residuals dataset $(X_i, U_i)_{i=1}^n$ with function $f_t$ such that
\begin{align*}
f_t := \argmin_{f \in \mathcal{F}_{H_t}} \sum_{i=1}^n (U_i - f(X_i))^2.
\end{align*}
\\
Update the residuals $U_i = U_i - \rho f_t(X_i)$ and MSE by $\epsilon_t = \frac{1}{n}\sum_{i=1}^n U_i^2$. 
\\
\If{$\epsilon_t > \epsilon_{t-1}$}{
	Continue;
}
Update the number of iteration by $t = t+1$.
}
\KwOut{
BHT Regressor $f_{\mathrm{D},h} := \sum_{l=1}^T \rho f_l$ and the residual dataset $D'_{h} := (X_i, U_i)_{i=1}^{n}$.
}
\end{algorithm}

It is well worth mentioning that BHT only adopts a na\"{i}ve version of gradient boosting, where the parameters of each base learner are the same. To be specific, in BHT, the bin width of each base learner is of the same order. However, the base learners in a boosting algorithm can actually have different parameters, so as to fit more complicated target functions. On the other hand, BHT failed to involve the idea of early stopping, which is frequently used in the real-world applications of boosting algorithms. In BHT, each base learner is trained on the entire domain $\mathcal{X}$. However, for complicated target functions, there are regions that are relatively easy to fit, and also regions that are relatively hard to fit. Therefore, if all base learners have the same parameters and training areas, some regions may be already overfitted with a certain number of iterations, while others remain under-fitted. These two flaws make BHT unadaptable to target functions with different orders of smoothness.

In this section, we introduce an adaptive version of BHT, namely \textit{adaptive boosting histogram transform (ABHT)} for regression, whose base learners can have different parameters and training areas. The main idea of ABHT is to train boosting histogram transform regressor with alternative bin widths and number of iterations in different subregions sequentially.

Compared with BHT, ABHT has the following characteristics:
\begin{itemize}
\item \textit{Locally adaptive bin width}. The bin width of each base learner can be different.
\item \textit{Early stopping}. We stop training the model in the region where the target function has already been well fitted.
\end{itemize}

To introduce our ABHT algorithm, we first need to do the initialization. To this end, we set the initialized regression function $f_{\mathrm{D},\mathrm{B}}^0(x) = 0$. Moreover, let $\mathfrak{X}_1 := (A_{1,j})_{j \in \mathfrak{J}_1}$ be a na\"ive histogram partition on $\mathcal{X}=[0,1]^d$ and the indices set $\mathfrak{J}_{1,*} := \emptyset$.

Now, let us formulate the iteration stage. For any $l \in [L]$, $L \in \mathbb{N}$, let
\begin{itemize}
\item 
$\mathfrak{X}_l$ be the region where the target function is fitted. Then we have the nested relationship $\mathfrak{X}_1 \supset \cdots \supset \mathfrak{X}_L$.
\item 
$\mathfrak{T}_l \in \mathbb{N}$ denote the numbers of iterations. If we set $T_0 := 0$ and $T_l := \sum_{i=1}^l \mathfrak{T}_i$ for $i \in [l]$, then $T := T_L$ is the total number of iterations. 
\item 
$\mathfrak{h}_l$ denote the corresponding bandwidths. If $h_t$ is the bin width of $t$-th iteration of the ABHT, then we have $h_t = \mathfrak{h}_l$ for any $t \in [T_{l-1} + 1, T_l]$. Given bin widths $h_t = \mathfrak{h}_l$, $t \in [T_{l-1} + 1, T_l]$, we generate $\mathfrak{T}_l$ i.i.d.~transforms $\{ H_t : t \in [T_{l-1} + 1, T_l] \}$ from the probability distribution $\mathrm{P}_{H}$  as mentioned in Section \ref{sub::histogram} and $\mathcal{F}_{H_t}$ is the function space defined by \eqref{equ::functionFn}.
\item 
$\rho \in [0,1)$ be a shrinkage parameter.
\end{itemize}
Then, for fixed parameters $\mathfrak{h}_l$ and $\mathfrak{T}_l$, if we consider the following function space 
\begin{align}\label{eq::Glh}
\mathfrak{F}^l_{\mathfrak{h}_l,\mathfrak{T}_l}
:= \biggl\{f=\sum_{t=T_{l-1}+1}^{T_l} w_t f_{t|\mathfrak{X}_l} + \rho\cdot \mathfrak{f}^{l-1}_{\mathrm{D},\mathrm{B}|\mathfrak{X}_l} \, : \, f_t \in \mathcal{F}_{H_t}, w_t > 0,  t \in [T_{l-1} + 1, T_l] \biggr\}
\end{align}
on the region $\mathfrak{X}_l$, then the empirical minimizer on $\mathfrak{X}_l$ is given by
\begin{align}\label{equ::fdtlhl}
f_{\mathrm{D},\mathfrak{h}_l,\mathfrak{T}_l}^l
:= \argmin_{f\in \mathfrak{F}^l_{\mathfrak{h}_l,\mathfrak{T}_l}} \mathcal{R}_{L_{\mathfrak{X}_l},\mathrm{D}}(f).
\end{align}
Here, in order to simplify the theoretical analysis of boosting, following the approach of \cite{blanchard2003rate}, we ignore the dynamics of the optimization procedure and simply consider minimizers of an empirical cost function.

According to the optimal parameter selection in \cite[Theorems 1 \& 2]{cai2020boosted}, we know that fitting the target function with a higher degree of smoothness requires larger bin width. Therefore, to get a lower complexity of our algorithm, we should first fit the subregions $\{ A_{l,j}, j \in \mathfrak{J}_l\setminus \mathfrak{J}_{l,*} \}$ of $\mathfrak{X}_l$ with the highest degree of smoothness as well as possible. To achieve this, we set the optimal bin width parameter $\mathfrak{h}_{l,*}$ for the whole $\mathfrak{X}_l$ to be the largest optimal bin width parameter $\mathfrak{h}_{l,j,*}$ on all subregions $\{ A_{l,j}, j \in \mathfrak{J}_l \setminus \mathfrak{J}_{l,*} \}$ of $\mathfrak{X}_l$.

Let $f_{\mathrm{D},\mathfrak{h}_l,\mathfrak{T}_l}^l$ be the empirical minimizer \eqref{equ::fdtlhl} and $\{ (\mathfrak{h}_{l,j},\mathfrak{T}_{l,j}), j \in \mathfrak{J}_l \setminus \mathfrak{J}_{l,*} \}$ be the bin width parameters and the corresponding numbers of iterations for the subregions $\{ A_{l,j}, j \in \mathfrak{J}_l\setminus \mathfrak{J}_{l,*} \}$ of $\mathfrak{X}_l$. To determine the optimal value for the parameters $(\mathfrak{h}_{l,j}, \mathfrak{T}_{l,j})$, we consider the following optimization problems on these subregions $A_{l,j}$:
\begin{align*}
(\mathfrak{h}_{l,j,*},\mathfrak{T}_{l,j,*})
= \argmin_{h_l \in \mathbb{R}, \mathfrak{T}_l \in \mathbb{N}} \lambda_{1,l,j} \mathfrak{h}_l^{-2d} + \lambda_{2,l,j} \mathfrak{T}_l^p + 
\mathcal{R}_{L_{A_{l,j}},\mathrm{D}}
(f^l_{\mathrm{D},\mathfrak{h}_l,\mathfrak{T}_l}),
\qquad
j \in \mathfrak{J}_l \setminus \mathfrak{J}_{l,*},
\end{align*}
where $\lambda_{1,l,j}, \lambda_{2,l,j} > 0$ are regularization parameters and $p > 2$ is a constant. Then we assign the largest value of all the optimal bin width $\{ \mathfrak{h}_{l,j,*}, j \in \mathfrak{J}_l\setminus \mathfrak{J}_{l,*} \}$ to the optimal bin width $\mathfrak{h}_{l,*}$ of the whole $\mathfrak{X}_l$, i.e., we set
\begin{align}\label{eq::hlstar}
\mathfrak{h}_{l,*} := \bigvee_{j \in \mathfrak{J}_l \setminus \mathfrak{J}_{l,*}} \mathfrak{h}_{l,j,*}.
\end{align}
The number of iterations $\mathfrak{T}_{l,j,*}$ corresponding to these largest bin widths $\mathfrak{h}_{l,j,*}$ will be assigned to the number of iterations $\mathfrak{T}_{l,*}$ for the whole $\mathfrak{X}_l$. Thus, we obtain the boosted regressor 
\begin{align}\label{eq::fDBl}
\mathfrak{f}_{\mathrm{D},\mathrm{B}}^l(x)
:= f^l_{\mathrm{D},\mathfrak{h}_{l,*},\mathfrak{T}_{l,*}}(x)
\end{align}
with optimal parameters $\mathfrak{h}_l = \mathfrak{h}_{l,*}$ and $\mathfrak{T}_l = \mathfrak{T}_{l,*}$ in \eqref{equ::fdtlhl} and \eqref{eq::Glh}.

Now, based on the optimal parameter $\mathfrak{h}_{l,*}$, we are able to find those subregions with the highest degree of smoothness, since larger bin width corresponds to a higher degree of smoothness of the target function in the subregions. To avoid overfitting, these well-fitted subregions should be early stopped. In other words, we aim to find out these early stopping subregions whose optimal bin width are $\mathfrak{h}_{l,*}$.

With the bin width $\mathfrak{h}_{l,*}$, we generate a new partition $\{ A_{l+1,j}, j \in \mathfrak{J}_{l+1} \}$ of $\mathfrak{X}_l$. Let $f_{\mathrm{D},\mathfrak{h}_l,\mathfrak{T}_l}^l$ be the empirical minimizer \eqref{equ::fdtlhl} and $\{ (\widetilde{\mathfrak{h}}_{l,j}, \widetilde{\mathfrak{T}}_{l,j}), j \in \mathfrak{J}_{l+1}  \}$ be the bin width parameters and the corresponding numbers of iterations for the subregions $\{ A_{l+1,j}, j \in \mathfrak{J}_{l+1} \}$ of $\mathfrak{X}_l$. To determine the optimal value for the parameters $(\widetilde{\mathfrak{h}}_{l,j}, \widetilde{\mathfrak{T}}_{l,j})$, we consider the following optimization problems on these subregions $A_{l+1,j}$:
\begin{align*}
(\widetilde{\mathfrak{h}}_{l,j,*}, \widetilde{\mathfrak{T}}_{l,j,*})
= \argmin_{h_l \in \mathbb{R}, \mathfrak{T}_l \in \mathbb{N}} 
\widetilde{\lambda}_{1,l,j} \mathfrak{h}_l^{-2d} + \widetilde{\lambda}_{2,l,j} \mathfrak{T}_l^p + \mathcal{R}_{L_{A_{l+1,j}}}(f^l_{\mathrm{D},\mathfrak{h}_l,\mathfrak{T}_l}),
\end{align*}
where $\widetilde{\lambda}_{1,l,j}, \widetilde{\lambda}_{2,l,j} > 0$ are regularization parameters. By setting
\begin{align*}
\mathfrak{J}_{l+1,*}
:= \Bigl\{ j : \argmax_{j\in \mathfrak{J}_{l+1}} \widetilde{\mathfrak{h}}_{l,j,*} \Bigr\},
\end{align*}
the early stopping region of $\mathfrak{X}_l$ can be given by
\begin{align}\label{eq::earlyregion}
\mathfrak{A}_{l,*}
:= \Delta \mathfrak{X}_l
:=\bigcup_{j\in \mathfrak{J}_{l+1}^*} A_{l+1,j}
\end{align}
and the corresponding residual region is denoted as
\begin{align}\label{eq::Xl}
\mathfrak{X}_{l+1} 
:= \mathfrak{X}_l \setminus \Delta \mathfrak{X}_l
:= \mathfrak{X}_l \setminus \mathfrak{A}_{l,*}
= \mathcal{X} \setminus \biggl( \bigcup_{j=1}^l \mathfrak{A}_{j,*} \biggr).
\end{align}
Thus, we find the corresponding early stopping region $\mathfrak{A}_{l,*}$ and finish the $l$-th iteration stage.

If the algorithm is terminated after $L$ iteration stages, then the adaptive boosting histogram transform (ABHT) for regression can be given by
\begin{align}\label{eq::fDB}
f_{\mathrm{D},\mathrm{B}}(x) := \sum_{l=1}^L \mathfrak{f}_{\mathrm{D},\mathrm{B} | \Delta \mathfrak{X}_l }^l(x) 
:= \sum_{l=1}^L \mathfrak{f}_{\mathrm{D},\mathrm{B} | \mathfrak{A}_{l,*}}^l(x),
\end{align}
where $\mathfrak{A}_{l,*} := \Delta \mathfrak{X}_l := \mathfrak{X}_l \setminus \mathfrak{X}_{l+1}$.

\begin{algorithm}
\caption{Adaptive Boosting Histogram Transform for Regression}
\label{alg::ABHT}
\KwIn{
Training data $D := (x_i, y_i)_{i=1}^{n}$;
\\
\quad \quad \quad \quad Shrinkage parameter $\rho > 0$;
\\
\quad \quad \quad \quad Bin width parameter gird $\boldsymbol{h}$;
\\
\quad \quad \quad \quad Maximum iteration times $T$.}
Initialization: Set $l=1$ and $D_1 = D$. Set $\boldsymbol{h}_1 := \boldsymbol{h}$.
\\
Generate a na\"ive histogram partition
$\mathfrak{X}_1=(A_{1,j})_{j \in \mathfrak{J}_1}$ on $\mathcal{X}$.\\
\While{$\mathfrak{X}_l \neq \emptyset$}{
\For{$h \in \boldsymbol{h}_l$}{
	With training data $D_l$, learning rate $\rho$, maximum iteration times $T$, and bin width $h$ as the input,
	we obtain the output $\mathfrak{f}_{\mathrm{D},h}^l$ and $D'_{h}$  by Algorithm \ref{alg::BHT}.
}
Determine the optimal bin width $\mathfrak{h}_{l,*}\in \boldsymbol{h}_l$  \eqref{eq::hlstar};
\\
Obtain the optimal boosted regressor $\mathfrak{f}_{\mathrm{D},\mathrm{B}}^l$ in \eqref{eq::fDBl};
\\
Partition the space $\mathfrak{X}_l$ to the cells with diameter $\mathfrak{h}_{l,*}$;
\\
Identify the early stopping region $\mathfrak{A}_{l,*}$ \eqref{eq::earlyregion} and the residual region $\mathfrak{X}_{l+1}$ \eqref{eq::Xl};
\\
Update the training data $D_{l+1} := \{ (x_i,y_i - \rho \cdot \mathfrak{f}_{\mathrm{D},\mathrm{B}}^l(x_i)) \}_{i=1}^n$;
\\
Set the bin width grid $\boldsymbol{h}_{l+1} := \{h\in\boldsymbol{h}_{l }: h\leq \mathfrak{h}_{l,*}\}$;
\\
Update $l = l+1$.
}
\KwOut{
ABHT Regressor $f_{\mathrm{D},\mathrm{B}} := \sum_{l=1}^L \mathfrak{f}_{\mathrm{D},\mathrm{B}|\mathfrak{A}_{l,*}}^l$ \eqref{eq::fDB}.
}
\end{algorithm}

Here, we call each iteration stage $l$ as a ``stage'' and $\mathfrak{X}_l$ as the ``region'' of the $l$-th stage. In fact, when the target function has different orders of smoothness in different subregions, ABHT separates the input domain into regions according to their local smoothness. In stage $l$, ABHT recognizes the region with the $l$-th largest local H\"{o}lder exponent as $\mathfrak{X}_l$, and trains only in this region. Then stage by stage, ABHT becomes adaptive to local smoothness. Specifically, when the number of stages $L=1$, ABHT degenerates to na\"{i}ve BHT. Moreover, the shrinkage parameter $\rho$ plays an important role in properly adjusting the learner trained in previous stages. Since the optimal parameters for the $(l+1)$-th stage is different from that for the previous stages, the learner $\mathfrak{f}_{\mathrm{D},\mathrm{B}}^l$ can only serve as a rough model for the $(l+1)$-th stage but cannot be fully accepted. Thus, we use a shrinkage parameter $\rho$ to adjust the weight between stages. We summarize our ABHT algorithm in Algorithm \ref{alg::ABHT}.

\subsection{Parallel Ensemble Histogram Transform (PEHT) for Regression} \label{sec::baggingwithHTR}

In this section, we recall the parallel ensemble histogram transform (PEHT) for regression proposed in \cite{hang2021histogram}. Given bin widths $(h_t)_{t=1}^T$, we randomly generate $T$ histogram transforms $H_t$ with $\{(R,s,b)\}_{t=1}^T$ i.i.d from the probability distribution $\mathrm{P}_{H_t}$. Based on $H_t$, we define the function space $\mathcal{F}_{H_t}$ in the same way as \eqref{equ::functionFn} and define the $t$-th base HT regressor $f_{\mathrm{D},t}$ by 
\begin{align}
f_{\mathrm{D},t}
= \argmin_{f \in \mathcal{F}_{H_t}} \; \mathcal{R}_{L,\mathrm{D}}(f)
= \sum_{j\in \mathcal{I}_{H_t}} \frac{\sum_{i=1}^nY_i\eins_{A_j}(X_i)}{\sum_{i=1}^n\eins_{A_j}(X_i)} \eins_{A_j},
\qquad
t \in [T].
\label{eq::fDtERM}
\end{align}
Then the PEHT is defined by 
\begin{align}\label{eq::fDE}
f_{\mathrm{D},\mathrm{E}} := \frac{1}{T}\sum_{t=1}^T f_{\mathrm{D},t}(x).
\end{align}
It is noteworthy that different from PEHT in \cite{hang2021histogram} whose the bin widths of all base regressors are of the same order w.r.t.~$n$, in this paper, we consider that there are $L$ different bin widths of base regressors, which are denoted as $(\mathfrak{h}_l)_{l=1}^L$. Let the number of base regressors whose bin width is $\mathfrak{h}_l$ be denoted as $\mathfrak{T}_l$. Obviously, there holds $\sum_{l=1}^L \mathfrak{T}_l = T$.

\section{Main Results} \label{sec::mainresults}

In this section, we first demonstrate the local adaptivity of ABHT by showing that it can filter out the regions with different local H\"{o}lder exponents. Based on this result, we then present the finite-sample upper bound for the excess risk of the ABHT under local H\"older smoothness assumption. Moreover, we establish the finite-sample lower bound for the excess risk of the PEHT. Then we compare the upper bound for the excess risk of the ABHT with the lower bound for the excess risk of the PEHT. Finally, we present some comments and discussions on the obtained results.

Let us begin with the following assumptions.

\begin{assumption}\label{def::localholderP}
We make the following two restrictions on the probability measure $\mathrm{P}$.
\begin{itemize}
\item[(i)][Local $\alpha$-H\"{o}lder continuity] 
For $(b_k)_{k\in[K]} \subset (0,1]$ with $b_K < \cdots < b_1 = 1$, we consider $d$-dimensional hypercubes $B_k = [(1-b_k)/2, (1+b_k)/2]^d$ in Assumption \ref{def::localholder2}. That is, we assume for $k \in [K]$, $\alpha_k:=\alpha_{\mathrm{loc}}(B_k, f) \in (0,1]$ and $\alpha_K > \cdots > \alpha_1$.
\item[(ii)][Marginal distribution] 
$\mathrm{P}_X$ is a uniform distribution on $[0,1]^d$. 
\end{itemize} 
\end{assumption}

Indeed, Assumption \textit{(ii)} is a common assumption in regression problems \cite{tsybakov2009introduction}. In the following, for the ease of convenience, we write $\Delta B_k := B_k \setminus B_{k+1}$, and $\Delta m_k := \mu(B_k) - \mu(B_{k+1}) := b_k^d - b_{k+1}^d$,  $k \in [K]$.

\subsection{Local Adaptivity of ABHT}\label{sec::localadapt}

The following proposition  shows that ABHT can filter out the regions with different local H\"{o}lder exponents as in Assumption \ref{def::localholderP}. In the $l$-th stage, $l \in [K]$, the identified region $\mathfrak{X}_l$ differs up to the bin width $\mathfrak{h}_{l,*}$ from the ground truth region $B_l$ with local exponent $\alpha_k$.

\begin{proposition}\label{prop::adaptive}
Let the probability measure $\mathrm{P}$ satisfy Assumption \ref{def::localholderP} with $\{ B_l,\ l \in [K] \}$. Moreover, let the optimal bin width $\mathfrak{h}_{l,*}$ and the residual region $\mathfrak{X}_l$ be defined as in \eqref{eq::hlstar} and \eqref{eq::Xl}, respectively. Then for $l \in [K]$, Algorithm \ref{alg::ABHT} returns regions $\mathfrak{X}_l$ satisfying
\begin{align*}
B_l \ominus \mathfrak{h}_{l,*} 
\subset \mathfrak{X}_l \subset B_l \oplus \mathfrak{h}_{l,*}
\end{align*}
with probability $\mathrm{P}^n$ at least $1 - 3l (l-1) /n$. 
\end{proposition}

\subsection{Upper Bound for ABHT}\label{sec::upperabht}

The next theorem establishes the finite-sample upper bound for the excess risk of ABHT under the local H\"{o}lder continuity assumption.

\begin{theorem} \label{thm::upperboost}
Let Assumption \ref{def::localholderP} hold with $K \geq 2$ and $f_{\mathrm{D},\mathrm{B}}$ be the ABHT regressor defined as in \eqref{eq::fDB}. For all $\delta \in (0, \alpha_1 / d)$, 
if we choose 
\begin{align} \label{eq::ConditionRho}
\rho \leq \bigwedge_{s=2}^K n^{-\frac{\alpha_s(1+\delta)(2+2\delta)(\alpha_1-\alpha_s)}{\delta((2+2\delta)\alpha_1+d)((2+2\delta)\alpha_s+d)}},
\end{align}
then by taking
\begin{align} \label{eq::hlxTlx}
\mathfrak{h}_{l,*} = n^{-\frac{1}{(2+2\delta)\alpha_l+d}}
\quad
\text{ and }
\quad
\mathfrak{T}_{l,*} = n^{0},
\end{align}
there exists a constant $c_B > 0$ independent of $n$ such that
\begin{align*}
\mathbb{E}_{\mathrm{P}_H} \bigl(  \mathcal{R}_{L,\mathrm{P}}(\mathfrak{f}_{\mathrm{D},\mathrm{B}}) - \mathcal{R}_{L,\mathrm{P}}^*\bigr)
\leq c_B \sum_{k=1}^{K} \Delta m_k n^{-\frac{2\alpha_k - \delta d/(1+\delta)}{(2+2\delta)\alpha_{k}+d}}
\end{align*}
holds with high probability $\mathrm{P}^n$ at least $1 - 3K/n$.
\end{theorem}

This theorem illustrates that the excess risk of ABHT consists of errors on $K$ different regions $\Delta B_l$, which rely on the local smoothness $\alpha_l$ and its volume $\Delta m_l$. In particular, if $K=1$, the target function belongs to the usual H\"older space $C^{\alpha}(\mathcal{X})$ with global smoothness parameter $\alpha = \alpha_1$, and ABHT degenerates to the BHT algorithm proposed in \cite{cai2020boosted}.
In this case, as a byproduct of Theorem \ref{thm::upperboost}, we prove the almost optimal convergence rate $n^{- 2\alpha / ((2+2\delta)\alpha+d)}$ for BHT.
Compared with the rate $n^{- 2\alpha / (4 - 2 \delta) \alpha + d}$ established in \cite{cai2020boosted}, our rate is strictly faster owing to the improvement of the complexity analysis in the function space.

We mention that Theorem \ref{thm::upperboost} also holds for piecewise H\"{o}lder continuous target functions \cite{morkisz2016approximation}, where there exist discontinuous ``jumps'' between different regions. In fact, due to the nature of histogram transforms, the non-adaptive version BHT can already achieve the same rate as in \cite{cai2020boosted} for piecewise H\"{o}lder continuous target functions with the same smoothness index on different regions, whereas it fails to properly approximate local H\"{o}lder continuous target functions with different H\"{o}lder exponents. Moreover, by adopting a restricted loss function as in \cite[Equation (13)]{cai2020boosted} or \cite[Theorem 4]{hang2021histogram}, we are able to leave out the boundary effect on the convergence rate as well.

\subsection{Lower Bound for PEHT}\label{sec::lowerpeht}

In this section, under the local H\"{o}lder continuity assumption, we present the lower bound for the excess risk of PEHT in the form of a bias-variance trade-off depending on the bin width parameter $h$ and the volume $\Delta m_k$ of the regions $\Delta B_k$.

\begin{theorem} \label{thm::lowerbag}
Let $\mathcal{P}$ be the class of the probability distribution satisfying Assumption \ref{def::localholderP}. Moreover, let $f_{\mathrm{D},\mathrm{E}}$ be the PEHT be defined as in \eqref{eq::fDE} with bin widths $(h_t)_{t=1}^T$. Then we have 
\begin{align} \label{eq::LBbag}
\inf_{f_{\mathrm{D},\mathrm{E}}}\sup_{\mathrm{P} \in \mathcal{P}}\mathbb{E}_{\mathrm{P}_H \otimes \mathrm{P}^n}
\mathcal{R}_{L,\mathrm{P}}(f_{\mathrm{D},\mathrm{E}}) - \mathcal{R}_{L,\mathrm{P}}^*  \geq c_E \inf_{h} \biggl( n^{-1} h^{-d} +  \sum_{k=1}^K \Delta m_k h^{2\alpha_k} \biggr),
\end{align}
where $c_E > 0$ is a constant which is independent of $n$ and will be specified in the proof. 
\end{theorem}

Theorem \ref{thm::lowerbag} gives a bias-variance trade-off of the lower bound for the excess risk of PEHT when the target function is locally $\alpha$-H\"{o}lder smooth. It is easy to see that if smaller $h$ is chosen, the first term on the right-hand side of \eqref{eq::LBbag} becomes larger whereas the second term becomes smaller, which corresponds to larger variance and lower bias of the estimator.

\subsection{Comparison of ABHT and PEHT}\label{sec::thmcompare}

The next theorem shows that under certain conditions, the finite-sample upper bound for the excess risk of ABHT can be significantly smaller than the lower bound for that of PEHT.

\begin{theorem}\label{thm::finiten}
Let Assumption \ref{def::localholderP} hold with $K \geq 2$. For any $\delta \in (0, \alpha_1 / d)$, let 
\begin{align}\label{eq::kstar}
k^* := \argmax_{k \in [K]} \Delta m_k n^{- \frac{2 \alpha_k}{(2 + 2 \delta) \alpha_k + d}}. 
\end{align}
Suppose that $\Delta m_{k^*} < (Kc_B/c_E)^{-(2\alpha_{k^*}+d)/(2\alpha_{k^*})}$, where $c_B$ and $c_E$ are the constants as in Theorem \ref{thm::upperboost} and \ref{thm::lowerbag}, respectively. Then for any $n \leq N(\delta)$ with 
\begin{align}\label{eq::n0}
N(\delta) := \biggl\lfloor \biggl( \biggl( \frac{K c_B}{c_E} \biggr)^{-\frac{2\alpha_{k^*}+d}{2\alpha_{k^*}}} \cdot \frac{1}{\Delta m_{k^*}} \biggr)^{\frac{\alpha_{k^*}(2\alpha_{k^*}+d)}{10d^2\delta}} \biggr\rfloor,
\end{align}
there holds
\begin{align} \label{eq::ratecomparison}
\mathbb{E}_{\mathrm{P}_H \otimes \mathrm{P}^n}
\mathcal{R}_{L,\mathrm{P}}(f_{\mathrm{D},\mathrm{E}}) - \mathcal{R}_{L,\mathrm{P}}^*
\geq n^{\frac{10d^2\delta}{(2\alpha_{k^*}+d)^2} } \cdot
\bigl( \mathbb{E}_{\mathrm{P}_H \otimes \mathrm{P}^n} \mathcal{R}_{L,\mathrm{P}}(f_{\mathrm{D},\mathrm{B}}) - \mathcal{R}_{L,\mathrm{P}}^* \bigr).
\end{align}
\end{theorem}

Given any finite sample size $n \in \mathbb{N}$, we can choose a sufficiently small $\delta > 0$ such that the critical sample size $N(\delta)$ in \eqref{eq::n0} satisfies $n \leq N(\delta)$ and the inequality \eqref{eq::ratecomparison} holds for all such $n \in \mathbb{N}$. In other words, on a given dataset $D_n$ the excess risk of PEHT is strictly larger than that of ABHT under the local H\"{o}lder continuity assumption. However, as the sample size $n \to \infty$, according to the definition of the critical sample size $N(\delta)$ in \eqref{eq::n0}, we have to force $\delta \to 0$ in order that $n \leq N(\delta)$ is satisfied. Consequently, we have $10 d^2 \delta / (2\alpha_{k^*}+d)^2 \to 0$ for the exponent of $n$ in \eqref{eq::ratecomparison}. In other words, if the sample size $n$ is sufficiently large, there will be no significant difference in the excess risks of PEHT and ABHT. These phenomena can be apparently observed from Figures \ref{fig::exp_trainsize_caseA} and \ref{fig::exp_trainsize_caseB} in Section \ref{sec::expn}.

Next, let us briefly discuss the reason why ABHT can have a smaller excess risk than PEHT under the local H\"older assumption. Recall that for a na\"{i}ve boosting algorithm, in order to achieve the smallest excess risk for learning target functions with global smoothness exponent $\alpha$, we select an optimal bin width which depends on $\alpha$. Therefore to achieve such a small risk, when fitting a locally H\"older smooth target function as defined in Assumption \ref{def::localholder2}, we should naturally select different bin widths for regions with different smoothness exponents. Generally speaking, smoother regions require larger optimal bin widths. However, as PEHT selects the same bin widths for the entire domain $\mathcal{X}$, which usually does not coincide with the optimal bin width for the subregions, it suffers from larger excess risk in these regions. To be specific, when the selected bin width is larger than the optimal value for a region, the approximation error is larger, while when the selected bin width is smaller than the optimal, the sample error becomes larger. By contrast, since our ABHT allows different bin widths for the regions with different orders of smoothness, it can approximate the local structure of the target function well. Thus benefited from its locally adaptive property, ABHT turns out to have a smaller approximation error than PEHT.

\subsection{Comments and Discussions}\label{sec::comments}

Previous theoretical works about boosting algorithms for regression include \cite{buhlmann2003boosting} and \cite{lin2019boosted}, where linear regressors and kernel ridge regressors are used as the base learners. These works analyze the learning performance by using the integral operator approach and prove the optimal convergence rate. However, this analysis turns out to be inapplicable to our method. In this paper, we conduct analysis under the framework of \textit{regularized empirical risk minimization} (RERM).

Recall that \cite{cai2020boosted} proposed the \textit{boosted histogram transform} (BHT) for regression, which implements a histogram transformed partition to the random affine mapped data, then adaptively leverages constant functions to obtain the individual regression estimates in the gradient boosting algorithm. In the space $C^{\alpha}$, $\alpha \in (0, 1]$, the convergence rate is proved to be $n^{-2\alpha/(4\alpha+d)}$. On the other hand, \cite{hang2021histogram} proposed the parallel ensemble histogram transforms (PEHT) for large-scale regression problems. The convergence rates of PEHT are shown to be $n^{-2\alpha/(2\alpha+d)}$. Therefore, the convergence rates established in \cite{cai2020boosted} failed to show the advantages of sequential over parallel ensemble learning in the commonly used H\"{o}lder space $C^{\alpha}$, $\alpha \in (0, 1]$.

In this paper, we mainly focus on the regression problem where the target function is locally H\"{o}lder continuous with exponents $\{ \alpha_k \in (0, 1], k \in [K] \}$, and propose a new variant of boosting algorithm in this setting, namely the \textit{adaptive boosting histogram transform} (ABHT) for regression. We successfully show that under the local H\"{o}lder conditions, the excess risk of ABHT algorithm can be significantly smaller than that of PEHT algorithm where the histogram transforms are used as base learners. 

Although sequential learning is empirically shown to be a more effective learning strategy than parallel ensemble learning for many real-world datasets, there has been little effort in explaining this observation theoretically. Instead of attaining a formal understanding of this problem in general, in this paper, we investigate the excess risk of two specific learning algorithms ABHT and PEHT by adopting the histogram transform regressors as base learners. Since the basic idea behind the boosting algorithm is to apply the functional gradient descent is to find the minimum of the loss function iteratively, the sequential method ABHT can capture the local properties of the target function well. To be specific, by exploiting the local H\"{o}lder exponent of the target function, Proposition \ref{prop::adaptive} shows that ABHT can filter out the regions with different local H\"{o}lder exponents. On the contrary, it is difficult for a parallel method to assign different optimal parameters to regions with different orders of smoothness. As a result, the approximation error (bias) of ABHT turns out to be smaller than that of PEHT (see Section \ref{sec::ErrorAnalysis}). Therefore, we are able to theoretically explain the advantages of sequential over parallel ensemble learning under particular conditions.

\section{Error Analysis} \label{sec::ErrorAnalysis}

In this section, we first conduct error analysis to obtain the upper bound of the excess risk for ABHT. To this end, we need to  analyze the order of bin width $\mathfrak{h}_l$ of $\mathfrak{f}_{\mathrm{D},\mathrm{B}}^l$ and the discrepancy between the early-stopping region $\mathfrak{X}_{l+1}$ defined by \eqref{eq::Xl} and the subregion $B_{l+1}$ in Section \ref{sec::hlstar} and \ref{sec::residualregion} respectively. Then we present the error decomposition for ABHT in Section \ref{sec::errordecomp}. Finally, in section \ref{sec::bagErrorAnalysis}, we analyze the lower bound of PEHT based on the bias-variance decomposition. Recall that the considered regression problem is associated with a locally $\alpha$-H\"{o}lder continuous function class.

\subsection{Error Analysis for ABHT} \label{sec::adaboostErrorAnalysis}

\subsubsection{Analysis on Adaptive Bin Width }\label{sec::hlstar}

In this section, to analyze the local excess risk of $\mathfrak{f}_{\mathrm{D},\mathrm{B}}^l$, we first need to analyze the order of bin width $\mathfrak{h}_{l,*}$ in \eqref{eq::hlstar} under Assumption \ref{def::localholderP}. We show that if the early stopping region $\mathfrak{X}_l$ approximates $B_l$ well, then the order of bin width $\mathfrak{h}_{l,*}$ relies on the local H\"{o}lder exponent of the regions $\Delta B_l$.

\begin{proposition}\label{prop::hl*}
Let Assumption \ref{def::localholderP} hold and $\mathfrak{h}_{l,*}$ be the optimal bin width defined as in \eqref{eq::hlstar}. For any fixed $l \in [K]$, if $B_l \ominus \mathfrak{h}_{l-1,*} \subset \mathfrak{X}_l \subset B_l \oplus \mathfrak{h}_{l-1,*}$ holds and $\rho$ satisfies \eqref{eq::ConditionRho}, then $\mathfrak{h}_{l,*}$ and $\mathfrak{T}_{l,*}$ are of the order in \eqref{eq::hlxTlx} with probability $\mathrm{P}^n$  at least $1 - 3l /n$. 
\end{proposition}

As shown above, if the $L_{\infty}$-norm distance between the sets $B_l$ and $\mathfrak{X}_l$ is less than $\mathfrak{h}_{l-1,*}$, then the optimal order of $\mathfrak{h}_{l,*}$ depends on the local H\"older exponent $\alpha_l$. More precisely, Proposition \ref{prop::hl*} shows that larger bin width $h_l$ are required for subregions with higher H\"{o}lder exponent. In particular, when $\alpha_l \in (0,1]$, optimal number of iterations $\mathfrak{T}_{l,*}$ are constants. In this case, more iteration times does not help to reduce the excess risk.

\subsubsection{Analysis on Localized Sub-regions}\label{sec::residualregion}

The following proposition shows the estimation accuracy of $\mathfrak{X}_{l+1}$ for subregions $B_{l+1}$ when the optimal order of $\mathfrak{h}_{l,*}$ in \eqref{eq::hlxTlx} is taken.

\begin{proposition}\label{prop::residualregion}
Let Assumption \ref{def::localholderP} hold and $l \in [K]$ be fixed. Moreover, for all $i \in [l]$, let the largest optimal bin width $\mathfrak{h}_{i,*}$ and the residual region $\mathfrak{X}_{i+1}$ be defined as \eqref{eq::hlstar} and \eqref{eq::Xl}, respectively. 
If we take $\rho$ as in \eqref{eq::ConditionRho}, and $\mathfrak{h}_{i,*}$, $\mathfrak{T}_{i,*}$ as in \eqref{eq::hlxTlx} for all $i \in [l]$, then
\begin{align*}
B_{l+1}\ominus\mathfrak{h}_{l,*} \subset \mathfrak{X}_{l+1} \subset B_{l+1}\oplus\mathfrak{h}_{l,*}
\end{align*}
holds with probability $\mathrm{P}^n$  at least $1 - 3l /n$. 
\end{proposition}

With the help of Propositions \ref{prop::hl*} and \ref{prop::residualregion}, we see that
bounding the excess risk of $\mathfrak{f}_{\mathrm{D},\mathrm{B}}^l$ can be reduced to
bounding the local excess risk of $\mathfrak{f}_{\mathrm{D},\mathrm{B}}^l$ on regions $ \Delta \mathfrak{X}_l$, which will be presented in the next subsections.

\subsubsection{Oracle Inequality for the $l$-th Stage}\label{sec::errordecomp}

To conduct our theoretical analysis, we need the population version of ABHT. 
To this end, let us define
\begin{align*}
\mathfrak{F}_{\mathfrak{h}_l}^l 
:= \biggl\{ f = \sum_{t=T_{l-1}+1}^{T_l} w_t f_t : f_t \in \mathcal{F}_{H_t}, h_t = \mathfrak{h}_l, t \in [T_{l-1}+1, T_l] \Big\}.
\end{align*}
Let $f_{\mathrm{P},t}$ be the population version of $f_{\mathrm{D},t}$ in
\eqref{eq::fDtERM}, that is,
\begin{align} \label{eq::fPtERM}
f_{\mathrm{P},t}(x) := \sum_{j\in \mathcal{I}_{H_t}} \frac{\sum_{i=1}^n f^*_{L,\mathrm{P}}(X_i)\eins_{A_j}(X_i)}{\sum_{i=1}^n\eins_{A_j}(X_i)} \eins_{A_j}(x).
\end{align}
Then we have
$\mathfrak{f}_{\mathrm{P}}^l := (1/\mathfrak{T}_l) \sum_{t=T_{l-1}+1}^{T_l}  f_{\mathrm{P},t} \in \mathfrak{F}_{\mathfrak{h}_l}^l $.
Let $\mathfrak{f}_{\mathrm{D},\mathrm{B}}^{l-1}$ 
and $\mathfrak{F}_{\mathfrak{h}_l,\mathfrak{T}_l}^l$ be defined as in \eqref{eq::fDBl} and \eqref{eq::Glh}, respectively.
Then we have
\begin{align}\label{eq::fPBl}
\mathfrak{f}_{\mathrm{P},\mathrm{B}}^l 
:= \rho \cdot \mathfrak{f}_{\mathrm{D},\mathrm{B} | \mathfrak{X}_l}^{l-1} 
+ \mathfrak{f}_{\mathrm{P} | \mathfrak{X}_l}^l \in \mathfrak{F}_{\mathfrak{h}_l,\mathfrak{T}_l}^l,
\end{align}
which can be used to approximate the target function $f^*_{L,\mathrm{P}|\Delta \mathfrak{X}_l}$.

Now, we are able to establish oracle inequalities for ABHT which will be crucial in establishing the convergence results of the estimator.

\begin{proposition}\label{prop::OracalBoost}
Let Assumption \ref{def::localholderP} hold. Moreover, let $\mathfrak{f}_{\mathrm{D},\mathrm{B}}^l$ and $\mathfrak{f}_{\mathrm{P},\mathrm{B}}^l$ be defined as in \eqref{eq::fDBl} and \eqref{eq::fPBl}, respectively. Then for any $\delta \in (0,1)$, there exists a constant $C_1>0$ independent of $n$ such that 
\begin{align*}
\mathcal{R}_{L_{\Delta \mathfrak{X}_l},\mathrm{P}}(\mathfrak{f}_{\mathrm{D},\mathrm{B}}^l) - \mathcal{R}_{L_{\Delta \mathfrak{X}_l},\mathrm{P}}^*
& \leq 12 \Bigl( \mathcal{R}_{L_{\Delta \mathfrak{X}_l},\mathrm{P}}(\mathfrak{f}_{\mathrm{P},\mathrm{B}}^l) - \mathcal{R}_{L_{\Delta \mathfrak{X}_l},\mathrm{P}}^* \Bigr) + 3456 M^2 \log n / n
\\
& \phantom{=} 
+ C_1 \Delta m_l \mathfrak{h}_{l,*}^{-\frac{\delta d}{1+\delta}} 
\bigvee_{i=1}^l \rho^{\frac{2\delta(l-i)}{1+\delta}} \mathfrak{h}_{i,*}^{-\frac{d}{1+\delta}} \mathfrak{T}_{i,*}^{-\frac{1}{1+\delta}} n^{-\frac{1}{1+\delta}}
\end{align*}
holds with probability $\mathrm{P}^n$ at least $1-3l/n$.
\end{proposition}

\subsubsection{Bounding the Approximation Error for the $l$-th Stage}

The next proposition presents the upper bound for the approximation error with restriction on subregions $\{ \Delta \mathfrak{X}_l, l \in [K] \}$.

\begin{proposition}\label{prop::ApproxBoost}
Let Assumption \ref{def::localholderP} hold. Moreover, let $\mathfrak{X}_l$ be the residual region as in \eqref{eq::Xl} and $\mathfrak{f}_{\mathrm{P},\mathrm{B}}^l$ be defined by \eqref{eq::fPBl}. Then for any $\delta \in (0,1)$, there exists a constant $C_2>0$ independent of $n$ such that 
\begin{align*}
&\mathbb{E}_{\mathrm{P}_H} \Bigl( \mathcal{R}_{L_{\Delta \mathfrak{X}_l},\mathrm{P}} \bigl( \mathfrak{f}_{\mathrm{P},\mathrm{B}}^l \bigr) -  \mathcal{R}_{L_{\Delta \mathfrak{X}_l},\mathrm{P}}^* \Bigr)
\\
& \leq   C_2 \Delta m_l \mathfrak{h}_{l,*}^{-\frac{\delta d}{1+\delta}} \biggl( \sum_{i=1}^l \rho^{2(l-i)} \bigl( \mathfrak{h}_{i,*}^2 \mathfrak{T}_{i,*}^{-1} + \mathfrak{h}_{i,*}^{2\alpha_l} \bigr)
+ \sum_{i=1}^{l-1} \rho^{\frac{2\delta(l-i)}{1+\delta}} \mathfrak{h}_{i,*}^{-\frac{d}{1+\delta}} \mathfrak{T}_{i,*}^{\frac{1}{1+\delta}} n^{-\frac{1}{1+\delta}} 
+ \frac{2 \log n}{n \mathfrak{h}_{l,*}^d} \biggr)
\end{align*}
holds with probability $\mathrm{P}^n$ at least $1-3l/n$.
\end{proposition}

\subsection{Error Analysis for PEHT} \label{sec::bagErrorAnalysis}

In this section, we present the lower bound of bias and variance of the PEHT when the regression function is locally H\"{o}lder continuous. First, let us define the population version of PEHT by 
\begin{align}\label{eq::fPE}
f_{\mathrm{P},\mathrm{E}} := \frac{1}{T} \sum_{t=1}^T f_{\mathrm{P},t},
\end{align}
where $f_{\mathrm{P},t}$ is defined as in \eqref{eq::fPtERM}. Then we make the following bias-variance decomposition:
\begin{align*}
\mathcal{R}_{L, \mathrm{P}}(f_{\mathrm{D},\mathrm{E}})-\mathcal{R}_{L, \mathrm{P}}^* 
= \big(\mathcal{R}_{L, \mathrm{P}}(f_{\mathrm{D},\mathrm{E}}) - \mathcal{R}_{L, \mathrm{P}}(f_{\mathrm{P},\mathrm{E}})\big)
+  \big(\mathcal{R}_{L, \mathrm{P}}(f_{\mathrm{P},\mathrm{E}})-\mathcal{R}_{L, \mathrm{P}}^*\big).
\end{align*}

\subsubsection{Lower Bound of Approximation Error of PEHT} \label{sec::approxbag}

The following proposition presents the lower bound of bias of the PEHT.

\begin{proposition}\label{prop::approxbag}
Let $\mathcal{P}$ be the class of the probability distribution satisfying Assumption \ref{def::localholderP} and $f_{\mathrm{P},\mathrm{E}}$ be defined by \eqref{eq::fPE}. Suppose that for certain constant $C_3 > 0$ independent of $n$, there holds $\mathfrak{T}_l \mathfrak{h}_l^{\alpha_{k}} \geq 2 C_3^{-1} c_L^2 L \mathfrak{T}_{l+1} \mathfrak{h}_{l+1}^{\alpha_{k}}$ for any $l \in [L-1]$, $k \in [K]$. Then we have 
\begin{align*}
\sup_{P\in\mathcal{P}} 
\mathbb{E}_{\mathrm{P}_H}
\mathcal{R}_{L, \mathrm{P}}(f_{\mathrm{P},\mathrm{E}})-\mathcal{R}_{L, \mathrm{P}}^* 
\geq C_3 \sum_{l=1}^L (\mathfrak{T}_l/T)^2 \sum_{k=1}^K \Delta m_k \mathfrak{h}_l^{2\alpha_k}.
\end{align*}
\end{proposition}

\subsubsection{Lower Bound of Variance of PEHT} \label{sec::samplebag}

Next we present the lower bound of variance of the PEHT.

\begin{proposition}\label{prop::samplebag}
Let $\mathcal{P}$ be the class of the probability distribution satisfying Assumption \ref{def::localholderP}. Moreover, let $f_{\mathrm{D},\mathrm{E}}$ and $f_{\mathrm{P},\mathrm{E}}$ be the PEHT defined as in \eqref{eq::fDE} and \eqref{eq::fPE}, respectively. Suppose that for certain constant $C_4 > 0$ independent of $n$, $\mathfrak{T}_l \mathfrak{h}_l^{-d} \geq 32 M^2 L C_4^{-1} \mathfrak{T}_{l+1} \mathfrak{h}_{l+1}^{-d}$ holds for any $l \in [L-1]$. Then we have 
\begin{align*}
\sup_{P \in \mathcal{P}} 
\mathbb{E}_{\mathrm{P}^n} \mathbb{E}_{\mathrm{P}_X} |f_{\mathrm{P},\mathrm{E}}(X) - f_{\mathrm{D},\mathrm{E}}(X)|^2  
\geq C_4 \sum_{l=1}^L (\mathfrak{T}_l/T)^2 n^{-1} \mathfrak{h}_l^{-d}.
\end{align*}
\end{proposition}

\section{Experiments} \label{sec::experiments}

In this section, we conduct numerical studies to validate the advantage of sequential over parallel ensemble algorithms by comparing the proposed adaptive boosting histogram transform (ABHT) with the parallel ensemble histogram transform (PEHT). Besides, we give an illustrative example to explain how ABHT can be locally adaptive on regions under different smoothness conditions.

\subsection{Experimental Settings}

We illustrate the experimental details of each comparing method below:
\begin{enumerate}
\item 
The PEHT is an ensemble version of HT regressors in a parallel manner. There are two hyper-parameters in total, including the bin width $h$ and the number of estimators $T$. For the hyper-parameters of PEHT, we search the number of estimators $T$ from $\{ 20, 50, 100, 200 \}$.
\item 
We conduct two boosting versions of HT regressor, including the classical BHT (Algorithm \ref{alg::BHT}) and the proposed ABHT (Algorithm \ref{alg::ABHT}). Two hyper-parameters are related to the boosting process, including the learning rate $\rho$, and the number of itertions $T$. We set the parameter range of the learning rate $\rho$ and the number of iteration $T$ to $\rho \in \{0.01, 0.02, 0.05, 0.1, 0.2\}$ and $T \in \{20, 50, 100, 200\}$. For ABHT, the initial region width $h_0$ is set to $0.2$ by default. To mention, two hyper-parameters $\rho$ and $T$ in ABHT are selected \textit{per stage}. If the number of validation points in a region is less than $10$, we also early stop this region, as there are not enough validation points to find out the best parameters.
\end{enumerate}
The common hyper-parameter for all methods is the bin width of the base HT regressor named $h$. We search the best parameter $h \in \{ 1e^{-3}, 2e^{-3}, 5e^{-3}, 1e^{-2}, 2e^{-2}, 5e^{-2}, 1e^{-1} \}$ in 1-dimensional synthetic experiments, $h \in \{2e^{-2}, 5e^{-2}, 1e^{-1}\}$ in 2-dimensional synthetic experiments, and $h \in \{5e^{-2}, 1e^{-1}, 2e^{-1}\}$ in 3-dimensional synthetic experiments.

In the experiments, we scale the features to the $[0, 1]$ range and use a separate validation set to select the best hyper-parameters. We evaluate the performance by repeating each experiments for 30 times and calculating the averaged mean squared errors under the test sets. 

\subsection{Experiments on Synthetic Datasets}

\subsubsection{Synthetic Cases}\label{sec::syndata}

We consider the following cases in synthetic experiments:
\vspace{0.2cm}

\noindent
\textbf{Case A:} 
As first, we consider a one-dimensional case with three different orders of smoothness. We define the target function in $[0, 1]$ as the combinations of three functions $f_1(x)$, $f_2(x)$, $f_3(x)$ in $[0, 1/8]$, $(1/8, 1/2]$, and $(1/2, 1]$ respectively. These three functions are continuous on the boundaries. The $\alpha$-H\"{o}lder conditions of these three functions are different. The definitions of these three functions are shown below:
\begin{enumerate}
\item $f_1(x) = 0.05 \cdot (-1)^{\lfloor x / 0.01 \rfloor + 1} + 0.05$, $x \in [0, 1/8]$,
\item $f_2(x) = 3 \cdot \sqrt[3]{x}$, $x \in (1/8, 1/2]$,
\item $f_3(x) = x$, $x \in (1/2, 1]$.
\end{enumerate}
Then the target function is defined by
\begin{align*}
f(x) = \begin{cases}
f_1(x) + \varepsilon, & \mbox{if $x \in [0, 1/8]$}, \\
f_2(x) + f_1(1 / 8) - f_2(1 / 8) + \varepsilon, & \mbox{if $x \in (1/8, 1/2]$}, \\
- f_3(x) + f_2(1 / 2) - f_2(1 / 8) + f_3(1 / 2) + \varepsilon, & \mbox{if $x \in (1/2, 1]$}, \\
\end{cases}
\end{align*}
where $\varepsilon \sim \mathcal{N}(0, 0.01^2)$ is a random variable.

\noindent
\textbf{Case B:}
We consider a 2-dimensional case, where the target function is a piecewise function with different $\alpha$-H\"{o}lder conditions in different regions. We define the target function $g$ by
\begin{align*}
g(x_1, x_2) = 
\begin{cases}
h(x_1, x_2) + (x_1 + x_2) / 3 + \varepsilon, & \mbox{ if } (x_1, x_2) \in [0, 1/3] \times [0, 1/3], 
\\
(\sqrt[3]{x_1} + \sqrt[3]{x_2} ) / 2 + \varepsilon, & \mbox{ if } (x_1, x_2) \in [0, 1/3] \times (1/3, 1], 
\\
(\sqrt[3]{x_1} + \sqrt[3]{x_2} ) / 2 + \varepsilon, & \mbox{ if } (x_1, x_2) \in (1/3, 1] \times [0, 1/3], 
\\
(x_1 + x_2) / 6 + 3/5 + \varepsilon, & \mbox{ if } (x_1, x_2) \in (1/3, 1] \times (1/3, 1], 
\end{cases}
\end{align*}
where $x_1, x_2 \in [0, 1]$ are respectively the first and the second dimension of sample points, $h(x_1, x_2) = 0.05 \cdot (-1)^{ \lfloor (x_1+x_2) / 0.1 \rfloor + 1}  + 0.45$, and $\varepsilon \sim \mathcal{N}(0, 0.01^2)$ is a random variable.

\begin{figure*}
\centering
\captionsetup[subfigure]{justification=centering, captionskip=2pt} 
\subfloat[][Case A] { 
\includegraphics[width=0.48\linewidth]{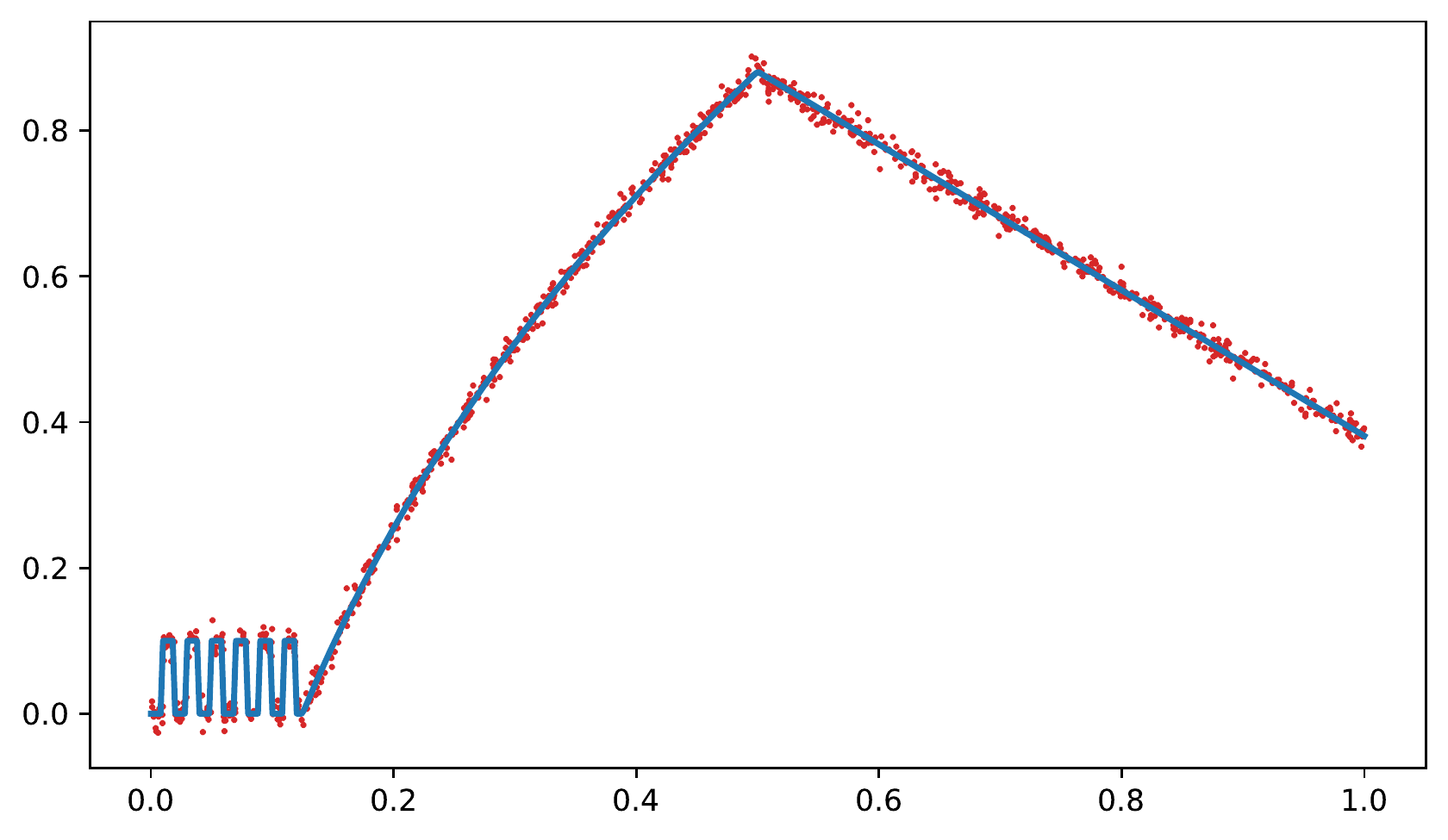} 
\label{fig::exp_syn_target1}
} 
\subfloat[][Case B] { 
\includegraphics[width=0.48\linewidth]{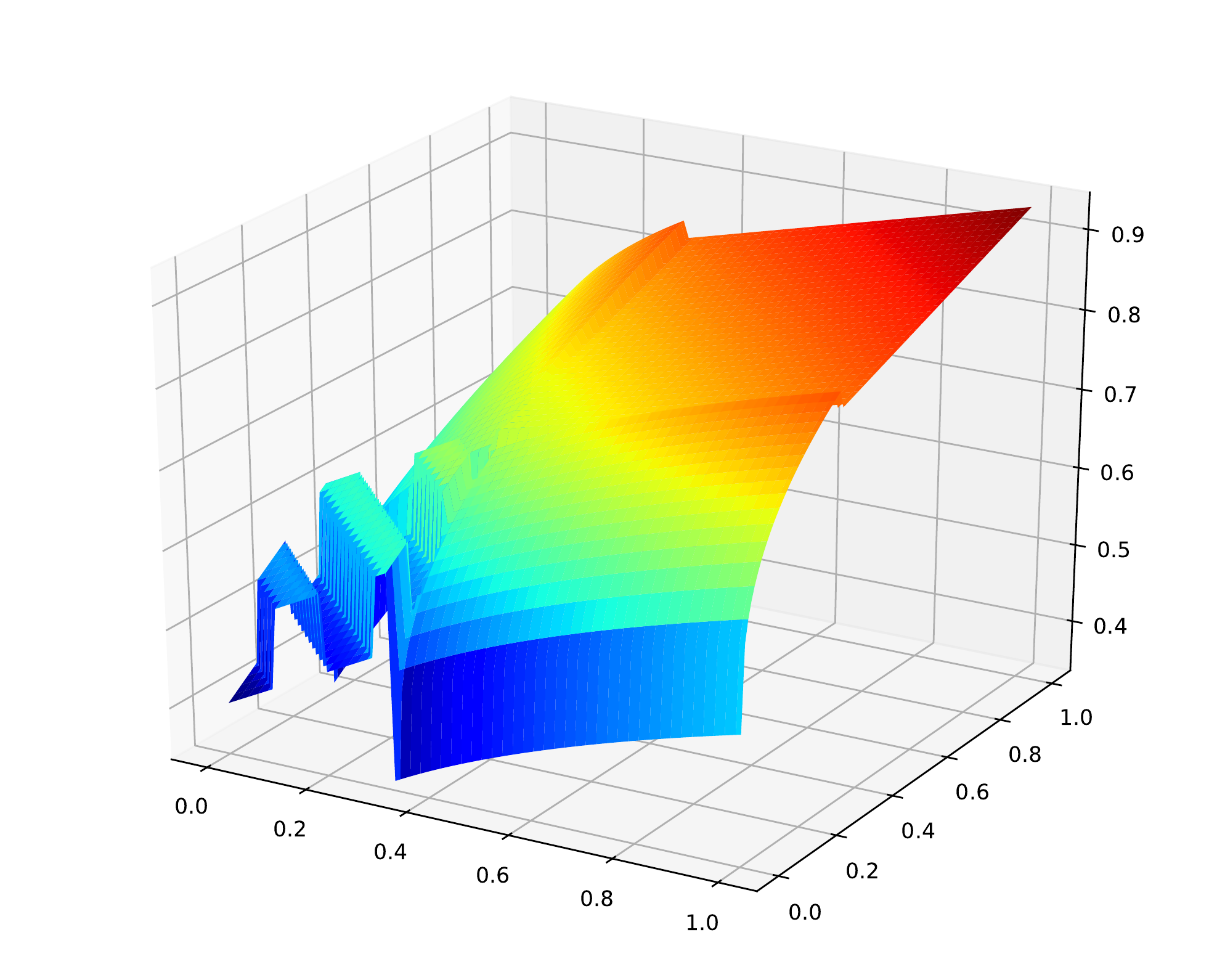} 
\label{fig::syn2dim_case2_target}
} 
\caption{Visualization of target functions. For Case A, we visualize the target function $f(x)$ (marked in blue) and one realization of training samples with sample size $1000$ (marked in red). For Case B, we only plot the surface of the target functions.
}
\end{figure*}

We visualize the target function $f(x)$ and one realization of training samples of Case A in Figure \ref{fig::exp_syn_target1} and the target function $g(x_1, x_2)$ of Case B in Figure \ref{fig::syn2dim_case2_target}.

In synthetic experiments of one-dimensional cases, we generate $1,000$ samples for training, $1,000$ samples for validation, and $10,000$ samples for test,
while in synthetic experiments of two-dimensional cases, we generate $10,000$ samples for training, $10,000$ samples for validation, and $100,000$ samples for test.

\subsubsection{Numerical Results of Synthetic Experiments}

Tables \ref{tbl::exp_syn_caseA} and \ref{tbl::exp_syn_caseB} list the averaged mean squared error of three comparing methods, including the overall MSEs and the MSEs under regions of different smooth conditions. The overall performance of ABHT is not only significantly better than PETR (1.500e-4 v.s.~2.589e-4), but also better than the global boosting version BHT (1.500e-4 v.s.~1.687e-4). It's shown that ABHT has the best performance among all competing methods.

\begin{table}[htbp]
\centering
\captionsetup{justification=centering}
\caption{Averaged Mean Squared Error on Case A}
\begin{tabular}{l|ccc}
\toprule
Domain & PEHT & BHT & ABHT \\
\midrule
$[0, 1]$ &   2.589e-04(3.300e-05) &  1.687e-04(1.343e-05) & \textbf{1.500e-04(9.557e-06)} \\
$[0, 1/8)$ &  1.220e-03(2.354e-04) &  4.631e-04(9.091e-05) &  \textbf{3.877e-04(6.975e-05)}   \\
$[1/8, 1/2)$ & 1.283e-04(9.180e-05) & 1.270e-04 (7.871e-06) &  \textbf{1.233e-04(4.998e-06)} \\
$[1/2, 1]$ & 1.145e-04(1.531e-05) & 1.259e-04(1.013e-05) &  \textbf{1.101e-04(3.737e-06)} \\
\bottomrule
\end{tabular}
\begin{minipage}{\textwidth}
\footnotesize
\begin{tablenotes}
	\item{*} The best results are marked in \textbf{bold}, and the standard deviation is reported in the parenthesis.
\end{tablenotes}
\end{minipage}
\label{tbl::exp_syn_caseA}
\end{table}

\begin{table}[htbp]
\centering
\captionsetup{justification=centering}
\caption{Averaged Mean Squared Error on Case B}
\begin{tabular}{l|ccc}
\toprule
Domain & PEHT & BHT & ABHT \\
\midrule
$[0, 1] \times [0, 1]$ &   2.320e-04(5.420e-06) &  1.956e-04(7.366e-06) & \textbf{1.662e-04(6.201e-06)} \\
$[0, 1/3] \times [0, 1/3]$ &  9.293e-04(2.758e-05) &  6.422e-04(2.963e-05) &  \textbf{5.040e-04(3.842e-05)}  \\
$[0, 1/3] \times (1/3, 1]$ & 1.630e-04(1.231e-05 ) & 1.478e-04(1.945e-05) &  \textbf{1.372e-04(6.012e-06)} \\
$(1/3, 1] \times [0, 1/3]$ & 1.622e-04(7.306e-06) & 1.435e-04(7.683e-06) &  \textbf{1.370e-04(6.763e-06)} \\
$(1/3, 1] \times (1/3, 1]$ & 1.267e-04(1.117e-05) & 1.338e-04(1.124e-05) &  \textbf{1.108e-04(3.808e-06)} \\
\bottomrule
\end{tabular}
\begin{minipage}{\textwidth}
\footnotesize
\begin{tablenotes}
	\item{*} The best results are marked in \textbf{bold}, and the standard deviation is reported in the parenthesis.
\end{tablenotes}
\end{minipage}
\label{tbl::exp_syn_caseB}
\end{table}

For Case A, from the MSE performances on different intervals we see that the main reason of performance gap lies on interval $[0, 1/8]$ which has lower order of smoothness. The MSE of PEHT on interval $[0, 1/8]$ is 1.220e-3, about three times larger than that of ABHT, which is 3.877e-4. However, PEHT performs better on large regions with higher order of smoothness. For one thing, the performance gaps on other two intervals between PEHT and ABHT are small. For another, the performance of PEHT on interval $[1/2, 1]$ with high order of smoothness is even better than that of BHT. Therefore, in this synthetic case which has significantly different smooth conditions on different regions, the PEHT fails while the proposed ABHT wins.

The performance of ABHT is consistently better than PEHT in regions with different smooth conditions. This is because a universal bandwidth $h$ in PEHT is not locally adaptive among regions with different smooth conditions: PEHT with a large $h$ cannot fit regions with low order of smooth conditions well, while PEHT with a small $h$ cannot fit regions with high order of smooth conditions well.

In the following subsection, we need to explore the inner details of the proposed ABHT. We show how the proposed ABHT performs well through the local adaptivity among different regions with different smooth conditions, and illustrate how the theoretical findings about the superiority of ABHT over PEHT match the numerical experiments.

\subsubsection{An Illustrative Example}

In order to reveal why ABHT can better fit the target function with different smoothness conditions in different regions, we take one experimental run as an example to illustrate the inner details of the ABHT algorithm. We generate 1000 points for training, 1000 points for validation, and 10000 points for test as usual.

\begin{figure}[htbp]
\centering
\vskip 0.0in
\centerline{\includegraphics[width=0.8\linewidth]{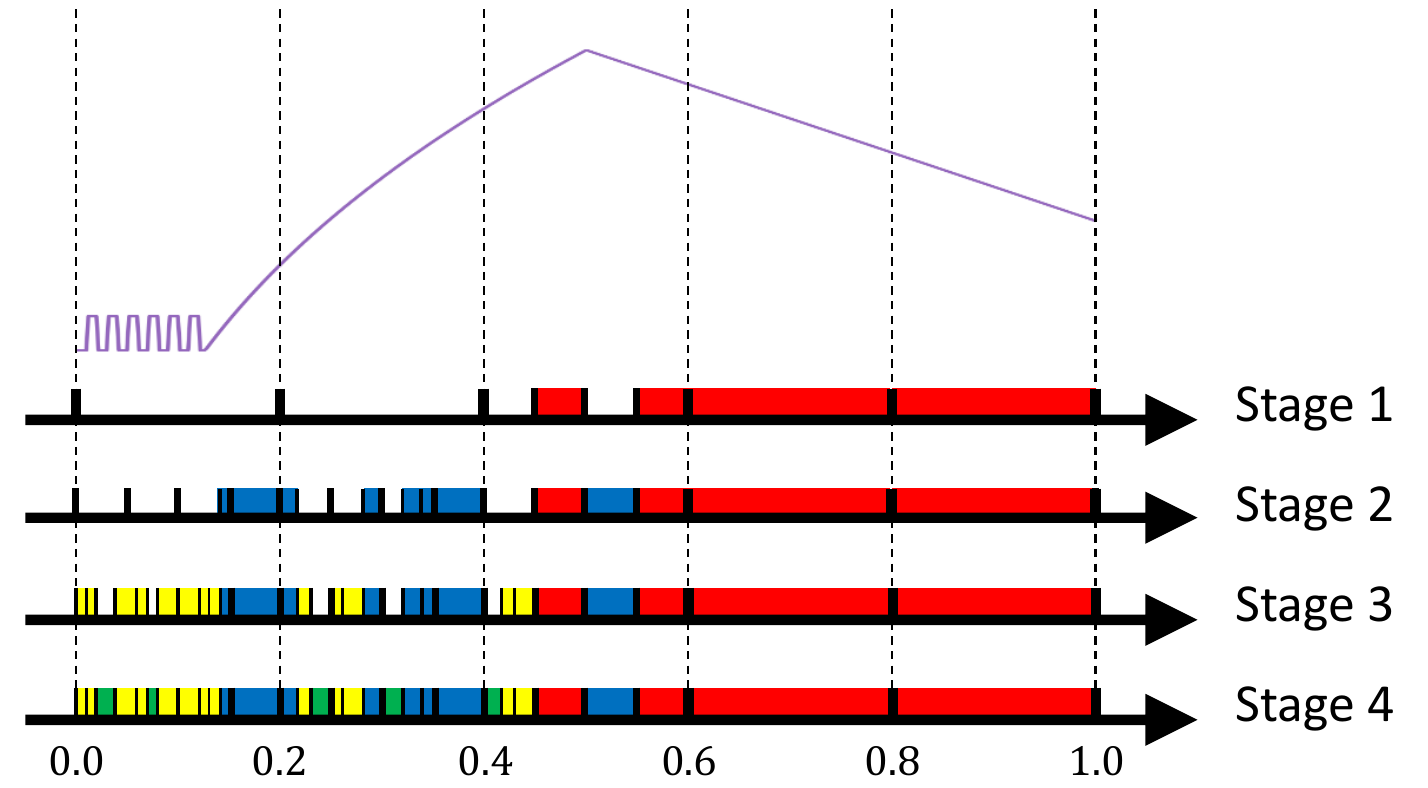}}
\vskip 0.0in
\caption{An illustrative example to show how ABHT works in the target function with different orders of smoothness.}
\label{fig::exp_visualize}
\vskip 0.0in
\end{figure}

The purple line at the top of Figure \ref{fig::exp_visualize} is the target function, in which the target function on the intervals $[0, 1/8)$, $[1/8, 1/2)$, $[1/2, 1]$ corresponds to the non-smooth region, the region with low order of smoothness, and the region with high order of smoothness, respectively. At the bottom of Figure \ref{fig::exp_visualize}, four coordinate axes with some regions marked in red, blue, yellow or green show the early stopping regions selected by the ABHT algorithm in each stage of the training process. In this run there are four stages in total. The regions marked in red, blue, yellow and green are early stopping regions selected in stage 1 to stage 4, respectively. 
\begin{itemize}
\item 
In stage 1 of the ABHT algorithm, the intervals $[0.45, 0.5]$, $[0.55, 0.6]$, $[0.6, 0.8]$, $[0.8, 1.0]$ marked in red are selected as the early stopping regions. Note that the interval $[1/2, 1]$ is the region of the highest order of smoothness, it is shown that the most smooth regions are almost covered in the first stage of ABHT. In this stage, the best band-width $h$ is $0.05$, the learning rate $\rho = 0.1$ and the number of iterations $T = 200$. We calculate the prediction of the test samples on the fitted model with only one stage, and the averaged mean squared errors on the intervals $[0, 1]$, $[0, 1/8)$, $[1/8, 1/2)$, $[1/2, 1]$ are 3.39e-04, 1.84e-03, 1.36e-04, and 1.13e-04, respectively.
\item 
In stage 2, we continue the boosting process on the sample points in the regions which are not marked in red in the first stage. Regions marked in blue are the early stopping regions in the second stage. We find that many regions with less smooth conditions are chosen. Besides, all areas in the interval $[1/2, 1]$ are early-stopped in the first two stages, while no regions in the interval $[0, 1/8]$ are selected as early stopping regions in the first two stages, which shows the ABHT algorithm can early stop regions with high order of smoothness and not stop the regions with poor smoothness at the front stage. In this stage, $h = 0.02$, $\rho = 0.2$, and $T = 100$. The fitted model with two stages are also evaluated and the averaged mean squared errors on the intervals $[0, 1]$, $[0, 1/8)$, $[1/8, 1/2)$, $[1/2, 1]$ are 3.31e-04, 1.81e-03, 1.23e-04, and 1.14e-04, respectively.
\item 
In the latter two stages, we continue the boosting process on the sample points in the regions which are not marked in red or blue in the first two stages. We continue to fit in the regions with less smoothness. Regions marked in yellow and green are the early stopping regions in the third stage and the forth stage. Regions with less smoothness are fitted with more iteration and with smaller bandwidth $h$. The best hyper-parameters in stage 3 are $h = 0.01$, $\rho = 0.2$, and $T = 200$, and the best hyper-parameters in stage 4 are $h = 0.005$, $\rho = 0.2$, and $T = 200$. The fitted model with three stages are evaluated and the averaged mean squared errors on the intervals $[0, 1]$, $[0, 1/8)$, $[1/8, 1/2)$, $[1/2, 1]$ are 1.57e-04, 4.31e-04, 1.23e-04, and 1.14e-04, respectively. The final fitted model with four stages in total are evaluated and the averaged mean squared errors on the intervals $[0, 1]$, $[0, 1/8)$, $[1/8, 1/2)$, $[1/2, 1]$ are 1.55e-04, 4.11e-04, 1.25e-04, and 1.14e-04, respectively.
\end{itemize}
The above fitting procedures in each stage illustrates the local adaptivity of the ABHT algorithm: we use few stages and a large bandwidth $h$ to fit regions with high order of smoothness, and use more stages and smaller bandwidths $h$ to fit regions with low order of smoothness. We analyze the local adaptivity of ABHT in the aspect of MSEs in regions of different smoothness conditions.  
\begin{itemize}
\item 
The MSE on the interval $[0, 1/8]$ with poor smoothness conditions are 1.84e-03, 1.81e-03, 4.31e-04, and 4.11e-04, respectively. There exists a significantly decrease in the MSE on the interval [0, 1/8], especially in stage 3 and 4. Three or four stages are needed to fit the target function with lower order of smoothness well. We need more iterations and base learners with smaller bandwidth $h$ to tackle this difficult case.
\item 
On the contrary, the MSE on the interval $[1/2, 1]$ changes little on different stages, changing from 1.13e-04 to 1.14e-04. This is because the target function on this interval is smooth with high order and is easy to fit well. The ABHT algorithm can early stop regions which are very smooth, then only use a small number of iterations and a relatively large bandwidth $h$ to fit these regions well.
\item 
Moreover, the MSE on the interval $[1/8, 1/2]$ changes from 1.36e-04 to 1.23e-04, and is finally stable at 1.25e-04, which shows that multi-stage training processing with different numbers of iterations and bandwidth $h$ are beneficial to the fitting on the interval $[1/8, 1/2]$.
\end{itemize}
For comparisons, we also take one experimental run with the same random generated samples to show the performance of PEHT.
In this run, $T = 50$ and $h = 0.01$ are cross-validated as the best hyper-parameters for all regions. 
And the performance shown by the MSEs of PEHT on the intervals $[0, 1]$, $[0, 1/8)$, $[1/8, 1/2)$, $[1/2, 1]$ are 2.43e-04, 1.15e-03, 1.12e-04, and 1.12e-04, respectively. 
The PEHT regressor with these hyper-parameters turns out to be more suitable for the intervals $[1/8,1/2)$ and $[1/2,1]$, whereas it has poor performance in the interval $[0,1/8)$.
Compared with PEHT, the superiority of ABHT attributes to the choice of different suitable bin width $h$ for regions with different smooth conditions.

\subsubsection{Impact of Training Size}\label{sec::expn}

In this part, we aim to verify the theoretical analysis in Section \ref{sec::thmcompare}. Here we use the synthetic cases described in Section \ref{sec::syndata} and run experiments with $n = 1000$, $3000$, $10000$, $30000$, and $50000$ to show the impact of training size $n$ on the performance of ABHT and PEHT.

\begin{figure*}[!h]
\centering
\captionsetup[subfigure]{justification=centering, captionskip=2pt} 
\subfloat[][Case A] { 
\includegraphics[width=0.49\linewidth]{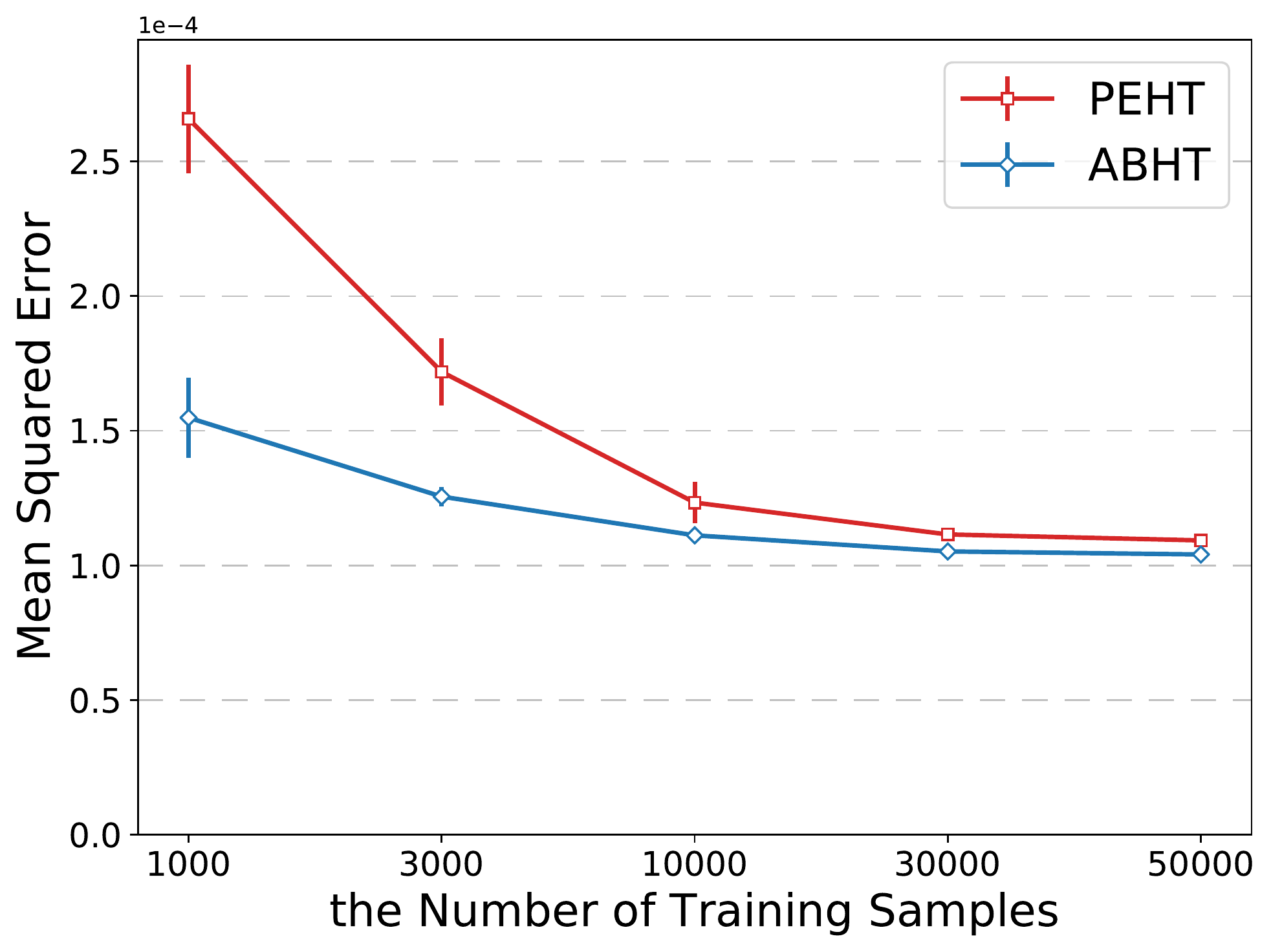} 
\label{fig::exp_trainsize_caseA}
} 
\subfloat[][Case B] { 
\includegraphics[width=0.49\linewidth]{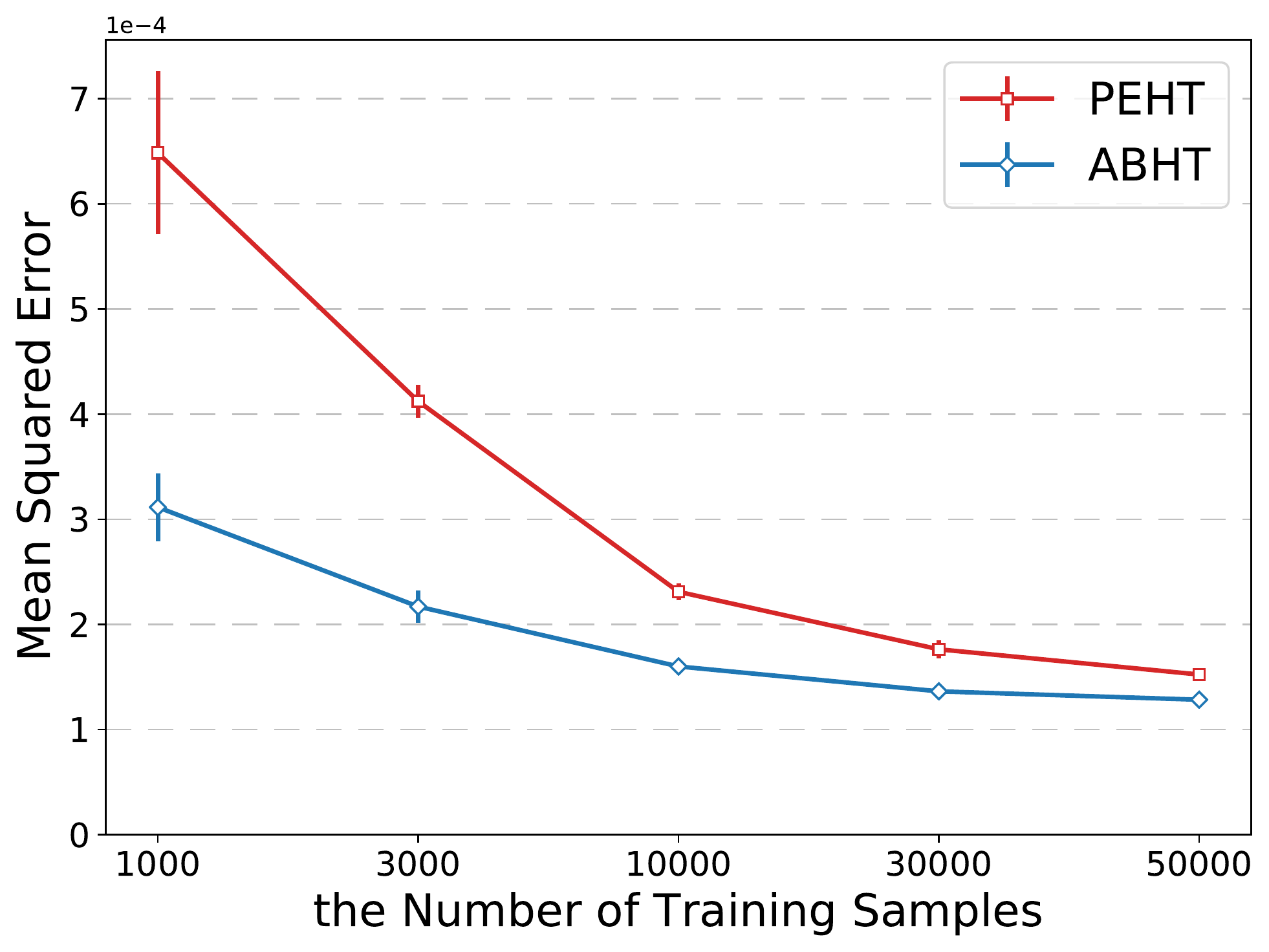} 
\label{fig::exp_trainsize_caseB}
} 
\caption{An illustrative example to show the impact of training size $n$ on the performance of ABHT and PEHT on Cases A and B.}
\end{figure*}

In Figures \ref{fig::exp_trainsize_caseA} and \ref{fig::exp_trainsize_caseB}, the blue line shows the MSE performance of ABHT and the red line represents that of PEHT. For one thing, we see that the MSE performance of both ABHT and PEHT enhances as the training size $n$ increases, and that ABHT uniformly outperforms PEHT under all $n$. However, as $n$ increases, the difference in MSE between ABHT and PEHT narrows. 
This experimental finding corresponds to the theoretical result in Theorem \ref{thm::finiten} that as the sample size $n \to \infty$, we have to let $\delta \to 0$, and thus the gap in the excess risk of PEHT and ABHT becomes insignificant.

\subsection{Real Data Experiments}

Until now, the histograms we use for boosting in Algorithm \ref{alg::ABHT} are partitioned in an equal-size bandwidth manner. Histograms are very useful in low-dimensional circumstances. However, histograms are less efficient with unacceptable and unnecessary computational costs in real-world high-dimensional cases, where the number of bins grows exponentially with the dimension $d$ and many bins will contain few or even no samples. Therefore,  we adopt the binary partitioning technique \cite{biau2012analysis} to construct the high-dimensional histograms named binary histograms. The depth of the binary histogram $p$ is the hyper-parameter that controls the number of partitions of binary histograms similar to the bin width of histograms $h$.

In the real data experiments, the histogram we use for ABHT in Algorithm \ref{alg::ABHT} is the binary histogram mentioned above. The differences between the ABHT algorithm with binary histograms and that with equal-size histograms are as follows:
\begin{itemize}
\item 
Different from Algorithm \ref{alg::ABHT} that the initial histogram partition $\mathfrak{X}_1$ is constructed by an equal-size histogram, the initial histogram partition is built up by a binary histogram partition with a sufficient large depth $P$. Correspondingly, the early stopping regions $\mathfrak{J}_l$ and the residual regions $\mathfrak{X}_l$ are composed of leaf cells of the binary histogram partition under a depth $p \in [1, P]$.
\item 
The BHT estimators $\mathfrak{f}_{\mathrm{D}, h}^l$ in each stage of the Algorithm \ref{alg::ABHT} are related to the bin width $h$, while in real data experiments, binary histograms with depth $p$ are used to build the BHT estimators $\mathfrak{f}_{\mathrm{D}, p}^l$.
\item 
The bin width parameter gird $\boldsymbol{h}$ is used for equal-size histograms, while the depth parameter grid $\boldsymbol{p}$ is used for binary histograms.
\end{itemize}
We also use the binary histograms for the comparing methods PEHT and BHT. The common hyper-parameter in real-world experiments is the depth of the binary histograms $p$.
We select the best depth $p \in \{4, 6, 8, 10, 12\}$ and best learning rate $\rho \in \{0.02, 0.05, 0.1, 0.2, 0.4\}$. In each repetition of the experiments, we randomly choose 40\% of the data set as the training set, another 40\% of the data set as the validation set, and the remaining 20\% of the data set as the test set. We standardize the datasets and repeat the real data experiments for 30 times.

\subsubsection{Descriptions of Real Data Sets}

We use five real-world datasets from the UCI machine learning repository \cite{Dua:2019} and LIBSVM Data \cite{Libsvm}. We provide the details of these data sets, including size and dimension in Table \ref{tbl::realdata_descriptions}.
\begin{table}[htbp]
\centering
\captionsetup{justification=centering}
\caption{Description over Real Data Sets}
\begin{tabular}{ccc}
\toprule
Datasets & Size &  Dimension \\
\midrule
{\tt  EGS } & $10,000$ & $12$ \\
{\tt  AEP } & $19,735$ & $27$ \\
{\tt  CAD } & $20,640$ & $8$ \\
{\tt  SCD } & $21,263$ & $81$ \\
{\tt  HPP } & $22,784$ & $8$ \\
{\tt  ONP } & $39,644$ & $58$ \\
{\tt  PTS } & $45,730$ & $9$ \\
\bottomrule
\end{tabular}
\label{tbl::realdata_descriptions}
\end{table}

\begin{itemize}
\item {\tt EGS}: The \textit{Electrical Grid Stability Simulated Data Set} ({\tt EGS}) \cite{arzamasov2018towards} is available on the UCI Machine Learning Repository. It contains $10,000$ samples in total. $12$ attributes are used to predict the maximal real part of the characteristic equation root.
\item {\tt AEP}: The \textit{Appliances Energy Prediction Data Set} ({\tt AEP}) \cite{candanedo2017data}, available on UCI Machine Learning Repository, contains $19,735$ samples of dimension $27$ with attribute ``date'' removed from the original data set. The data is used to predict the appliances energy use in a low energy building.
\item {\tt CAD}: The \textit{California Housing Prices Data Set} ({\tt CAD}) is avaliable on the LIBSVM Data. This spacial data can be traced back to \cite{Pace1997Sparse}. It consists $20,640$ observations on housing prices with $8$ economic covariates.
Note that for the sake of clarity, all house prices in the original data set has been modified to be counted in thousands.
\item {\tt SCD}: The \textit{Superconductivity Data Set} ({\tt SCD}) \cite{hamidieh2018data}, available on the UCI Machine Learning Repository, is supported by the NIMS, a public institution based in Japan. This database has $21,263$ samples with $81$ features. The goal is to predict the critical temperature based on the features extracted.
\item {\tt HPP}: The \textit{House Price Prototask Data Set} ({\tt HPP}) is originally taken from the census-house dataset in the DELVE Datasets. We use the house-price-8H prototask, which contains $22,784$ observations. We use $8$ features to predict the median house prices from $1990$ US census data. Similar as the data preprocessing for {\tt CAD}, all house prices in the original data set has been modified to be counted in thousands.
\item {\tt ONP}: The \textit{Online News Popularity Data Set} ({\tt ONP}) \cite{fernandes2015proactive}, available on the UCI Machine Learning Repository, is a database summarizing a heterogeneous set of features about articles published by Mashable in a period of two years. It contains $39,644$ observations with $58$ predictive attributes. This data set is used to predict the number of shares of the online news.
\item {\tt PTS}: \textit{Physicochemical Properties of Protein Tertiary Structure Data Set} (${\tt PTS}$) is available on the UCI Machine Learning Repository. It contains $45,730$ samples of dimension $9$. The regression task is to predict the size of the residue.
\end{itemize}

\subsubsection{Numerical Results of Real Data Experiments}

For the consideration of computational efficiency, we restrict the maximal number of stages $L$ to be $3$. Moreover, since when the dimension is relatively high, the samples prone to distribute sparsely over the input space, therefore, we can also avoid overfitting by putting a restriction on the maximal number of stages.

\begin{table}[htbp]
\centering
\captionsetup{justification=centering}
\caption{Averaged Mean Squared Error over Real Data Sets}
\begin{tabular}{l|ccc}
\toprule
Data & PEHT & BHT & ABHT \\
\midrule                                                 
{\tt EGS} & 5.8209e-4(1.7851e-5)  & 2.2675e-4(1.1677e-5)  & \textbf{2.1530e-4(1.0872e-5)} \\
{\tt SCD} & 1.3659e+2(4.0815e+0)  & \textbf{1.1841e+02(4.3743e+0)}  & 1.1880e+2(4.8246e+0) \\
{\tt ONP} & 1.2964e+2(5.3508e+1)  & 1.2904e+2(5.3372e+1)  & \textbf{1.2897e+2(5.3295e+1)} \\
{\tt CAD} & 4.2002e+3(1.4852e+2)  & 3.3737e+3(1.4554e+2)  & \textbf{3.3625e+3(1.1456e+2)} \\
{\tt PTS} & 1.8359e+1(2.6292e-1)  & 1.4502e+1(2.7630e-1)  & \textbf{1.4339e+1(2.7389e-1)} \\
{\tt AEP} & 7.6432e+3(3.6636e+2)  & \textbf{7.0670e+3(4.9574e+2)}  & 7.2562e+3(3.9800e+2) \\
{\tt HPP} & 1.6014e+3(1.1586e+2)  & \textbf{1.3843e+3(1.1008e+2)}  & 1.3982e+3(1.0214e+2) \\
\bottomrule
\end{tabular}
\begin{minipage}{\textwidth}
\footnotesize
\begin{tablenotes}
	\item{*} The best results are marked in \textbf{bold}, and the standard deviation is reported in the parenthesis.
\end{tablenotes}
\end{minipage}
\label{tbl::exp_realdata}
\end{table}

In Table \ref{tbl::exp_realdata}, we report the averaged MSEs of three comparing methods over several real data sets. Let us briefly discuss the experimental results. Firstly, the performance of ABHT consistently outperforms PEHT in all these data sets. These experimental results validate the theoretical analysis in Theorem \ref{thm::finiten} that the convergence rate of ABHT is faster than that of PEHT by $n^{10d^2 \delta/(2\alpha_{k^*}+d)^2}$ when $n < N(\delta)$, and that $\delta \to 0$ only if $n \to \infty$ and $N(\delta) \to \infty$. 
In practice, the sample size $n$ cannot reach infinity. Therefore, there exist a finite $N(\delta)$ such that Theorem \ref{thm::finiten} holds with a relatively large $\delta > 0$, i.e. the excess risk of ABHT is significantly smaller than that of PEHT. This explains the observation that the performance gap w.r.t. MSE between ABHT and PEHT is significant. 
For another, the performance of ABHT is comparable to and sometimes even better than BHT, which shows empirically that ABHT is a competent alternative of BHT and thus the theoretical results about the benefits of ABHT over PEHT should be an appropriate theoretical perspective to illustrate the advantage of sequential over parallel ensemble algorithms.

\section{Proofs} \label{sec::proofs}

\subsection{Proofs Related to ABHT}

\subsubsection{Proofs Related to Section \ref{sec::hlstar}}

To derive bounds on the sample error of regularized empirical risk minimizers, let us briefly recall the definition of VC dimension measuring the complexity of the underlying function class.

\begin{definition}[VC dimension] \label{def::VC dimension}
Let $\mathcal{B}$ be a class of subsets of $\mathcal{X}$ and $A \subset \mathcal{X}$ be a finite set. The trace of $\mathcal{B}$ on $A$ is defined by $\{ B \cap A : B \subset \mathcal{B}\}$. Its cardinality is denoted by $\Delta^{\mathcal{B}}(A)$. We say that $\mathcal{B}$ shatters $A$ if $\Delta^{\mathcal{B}}(A) = 2^{\#(A)}$, that is, if for every $A' \subset A$, there exists a $B \subset \mathcal{B}$ such that $A' = B \cap A$. For $n \in \mathrm{N}$, let
\begin{align*}
m^{\mathcal{B}}(n) := \sup_{A \subset \mathcal{X}, \, \#(A) = n} \Delta^{\mathcal{B}}(A).
\end{align*}
Then, the set $\mathcal{B}$ is a Vapnik-Chervonenkis class if there exists $n<\infty$ such that $m^{\mathcal{B}}(n) < 2^n$ and the minimal of such $n$ is called the VC dimension of $\mathcal{B}$, and abbreviate as $\mathrm{VC}(\mathcal{B})$.
\end{definition}

Since an arbitrary set of $n$ points $\{x_1,\ldots,x_n\}$ possess $2^n$ subsets, we say that $\mathcal{B}$ \textit{picks out} a certain subset from $\{ x_1, \ldots, x_n\}$ if this can be formed as a set of the form $B\cap \{x_1,\ldots,x_n\}$ for a $B\in \mathcal{B}$. The collection $\mathcal{B}$ \textit{shatters} $\{x_1,\ldots,x_n\}$ if each of its $2^n$ subsets can be picked out in this manner. From Definition \ref{def::VC dimension} we see that the VC dimension of the class $\mathcal{B}$ is the smallest $n$ for which no set of size $n$ is shattered by $\mathcal{B}$,  that is,
\begin{align*}
\mathrm{VC}(\mathcal{B}) =\inf \Bigl\{n:\max_{x_1,\ldots,x_n} \Delta^{\mathcal{B}}(\{ x_1,\ldots,x_n \})\leq 2^n\Bigr\},
\end{align*}
where $\Delta^{\mathcal{B}}(\{ x_1, \ldots,x_n \})=\#\{B\cap \{x_1,\ldots,x_n\}:B\in \mathcal{B}\}$. Clearly, the more refined $\mathcal{B}$ is, the larger is its index.

To prove Lemma \ref{VCIndex}, we need the following fundamental lemma concerning the VC dimension of purely random partitions, which follows the idea put forward by \cite{breiman2000some} of the construction of purely random forest. To this end, let $p \in \mathbb{N}$ be fixed and $\pi_p$ be a partition of $\mathcal{X}$ with number of splits $p$ and $\pi_{(p)}$ denote the collection of all partitions $\pi_p$.

\begin{lemma}\label{VCIndex}
Let $\mathcal{B}_p$ be defined by
\begin{align*} 
\mathcal{B}_p := \biggl\{ B : B = \bigcup_{j \in J} A_j, J \subset \{ 0, 1, \ldots, p \}, A_j \in \pi_p \in \pi_{(p)} \biggr\}.
\end{align*}
Then we have $\mathrm{VC}(\mathcal{B}_p) \leq d p + 2$. 
\end{lemma}

To further bound the capacity of the function sets, we need to introduce the following fundamental descriptions which enables an approximation of an infinite set by finite subsets.

\begin{proof}[Proof of  Lemma \ref{VCIndex}]
This proof is conducted from the perspective of geometric constructions. 

\begin{figure}[htbp]
\centering
\begin{minipage}[b]{0.16\textwidth}
	\centering
	\includegraphics[width=\textwidth]{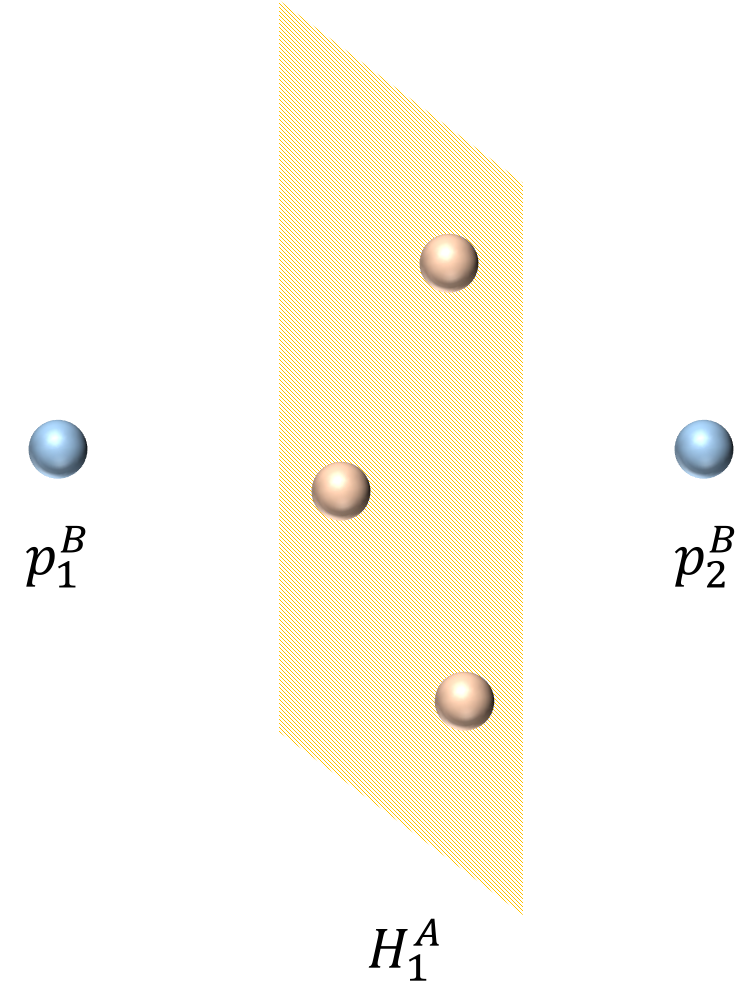}
	$p=1$
	\centering
	\label{fig::p=1}
\end{minipage}
\qquad
\begin{minipage}[b]{0.22\textwidth}
	\centering
	\includegraphics[width=\textwidth]{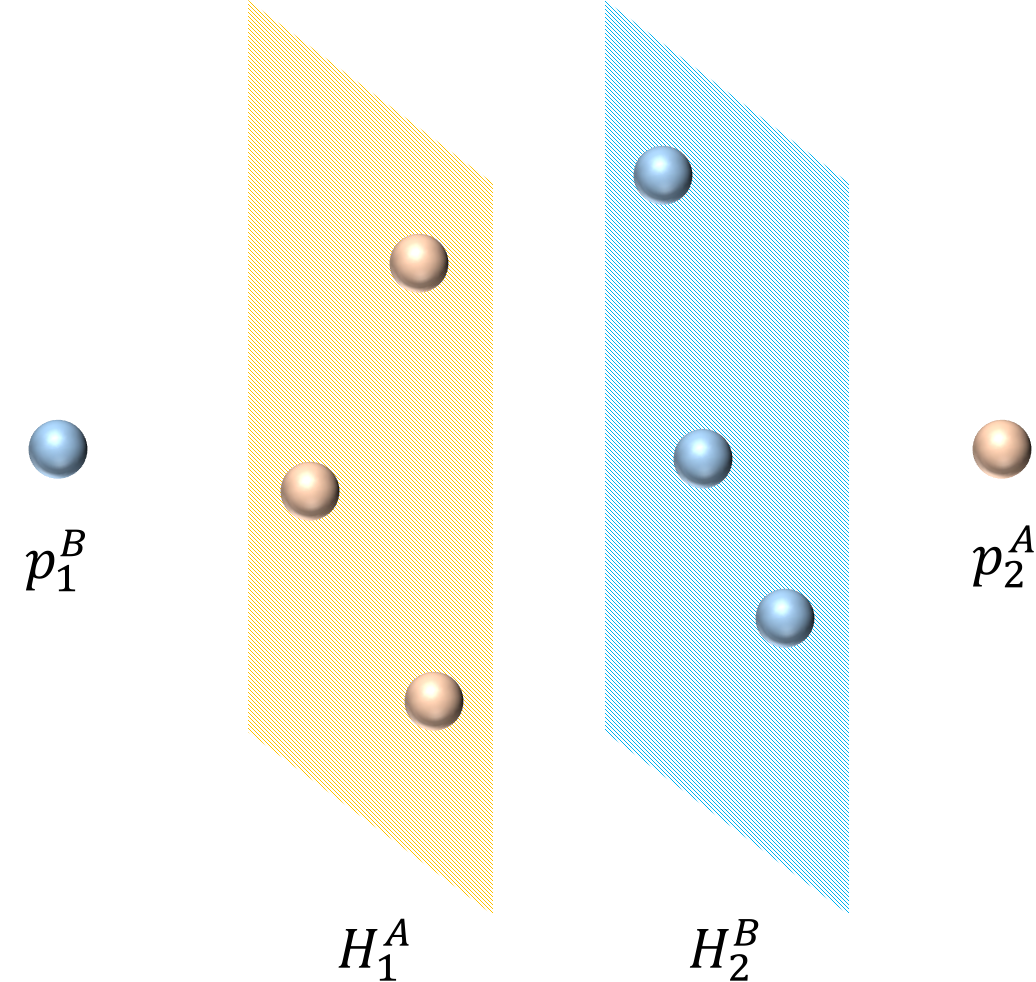}
	$p=2$
	\label{fig::p=2}
\end{minipage}
\qquad
\begin{minipage}[b]{0.4\textwidth}
	\centering
	\includegraphics[width=\textwidth]{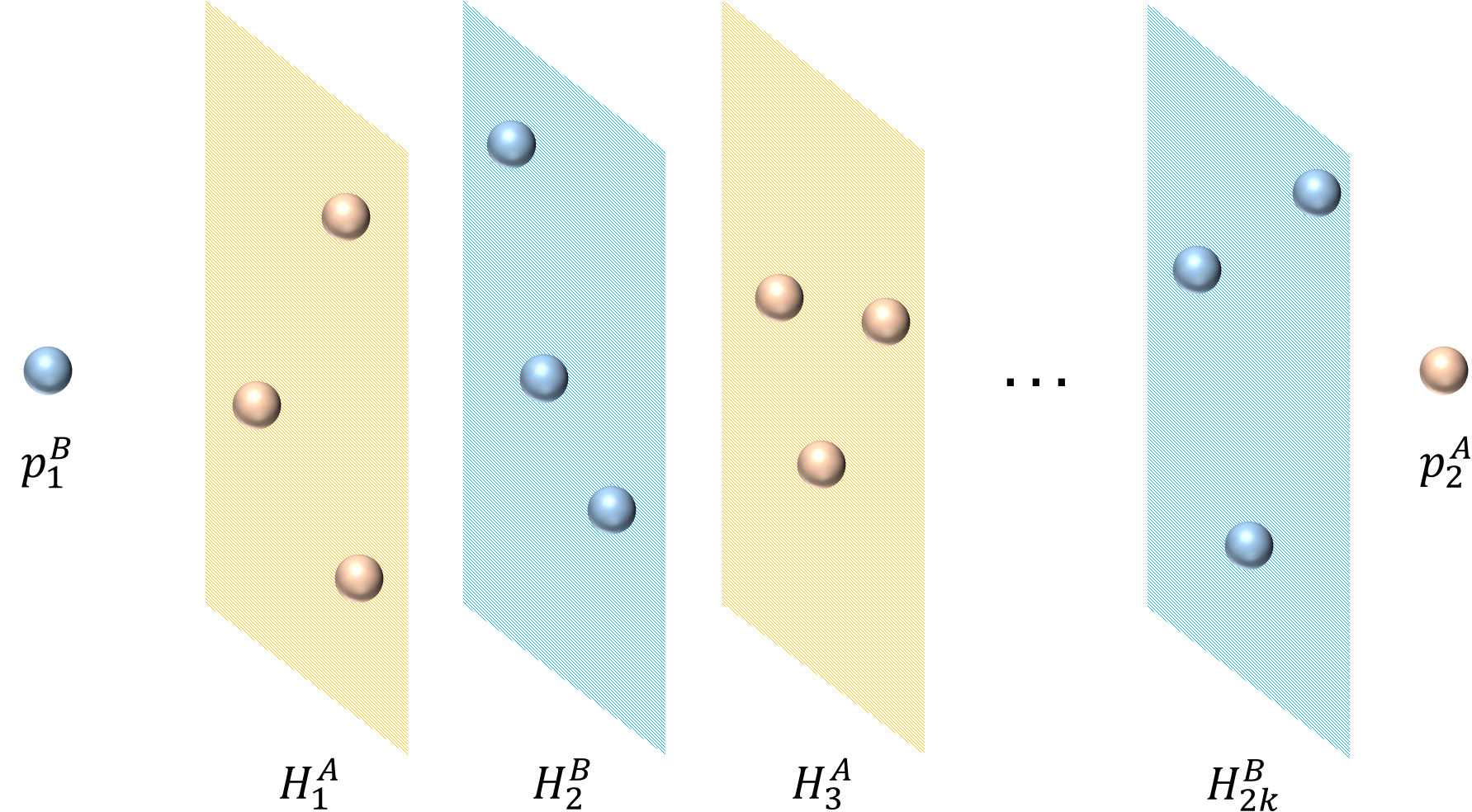}
	$p=2k$
	\label{fig::p=2k}
\end{minipage}
\caption{We take one case with $d=3$ as an example to illustrate the geometric interpretation of the VC dimension. The yellow balls represent samples from class $A$, blue ones are from class $B$ and slices denote the hyper-planes formed by samples. }
\label{fig::VC}
\end{figure}

We proceed by induction.
Firstly, we concentrate on partition with the number of splits $p=1$. Because of the dimension of the feature space is $d$,  the smallest number of sample points that cannot be divided by $p=1$ split is $d+2$. Concretely, owing to the fact that $d$ points can be used to form $d-1$ independent vectors and hence a hyperplane in a $d$-dimensional space, we might take the following case into consideration: There is a hyperplane consisting of $d$ points all from one class, say class $A$, and two points $p_1^B$, $p_2^B$ from the opposite class $B$ located on the opposite sides of this hyperplane, respectively. We denote this hyperplane by $H_1^A$. In this case, points from two classes cannot be separated by one split (since the positions are $p_1^B, H_1^A, p_2^B$), so that we have $\mathrm{VC}(\mathcal{B}_1) \leq d + 2$.

Next, when the partition is with the number of splits $p=2$, we analyze in the similar way only by extending the above case a little bit. Now, we pick either of the two single sample points located on opposite side of the $H_1^A$, and add $d-1$ more points from class $B$ to it. Then, they together can form a hyperplane $H_2^B$ parallel to $H_1^A$. After that, we place one more sample point from class $A$ to the side of this newly constructed hyperplane $H_2^B$. In this case, the location of these two single points and two hyperplanes are $p_1^B, H_1^A, H_2^B, p_2^A$. Apparently, $p=2$ splits cannot separate these $2d+2$ points. As a result, we have $\mathrm{VC}(\mathcal{B}_2) \leq 2d + 2$.

Inductively, the above analysis can be extended to the general case of number of splits $p \in \mathbb{N}$. In this manner, we need to add points continuously to form $p$ mutually parallel hyperplanes where any two adjacent hyperplanes should be constructed from different classes. Without loss of generality, we consider the case for $p=2k+1$, $k \in \mathbb{N}$, where two points (denoted as $p_1^B$, $p_2^B$) from class $B$ and $2k+1$ alternately appearing hyperplanes form the space locations: $p_1^B, H_1^A, H_2^B, H_3^A, H_4^B, \ldots, H_{(2k+1)}^A, p_2^B$. 
Accordingly, the smallest number of points that cannot be divided by $p$ splits is $dp+2$, leading to $\mathrm{VC}(\mathcal{B}_p) \leq d p + 2$. This completes the proof.
\end{proof}

To further bound the capacity of the function sets, we need to introduce the following fundamental descriptions which enables an approximation of an infinite set by finite subsets, see e.g.~\cite[Definition 6.19]{StCh08}.

\begin{definition}[Covering Numbers]\label{def::Covering Numbers}
Let $(\mathcal{X}, d)$ be a metric space, $A \subset \mathcal{X}$ and $\varepsilon > 0$. We call $A' \subset A$ an $\varepsilon$-net of $A$ if for all $x \in A$ there exists an $x' \in A'$ such that $d(x, x') \leq \varepsilon$. Moreover, the $\varepsilon$-covering number of $A$ is defined as
\begin{align*}
\mathcal{N}(A, d, \varepsilon)
& = \inf \biggl\{ n \geq 1 : \exists x_1, \ldots, x_n \in \mathcal{X},
\text{ such that } A \subset \bigcup_{i=1}^n B_d(x_i, \varepsilon) \biggr\},
\end{align*}
where $B_d(x, \varepsilon)$ denotes the closed ball in $\mathcal{X}$ centered at $x$ with radius $\varepsilon$.
\end{definition}

To investigate the capacity of continuous-valued functions, we need to introduce the concept 
\textit{VC-subgraph class}. To this end, the \emph{subgraph} of a function $f : \mathcal{X} \to \mathbb{R}$ is defined by 
$\textit{sg}(f) := \{ (x, t) : t < f(x) \}$.
A class $\mathcal{F}$ of functions on $\mathcal{X}$ is said to be a VC-subgraph class, if the collection of all subgraphs of functions in $\mathcal{F}$, denoted by $\textit{sg}(\mathcal{F}) := \{ \textit{sg}(f) : f \in \mathcal{F} \}$, is a VC class of sets in $\mathcal{X} \times \mathbb{R}$. Then the VC dimension of $\mathcal{F}$ is defined by the VC dimension of the collection of the subgraphs, that is, $\mathrm{VC}(\mathcal{F}) = \mathrm{VC}(\textit{sg}(\mathcal{F}))$.

We denote the function set $\mathcal{F}$ as
\begin{align}\label{equ::functionFH}
\mathcal{F} := \bigcup_{H \sim \mathrm{P}_H} \mathcal{F}_H,
\end{align}
which contains all the functions of $\mathcal{F}_H$ induced by histogram transforms $H$ with bin width $h_0$. The following lemma presents the upper bound for the VC dimension of the function set $\mathcal{F}$.

\begin{lemma}\label{lem::VCFn}
Let $\mathcal{F}$ be the function set defined as in \eqref{equ::functionFH}. Then $\mathcal{F}$ is a $\mathrm{VC}$-subgraph class with 
\begin{align*}
\mathrm{VC}(\mathcal{F}) 
\leq (d+1) 2^{d+1} \bigl( \lfloor \sqrt{d} / h_0 \rfloor + 1 \bigr)^d.
\end{align*}
\end{lemma}

\begin{proof}[Proof of  Lemma \ref{lem::VCFn}]
Recall that for a histogram transform $H$, the set $\pi_H =(A_j)_{j\in \mathcal{I}_H}$  is a partition of $B := [0,1]^d$ with the index set $\mathcal{I}_{H}$ induced by $H$. The choice $k := \lfloor  \sqrt{d} / h_0 \rfloor + 1$ leads to the partition of $B$ of the form $\pi_k := \{ B_{i_1, \ldots, i_d} \}_{i_j \in [k]}$ with
\begin{align} \label{def::cells}
B_{i_1, \ldots, i_d} 
:= \prod_{j=1}^d A_j
:= \prod_{j=1}^d \biggl[ \frac{i_j-1}{k}, \frac{i_j}{k} \biggr).
\end{align}
Obviously, we have $|B_{i_j}| \leq h_0 / \sqrt{d}$. Let $D$ be a data set of the form $D := \{ (x_i, t_i) : x_i \in B, t_i \in [-M, M], i = 1, \cdots, m \}$ with $m := \#(D) = 2^{d+1}(d+1) \bigl( \lfloor \sqrt{d} / h_0 \rfloor + 1 \bigr)^d$. Then there exists at least one cell $A$ with 
\begin{align} \label{DcapANo}
\#(D \cap (A\times [-M,M])) \geq 2^{d+1}(d+1).
\end{align}
Moreover, for any $x, x' \in A$, the construction of the partition \eqref{def::cells} implies $\|x - x'\| \leq h_0$. Consequently, for any arbitrary histogram transform $H$ and $A_j\in \pi_H$, at most one vertex of $A_j$ lies in $A$, since the bin width of $A_j$ is larger than $h_0$. Therefore, 
\begin{align*} 
\Pi_{H|A} 
:= \biggl\{ \bigcup_{j \in I} \bigl( (A_j \cap A) \times [-M, c_j] \bigr), I \subset \mathcal{I}_H \biggr\}
\cup \biggl\{ \bigcup_{j \in I} \bigl( (A_j \cap A) \times (c_j, M] \bigr), I \subset \mathcal{I}_H \bigg\}
\end{align*}
forms a partition of $A\times [-M,M]$ with $\#({\Pi}_{H|A}) \leq 2^{d+1}$. It is easily seen that this partition can be generated by $2^{d+1}-1$ splitting hyperplanes on the space $A\times [-M,M]$. In this way, Lemma \ref{VCIndex} implies that ${\Pi}_{H|A}$ can only shatter a dataset with at most $(d+1)(2^{d+1}-1)+1$ elements. Thus \eqref{DcapANo} indicates that ${\Pi}_{H|A}$ fails to shatter $D \cap (A\times[-M,M])$. Therefore, the subgraphs of $\mathcal{F}$, that is, $\bigl\{ \{(x,t):t<f(x)\},f\in \mathcal{F} \bigr\}$ cannot shatter the data set $D$ as well. By Definition \ref{def::VC dimension}, we immediately get $\mathrm{VC}(\mathcal{F}) \leq 2^{d+1}(d+1) \bigl( \lfloor \sqrt{d} / h_0 \rfloor + 1 \bigr)^d$ and the assertion is thus proved.
\end{proof}

Let $A := \otimes_{i=1}^d [l_i,r_i]$ be a hypercube with $r_i - l_i = r_j - l_j$ for any $i \neq j$. Then the diameter of the hypercube $A$ is given by $|A| = r_1 - l_1$. Let $\mathfrak{F}_{\mathfrak{h}_l,\mathfrak{T}_l}^l$ be the function set defined as in \eqref{eq::Glh}. The next lemma gives the upper bound of the covering number of the function space $\mathfrak{F}_{\mathfrak{h}_l,\mathfrak{T}_l|A}^l := \{ f \cdot \eins_A : f \in \mathfrak{F}_{\mathfrak{h}_l,\mathfrak{T}_l}^l \}$ when the diameter of the hypercube $A$ is larger than the bin width of base HT regressor in the $l$-th stage.

\begin{lemma}\label{lem::coverT2}
For a fixed $l \in [K]$, let $B_l$ be defined as in Assumption \ref{def::localholderP}. Furthermore, let $\mathfrak{h}_l$ and $\mathfrak{T}_l$ be the bin width and the number of iterations in the $l$-th stage of ABHT. Suppose that $A \subset B_l$ is a hypercube satisfying $|A| \geq \mathfrak{h}_l$. Moreover, for $j \in [l-1]$, let $h_{j,*}$ be the optimal bin width defined as in \eqref{eq::hlstar} and $\mathfrak{T}_{j,*}$ be the corresponding number of iteration. Then for any $\delta \in (0,1)$, $\varepsilon \in (0,1)$, and any probability measure $\mathrm{Q}$, we have
\begin{align*}
\log \mathcal{N}(\mathfrak{F}_{\mathfrak{h}_l,\mathfrak{T}_l|A}^l, \|\cdot\|_{L_2(\mathrm{Q})}, \varepsilon)  
\leq C_9 |A|^d l^{2\delta} \biggl(\sum_{j=1}^{l-1} \rho^{2\delta(l-j)} \mathfrak{T}_{j,*}  (\mathfrak{h}_{j,*})^{-d} + \mathfrak{T}_l \mathfrak{h}_l^{-d} \biggr) \varepsilon^{-2\delta},
\end{align*}
where $C_9$ is a constant only depending on $d$ and $\delta$.
\end{lemma}

\begin{proof}[Proof of  Lemma \ref{lem::coverT2}]
Recall that the function set $\mathcal{F}_{H_t}$ is induced by the histogram transform $H_t$ in the same way as in \eqref{equ::functionFn}. For any $A \subset B_l$, let $\mathcal{F}_{H_t|A} := \{  f \cdot \eins_A : f \in \mathcal{F}_{H_t} \}$. By Lemma \ref{lem::VCFn}, for any $t \in [T_{l-1} + 1, T_l]$, we have $h_t = \mathfrak{h}_l$ and thus
\begin{align*} 
\mathrm{VC} \bigl( \mathcal{F}_{H_t|A} \bigr) 
\leq 2^{d+1} (d + 1) \bigl( 2 |A| \sqrt{d} / \mathfrak{h}_l + 2 \bigr)^d
\leq 2^{d+2} d \bigl( 4 |A| \sqrt{d} / \mathfrak{h}_l \bigr)^d
= \bigl( c_d |A| / \mathfrak{h}_l \bigr)^d,
\end{align*}
where $c_d:= 2^{1+4/d} d^{1/2+1/d}$. This together with Theorem 2.6.7 in \cite{van1996weak} yields that there exists a universal constant $c_1 > 0$ such that
\begin{align*} 
\mathcal{N} \bigl( \mathcal{F}_{H_t|A}, \|\cdot\|_{L_2(\mathrm{Q})}, \varepsilon \bigr)
\leq c_1 \bigl( c_d |A| / \mathfrak{h}_l \bigr)^d \cdot (16 e)^{(c_d |A| / \mathfrak{h}_l)^d} \varepsilon^{2 (\mathfrak{h}_l / (c_d |A|))^d - 2}.
\end{align*}
Elementary calculations show that for any $\varepsilon \in (0, 1 / (e \vee K \vee c_1))$, there holds
\begin{align*}
& \log \mathcal{N} \bigl( \mathcal{F}_{H_t|A}, \|\cdot\|_{L_2(\mathrm{Q})}, \varepsilon \bigr) 
\leq \log \Bigl( c_1 \bigl( c_d |A| / \mathfrak{h}_l + 1 \bigr)^d 
(16 e)^{(c_d |A| / \mathfrak{h}_l + 1)^d} (1/\varepsilon)^{2 (c_d |A| / \mathfrak{h}_l + 1)^d - 2} \Bigr)
\\
& = \log c_1 + d \log \bigl( c_d |A| / \mathfrak{h}_l + 1 \bigr) 
+ \bigl( c_d |A| / \mathfrak{h}_l + 1)^d \log (16 e) 
+ 2 \bigl( c_d |A| / \mathfrak{h}_l + 1 \bigr)^d \log (1/\varepsilon)
\\
& \leq 16 \bigl( 2 c_d |A| / \mathfrak{h}_l \bigr)^d \log (1/\varepsilon).
\end{align*}
Consequently, for all $\delta \in (0,1)$, we have
\begin{align} \label{eq::CoverNumFH}
\sup_{\varepsilon \in (0, 1/(e \vee K))} \varepsilon^{2\delta} \log \mathcal{N} \bigl( \mathcal{F}_{H_t|A}, \|\cdot\|_{L_2(\mathrm{Q})}, \varepsilon \bigr) 
\leq 16 \bigl( 2 c_d |A| / \mathfrak{h}_l \bigr)^d
\sup_{\varepsilon \in (0, 1)} \varepsilon^{2\delta} \log (1/\varepsilon).
\end{align}
Maximizing the right-hand side of \eqref{eq::CoverNumFH} w.r.t.~$\varepsilon$, we obtain
\begin{align}\label{eq::logNFtl}
\log \mathcal{N} \bigl( \mathcal{F}_{H_t|A}, \|\cdot\|_{L_2(\mathrm{Q})}, \varepsilon \bigr) 
\leq (16/(2e\delta)) (2 c_d |A| / \mathfrak{h}_l)^d \varepsilon^{-2\delta},
\end{align}
where the maximum is attained at $\varepsilon^* = e^{-1/(2\delta)}$.

Now, we define a function set $\mathfrak{F}_{\mathfrak{h}_l}^l$ whose element is a linear combination of $\mathfrak{T}_l$ base learners with the same bin width $\mathfrak{h}_l$, i.e.
\begin{align} \label{eq::Flhl}
\mathfrak{F}_{\mathfrak{h}_l}^l 
:= \biggl\{ f = \sum_{t=T_{l-1}+1}^{T_l} w_t f_t : f_t \in \mathcal{F}_{H_t}, h_t = \mathfrak{h}_l, t \in [T_{l-1} + 1, T_l] \biggr\}.
\end{align}
For $t \in [T_{l-1} + 1, T_l]$, let $\{ g_{t,j} : j \in [m_l] \} \subset \mathcal{F}_{H_t|A}$ be the $\varepsilon$-net of $\mathcal{F}_{H_t|A}$ with $m_l := \mathcal{N}(\mathcal{F}_{H_t|A}, \|\cdot\|_{L_2(\mathrm{Q})}, \varepsilon)$. Let $\mathcal{F}_{\mathfrak{h}_l|A}^l  := \{ f \cdot \eins_A : f \in \mathfrak{F}_{\mathfrak{h}_l}^l  \}$. By the definition of $\mathfrak{F}_{\mathfrak{h}_l}^l$, we see that for any $g \in \mathcal{F}_{\mathfrak{h}_l | A}^l$, there exist $w_t$ and $g_t \in \mathcal{F}_{H_t|A}$, $t \in [T_{l-1} + 1, T_l]$ such that 
\begin{align*}
g = \sum_{t=T_{l-1}+1}^{T_l} w_t g_t 
= \frac{1}{\mathfrak{T}_l} \sum_{t=T_{l-1}+1}^{T_l} \mathfrak{T}_l w_t g_t.
\end{align*}
Let $g'_t := \mathfrak{T}_l w_t g_t$, then we have $g'_t \in \mathcal{F}_{H_t|A}$ and $g = \frac{1}{\mathfrak{T}_l} \sum_{t=T_{l-1}+1}^{T_l} g'_t$. According to the definition of the $\varepsilon$-net, there exists some index $j \in [m_l]$ such that $\|g'_t - g_{t,j}\|_{L_2(Q)}  \leq  \varepsilon$. Therefore, for any $g \in \mathfrak{F}_{\mathfrak{h}_l|A}$, there holds
\begin{align*}
\biggl\| g - \frac{1}{\mathfrak{T}_l} \sum_{t=T_{l-1}+1}^{T_l} g_{t,j} \biggr\|_2 
= \biggl\| \frac{1}{\mathfrak{T}_l} \sum_{t=T_{l-1}+1}^{T_l} (g'_t - g_{t,j}) \biggr\|_2 
\leq \biggl( 2 \cdot \frac{1}{\mathfrak{T}_l} \sum_{t=T_{l-1}+1}^{T_l} \|g'_t - g_{t,j} \|_2 \biggr)^{\frac{1}{2}}
\leq 2 \varepsilon.
\end{align*}
Consequently, the function set $\mathcal{G}_l := \bigl\{ \frac{1}{\mathfrak{T}_l} \sum_{t=T_{l-1}+1}^{T_l} g_{t,j} : j \in [m_l] \bigr\}$ is a $2\varepsilon$-net of $\mathfrak{F}_{\mathfrak{h}_l,\mathfrak{T}_l|A}$ and $\#(\mathcal{G}_l) = \prod_{t=T_{l-1}+1}^{T_l} m_l = m_l^{\mathfrak{T}_l}$. Therefore, for any probability distribution $\mathrm{Q}$, we have
\begin{align}
\log \mathcal{N}(\mathfrak{F}_{\mathfrak{h}_l|A}^l, & \|\cdot\|_{L_2(\mathrm{Q})}, 2\varepsilon) 
\leq \log \biggl( \prod_{t=T_{l-1}+1}^{T_l} \mathcal{N}(\mathcal{F}_{H_t|A}, \|\cdot\|_{L_2(\mathrm{Q})}, \varepsilon) \biggr)
\nonumber\\
& = \log \Big(\mathcal{N}(\mathcal{F}_{H_{T_l}|A}, \|\cdot\|_{L_2(\mathrm{Q})}, \varepsilon)^{\mathfrak{T}_l}\Big)
\leq \mathfrak{T}_l \cdot 16/(2e\delta) (2 c_d |A| /\mathfrak{h}_l)^d \varepsilon^{-2\delta},
\label{eq::logNFhll}
\end{align}
where the last inequality is due to \eqref{eq::logNFtl}. By the definition of the function sets $\mathfrak{F}_{\mathfrak{h}_l,\mathfrak{T}_l}^l$ and $\mathfrak{F}_{\mathfrak{h}_l}^l$ in \eqref{eq::Glh} and \eqref{eq::Flhl}, respectively, we see that for any $\mathfrak{f} \in \mathfrak{F}_{\mathfrak{h}_l,\mathfrak{T}_l}^l$, there exist $\mathfrak{f}_{\mathrm{D}}^l \in \mathfrak{F}_{\mathfrak{h}_l}^l$ and $\mathfrak{f}_{\mathrm{D}}^j \in  \mathfrak{F}_{\mathfrak{h}_{j,*}}^j$, $j \in [l-1]$, such that
\begin{align*}
\mathfrak{f}
& = \mathfrak{f}_{\mathrm{D}|\mathfrak{X}_l}^l + \rho \cdot \mathfrak{f}_{\mathrm{D},\mathrm{B}|\mathfrak{X}_l}^{l-1}
= \Bigl( \mathfrak{f}_{\mathrm{D}|\mathfrak{X}_l}^l + \rho \bigl( \mathfrak{f}_{\mathrm{D}|\mathfrak{X}_l}^{l-1} + \rho \cdot \mathfrak{f}_{\mathrm{D},\mathrm{B}|\mathfrak{X}_l}^{l-2} \bigr) \Bigr) 
\\
& = \Bigl(\mathfrak{f}_{\mathrm{D}|\mathfrak{X}_l}^l + \bigl( \rho \cdot \mathfrak{f}_{\mathrm{D}|\mathfrak{X}_l}^{l-1} + \rho^2 \cdot \mathfrak{f}_{\mathrm{D}|\mathfrak{X}_l}^{l-2} + \cdots + \rho^{l-1} \cdot \mathfrak{f}_{\mathrm{D}|\mathfrak{X}_l}^1 \bigr)\Bigr) 
= \sum_{j=1}^l\rho^{l-j} \mathfrak{f}_{\mathrm{D}|\mathfrak{X}_l}^j.
\end{align*}
Here, the recursion formula follows from the iterative construction of the ABHT algorithm. Therefore, we have
\begin{align} \label{eq::lFhlTlASSFlhlA}
\mathfrak{F}_{\mathfrak{h}_l,\mathfrak{T}_l|A}^l 
\subset \sum_{j=1}^{l-1} \rho^{l-j} \mathfrak{F}_{\mathfrak{h}_{j,*}|A}^j + \mathfrak{F}_{\mathfrak{h}_l|A}^l.
\end{align}
This together with \eqref{eq::logNFhll} yields that for any probability distribution $\mathrm{Q}$, there holds
\begin{align*}
& \log \mathcal{N} \bigl( \mathfrak{F}_{\mathfrak{h}_l,\mathfrak{T}_l|A}^l, \|\cdot\|_{L_2(\mathrm{Q})}, \varepsilon \bigr)
\\
& \leq \log \biggl( \prod_{j=1}^{l-1} \mathcal{N} \bigl( \rho^{l-j} \mathfrak{F}_{\mathfrak{h}_{j,*}|A}^j, \|\cdot\|_{L_2(\mathrm{Q})}, \varepsilon/l \bigr)\cdot  \mathcal{N} \bigl(\mathfrak{F}_{\mathfrak{h}_l|A}^l, \|\cdot\|_{L_2(\mathrm{Q})}, \varepsilon/l \bigr)\biggr)
\\
& = \sum_{j=1}^{l-1} \log\mathcal{N} \bigl( \mathfrak{F}_{\mathfrak{h}_{j,*}|A}^j, \|\cdot\|_{L_2(\mathrm{Q})}, \rho^{j-l} \varepsilon / l \bigr) + \log\mathcal{N} \bigl( \mathfrak{F}_{\mathfrak{h}_l|A}^l, \|\cdot\|_{L_2(\mathrm{Q})}, \varepsilon / l \bigr)
\\
& \leq C_9 |A|^dl^{2\delta} \biggl(\sum_{j=1}^{l-1} \rho^{2\delta(l-j)} \mathfrak{T}_{j,*}  (\mathfrak{h}_{j,*})^{-d} + \mathfrak{T}_l\mathfrak{h}_l^{-d} \biggr) \varepsilon^{-2\delta},
\end{align*}
where $C_9 := 3(2c_d)^d\delta^{-1}$. Therefore, we finished the proof.
\end{proof}

Next, let us recall the entropy numbers, which can be considered as the ``inverse'' concept of the covering numbers, see e.g.~\cite[Definition 6.20]{StCh08}.

\begin{definition}[Entropy Numbers] \label{def::entropy numbers}
Let $(\mathcal{X}, d)$ be a metric space, $A \subset \mathcal{X}$ and $i \geq 1$ be an integer. The $i$-th entropy number of $(A, d)$ is defined as
\begin{align*}
e_i(A, d) 
= \inf \biggl\{ \varepsilon > 0 : \exists x_1, \ldots, x_{2^{i-1}} \in \mathcal{X} 
\text{ such that } A \subset \bigcup_{j=1}^{2^{i-1}} B_d(x_j, \varepsilon) \biggr\}.
\end{align*}
\end{definition}
For a finite set $D \in \mathcal{X}^n$, we define the norm of an empirical $L_2$-space by
\begin{align*}
\|f\|^2_{L_2(\mathrm{D})}
= \mathbb{E}_{\mathrm{D}} |f|^2
:= \frac{1}{n} \sum_{i=1}^n |f(x_i)^2|.
\end{align*}

In order to present the following oracle inequality for ABHT at the $l$-th stage which holds with restriction on the hypercube $A$, we define the approximation error function by
\begin{align}\label{equ::approximationerror}
a_A(\lambda_l) 
:= \inf_{\mathfrak{h}_l,\mathfrak{T}_l} \lambda_{1,l} \mathfrak{h}_l^{-2d} + \lambda_{2,l} \mathfrak{T}_l^p  + \mathcal{R}_{L_A,\mathrm{P}}(\mathfrak{f}_{\mathrm{D},\mathfrak{h_l},\mathfrak{T}_l}^l) - \mathcal{R}^*_{L_A,\mathrm{P}}.
\end{align}

\begin{proposition}\label{prop::secondOracalBoost}
For a fixed $l \in [K]$, let $B_l$ be defined as in Assumption \ref{def::localholderP}. Furthermore, let $\mathfrak{h}_l$ and $\mathfrak{T}_l$ be the bin width and the number of iterations in the $l$-th stage of ABHT. Let $\mathfrak{f}_{\mathrm{D}, \mathfrak{h_l}, \mathfrak{T}_l}^l$ be the ABHT regressor defined in \eqref{equ::fdtlhl} and $a_A(\lambda_l)$ be the corresponding approximation error defined by \eqref{equ::approximationerror}. For $j \in [l-1]$, let $h_{j,*}$ be the optimal bin width defined as in \eqref{eq::hlstar} and $\mathfrak{T}_{j,*}$ be the corresponding number of iteration. If $\text{diam}(A) \geq \mathfrak{h}_l$, then for all $\tau>0$, with probability $\mathrm{P}^n$ not less than $1-3e^{-\tau}$, there holds
\begin{align*}
& \lambda_{1,l} \mathfrak{h}_l^{-2d} + \lambda_{2,l} \mathfrak{T}_l^p  + \mathcal{R}_{L_A,\mathrm{P}}(\mathfrak{f}_{\mathrm{D},\mathfrak{h_l},\mathfrak{T}_l}^l) - \mathcal{R}_{L_A,\mathrm{P}}^*
\\
& \leq 12 a_A(\lambda_l) + 3456 M^2 \tau / n 
+ 3C_{10} \biggl( \biggl( \bigvee_{j=1}^{l-1} \rho^{\frac{2\delta(l-j)}{1+\delta}} |A|^{\frac{d}{1+\delta}} \mathfrak{h}_{j,*}^{-\frac{d}{1+\delta}} \mathfrak{T}_{j,*}^{-\frac{1}{1+\delta}} 
n^{-\frac{1}{1+\delta}} \biggr)  
\\
&\qquad \qquad \qquad \qquad \qquad \qquad \qquad\qquad
\vee \biggl( \lambda_{1,l}^{-\frac{p}{p-2+2p\delta}} \lambda_{2,l}^{-\frac{2}{p-2+2p\delta}} n^{-\frac{2p}{p-2+2p\delta}} |A|^{\frac{2pd}{p-2+2p\delta}} \biggr) \biggr),
\end{align*}
where $C_{10}$ is a constant only depending on $\delta$, $M$, $l$ and $d$.
\end{proposition}

\begin{proof}[Proof of  Proposition \ref{prop::secondOracalBoost}]
Denote $r^* := \Omega_{\lambda_l}(f) + \mathcal{R}_{L_A,\mathrm{P}}(f) - R^*_{L_A,\mathrm{P}}$, and for $r > r^*$, write
\begin{align*}
\mathcal{F}_r^l & := \{ f \in \mathfrak{F}_{\mathfrak{h}_l,\mathfrak{T}_l|A}^l : \Omega(f) + \mathcal{R}_{L_{A},\mathrm{P}}(f) - \mathcal{R}^*_{L_{A},\mathrm{P}} \leq r \},
\\
\mathcal{H}_r^l & := \{ L_{A} \circ f - L_{A} \circ f^*_{L,\mathrm{P}} : f \in \mathcal{F}_r^l \}.
\end{align*}
Note that for $f \in \mathcal{F}_r^l$, we have $\lambda_{2,l} \mathfrak{T}_l^p \leq r$ and $\lambda_{1,l} \mathfrak{h}_l^{-2d} \leq r$, that is,
\begin{align}\label{eq::Thcond2}
\mathfrak{T}_l\leq \big(r/\lambda_{2,l}\big)^{1/p}
\quad
\text{ and }
\quad
\mathfrak{h}_l^{-d} \leq(r/\lambda_{1,l})^{1/2}.
\end{align}
Consequently, we have $\mathcal{F}_r^l \subset \mathfrak{F}_{\mathfrak{h}_l,\mathfrak{T}_l | A}^l$ with $\mathfrak{T}_l$ and $\mathfrak{h}_l$ satisfying \eqref{eq::Thcond2}. Exercise 6.8 in \cite{StCh08} yields
\begin{align}
\ln \mathcal{N}(T, d, \varepsilon) < (a/\varepsilon)^q, \quad \forall \, \varepsilon > 0 
\quad \Longrightarrow 
\quad e_i(T, d) \leq 3^{1/q} a i^{-1/q}, \quad \forall \, i \geq 1.
\label{CoverEntropy2}
\end{align}
Then \eqref{CoverEntropy2} together with Lemma \ref{lem::coverT2} yields
\begin{align}\label{eq::entropyT2}
e_i(\mathfrak{F}_{\mathfrak{h}_l,\mathfrak{T}_l|A}^l, d) 
\leq \biggl( 3 C_9 l^2 |A|^d \biggl( \sum_{j=1}^{l-1} \rho^{2\delta(l-j)} \mathfrak{T}_{j,*}  \mathfrak{h}_{j,*}^{-d} + \mathfrak{T}_l\mathfrak{h}_l^{-d} \biggr) \biggr)^{1/2\delta} i^{-1/2\delta}, 
\quad 
\forall \, i \geq 1,
\end{align}
where $\delta\in (0,1)$. Since the least squares loss $L$ is Lipschitz continuous with Lipschitz constant $|L|_1 \leq 4M$, we find
\begin{align*}
e_i(\mathcal{H}_r^l, L_2(\mathrm{D}))
& \leq 4 M e_i(\mathcal{F}_r^l, L_2(\mathrm{D}))
\leq 4 M e_i(\mathfrak{F}_{\mathfrak{h}_l,\mathfrak{T}_l|A}^l, L_2(\mathrm{D}))
\\
& \leq 4 M \biggl( 3 C_9 l^2 |A|^d \biggl( \sum_{j=1}^{l-1} \rho^{2\delta(l-j)} \mathfrak{T}_{j,*}  \mathfrak{h}_{j,*}^{-d} + \mathfrak{T}_l\mathfrak{h}_l^{-d} \biggr) \biggr)^{\frac{1}{2\delta}} i^{-\frac{1}{2\delta}}
\\
&\leq 4 M \big(3C_9l^2|A|^d\big)^{\frac{1}{2\delta}}
\biggl( \sum_{j=1}^{l-1} \rho^{2\delta(l-j)} \mathfrak{T}_{j,*} \mathfrak{h}_{j,*}^{-d}  +    (r/\lambda_{1,l})^{\frac{1}{2}} (r/\lambda_{2,l})^{\frac{1}{p}} \biggr)^{\frac{1}{2\delta}} i^{-\frac{1}{2\delta}},
\end{align*}
where the last two inequalities follow from \eqref{eq::entropyT2} and \eqref{eq::Thcond2}, respectively. Taking expectation with respect to $\mathrm{P}^n$, we get
\begin{align*}
\mathbb{E}_{\mathrm{P}^n} e_i(\mathcal{H}_r^l,L_2(\mathrm{D}))
\leq c_1|A|^{\frac{d}{2\delta}} \biggl( \sum_{j=1}^{l-1} \rho^{2\delta(l-j)} \mathfrak{T}_{j,*}  \mathfrak{h}_{j,*}^{-d} + (r/\lambda_{1,l})^{\frac{1}{2}} (r/\lambda_{2,l})^{\frac{1}{p}} \biggr)^{\frac{1}{2\delta}} i^{-\frac{1}{2\delta}},
\end{align*}
where $c_1 := 4 M (3C_9 l^2)^{1/2\delta}$. For least squares loss, the superemum bound $L_A(x,y,t) \leq 4 M^2$ holds for all $(x, y) \in \mathcal{X} \times \mathcal{Y}$, $t \in [-M, M]$, and the variance bound $\mathbb{E}(L_{A} \circ g - L_{A} \circ f_{L,\mathrm{P}}^*)^2 \leq V (\mathbb{E}(L_{A} \circ g - L_{A} \circ f^*_{L_{A},\mathrm{P}}))^{\vartheta}$ holds for $V := 16 M^2$ and $\vartheta := 1$. Therefore, for $h \in \mathcal{H}_r^l$, we have $\|h\|_{\infty} \leq 8 M^2$ and $\mathbb{E}_{\mathrm{P}} h^2 \leq 16 M^2 r$. Then Theorem 7.16 in \cite{StCh08} with 
$a := c_1 |A|^{d/(2\delta)} \bigl( \sum_{j=1}^{l-1} \rho^{l-j} \mathfrak{T}_{j,*}^{1/(2\delta)} \mathfrak{h}_{j,*}^{-d/(2\delta)} + (r /\lambda_{1,l})^{1/(4\delta)} (r/\lambda_{2,l})^{1/(2p\delta)} \bigr)$
yields that there exists a constant $c_{\delta} > 0$ depending on $\delta$ such that
\begin{align*}
&\mathbb{E}_{\mathrm{P}^n} \mathrm{Rad}_D(\mathcal{H}_r^l,n)
\\
& \leq c_{\delta} (c_1 l)^{2\delta} 
\biggl( \biggl( \bigvee_{j=1}^{l-1}
\Bigl( |A|^{\frac{d}{2}}\rho^{(l-j)\delta} \mathfrak{T}_{j,*}^{\frac{1}{2}} \mathfrak{h}_{j,*}^{-\frac{d}{2}} n^{-\frac{1}{2}}r^{\frac{1-\delta}{2}} \Bigr) 
\vee \Bigl( |A|^{\frac{d}{1+\delta}}\rho^{\frac{2\delta(l-j)}{1+\delta}} \mathfrak{T}_{j,*}^{\frac{1}{1+\delta}} \mathfrak{h}_{j,*}^{-\frac{d}{1+\delta}} n^{-\frac{1}{1+\delta}} \Bigr) \biggr)
\\
& \qquad 
\vee \Bigl( r^{\frac{3p+2}{4p}-\frac{\delta}{2}} \lambda_{1,l}^{-\frac{1}{4}} \lambda_{2,l}^{-\frac{1}{2p}} n^{-\frac{1}{2}} |A|^{\frac{d}{2}} \Bigr) \vee 
\Bigl( r^{\frac{p+2}{2p(\delta+1)}} \lambda_{1,l}^{-\frac{1}{2(1+\delta)}} \lambda_{2,l}^{-\frac{1}{p(1+\delta)}} n^{- \frac{1}{1+\delta}} |A|^{\frac{d}{1+\delta}} \Bigr) \biggr)
=: c_2 \varphi_n(r),
\end{align*}
where $c_2:= c_\delta (c_1l)^{2\delta}$. Simple algebra shows that the condition $\varphi_n(4r) \leq 2\sqrt{2} \varphi_n(r)$ is satisfied. Since $2\sqrt{2} < 4$, similar arguments show that there still hold the statements of the Peeling Theorem 7.7 in \cite{StCh08}. Consequently, Theorem 7.20 in \cite{StCh08} can also be applied, if the assumptions on $\varphi_n$ and $r$ are modified to $\varphi_n(4r) \leq 2\sqrt{2} \varphi_n(r)$ and $r \geq (75 \varphi_n(r)) \vee (1152 M^2\tau/n) \vee r^*$, respectively. It is easy to verify that the condition is satisfied if
\begin{align*}
r & \geq 75c_2
\biggl( \Bigl( \bigvee_{j=1}^{l-1} \rho^{\frac{2\delta(l-j)}{1+\delta}}|A|^{\frac{d}{1+\delta}}(h_{j,*})^{-\frac{d}{1+\delta}} (\mathfrak{T}_{j,*})^{\frac{1}{1+\delta}} 
n^{-\frac{1}{1+\delta}} \Bigr)
\\
& \qquad \qquad 
\vee \Bigl( \lambda_{1,l}^{-\frac{p}{p-2+2p\delta}} \lambda_{2,l}^{-\frac{2}{p-2+2p\delta}} n^{-\frac{2p}{p-2+2p\delta}}|A|^{\frac{2pd}{p-2+2p\delta}} \Bigr) \biggr) \vee \frac{1152M^2 \tau}{n}
\end{align*}
holds with probability at least $1-3e^{-\tau}$. With $C_{10} := 75c_2$ we finish the proof.
\end{proof}

In the following, for each $l$ and $j\in \mathfrak{J}_l$, we will bound the approximation error on the hypercube $A_{l,j}$. The following Lemma presents the explicit representation of the histogram cell $A_H(x)$ which will be used later in the proofs of Proposition \ref{prop::biasterm}.

\begin{lemma}\label{binset}
Let the histogram transform $H$ be defined as in \eqref{equ::HT} and $A'_H$, $A_H$ be as in \eqref{equ::InputBin} and \eqref{TransBin}, respectively. Then for any $x \in \mathbb{R}^d$, the set $A_H(x)$ can be represented as 
\begin{align*}
A_H(x) = \bigl\{ x + (s R)^{-1} z :  z \in [-b', 1 - b'] \bigr\},
\end{align*}
where $b' \sim \mathrm{Unif}(0, 1)^d$.
\end{lemma}

\begin{proof}[Proof of  Lemma \ref{binset}]
For any $x \in \mathbb{R}^d$, we define $b' := H(x) - \lfloor H(x) \rfloor \in \mathbb{R}^d$. Then we have $b' \sim \mathrm{Unif}(0,1)^d$ according to the definition of $H$. For any $x' \in A'_H(x)$, we define $z := H(x') - H(x) = (s R) (x' - x)$. Then we have $x' = x + (s R)^{-1} z$. Moreover, since $\lfloor H(x') \rfloor = \lfloor H(x) \rfloor$, we have $z \in [-b', 1 - b']$. 
\end{proof}

The following proposition establishes the pointwise approximation error of $f_{\mathrm{P},\mathrm{E}}$ which combines the base learners with the same bin width under the ordinary H\"older assumption.

\begin{proposition}\label{prop::biasterm}
Let the histogram transform $H_t$ be defined as in \eqref{equ::HT} with bin widths $h_t$. Assume that all bin widths $h_t$ have the same bin width $h_{0}$. Furthermore, let $\mathrm{P}_X$ be uniform distribution and $f_{L,\mathrm{P}}^*\in C^{\alpha}$ with the H\"older exponent $\alpha\in(0,1]$ and the constant $c_L$. Then we have
\begin{align*}
\mathbb{E}_{\mathrm{P}_H} \bigl( f_{\mathrm{P},\mathrm{E}}(x) - f_{L,\mathrm{P}}^*(x) \bigr)^2
\leq d c_L^2 h_0^{2\alpha} + T^{-1} \cdot d c_L^2 h_0^2.
\end{align*}
\end{proposition}

\begin{proof}[Proof of  Proposition \ref{prop::biasterm}]
According to the generation process, the histogram transforms $\{H_t\}_{t=1}^T$ are i.i.d. Therefore, for any $x \in \mathcal{X}$, the expected approximation error term can be decomposed as 
\begin{align}
\mathbb{E}_{\mathrm{P}_H}  \bigl( f_{\mathrm{P},\mathrm{E}}(x)- f_{L, \mathrm{P}}^*(x) \bigr)^2 
& = \mathbb{E}_{\mathrm{P}_H} \bigl( 
(f_{\mathrm{P},\mathrm{E}}(x) - \mathbb{E}_{\mathrm{P}_H}(f_{\mathrm{P},\mathrm{E}}(x)) )
+ (\mathbb{E}_{\mathrm{P}_H}(f_{\mathrm{P},\mathrm{E}}(x)) - f_{L,\mathrm{P}}^*(x)) \bigr)^2
\nonumber\\
& =  \mathrm{Var}(f_{\mathrm{P},\mathrm{E}}(x))
+ (\mathbb{E}_{\mathrm{P}_H}(f_{\mathrm{P},\mathrm{E}}(x))-f_{L,\mathrm{P}}^*(x))^2
\nonumber\\
& = T^{-1} \cdot \mathrm{Var}_{\mathrm{P}_H}(f_{\mathrm{P}, H_1}(x))
+ \bigl( \mathbb{E}_{\mathrm{P}_H} ( f_{\mathrm{P},H_1}(x) ) - f_{L, \mathrm{P}}^*(x) \bigr)^2.
\label{equ::biasvarianceDecom}
\end{align}
In the following, for the simplicity of notations, we drop the subscript of $H_1$ and write $H$ instead of $H_1$ when there is no confusion.

For the first term in \eqref{equ::biasvarianceDecom}, the assumption $f_{L,\mathrm{P}}^* \in C^{\alpha}$ implies
\begin{align}\label{equ::first}
\mathrm{Var}_{\mathrm{P}_H} \bigl( f_{\mathrm{P},H}(x) \bigr)
& = \mathbb{E}_{\mathrm{P}_H} \bigl( f_{\mathrm{P},H}(x) - \mathbb{E}_{\mathrm{P}_H}(f_{\mathrm{P},H}(x)) \bigr)^2
\leq \mathbb{E}_{\mathrm{P}_H} \bigl( f_{\mathrm{P},H}(x) - f_{L, \mathrm{P}}^*(x) \bigr)^2
\nonumber\\
& = \mathbb{E}_{\mathrm{P}_H} \biggl( \frac{1}{\mu(A_H(x))} \int_{A_H(x)} f_{L, \mathrm{P}}^*(x') \, dx' 
- f_{L, \mathrm{P}}^*(x) \biggr)^2
\nonumber\\
& = \mathbb{E}_{\mathrm{P}_H} \biggl( \frac{1}{\mu(A_H(x))} \int_{A_H(x)} \bigl( f_{L, \mathrm{P}}^*(x') - f_{L, \mathrm{P}}^*(x) \bigr) \, dx' \biggr)^2
\nonumber\\
& \leq \mathbb{E}_{\mathrm{P}_H} \bigl( c_L |A_H(x)| \bigr)^2
\leq c_L^2 d h_0^2.
\end{align}

We now consider the second term in \eqref{equ::biasvarianceDecom}. 
For $0<\alpha<1$, the second term of \eqref{equ::biasvarianceDecom} is bounded as follows,
\begin{align}\label{eq::PH1LP*}
\bigl( \mathbb{E}_{\mathrm{P}_H} ( f_{\mathrm{P},H_1}(x) ) - f_{L, \mathrm{P}}^*(x) \bigr)^2& \leq \biggl( \mathbb{E}_{\mathrm{P}_H} \biggl( \frac{1}{\mu(A_H(x))} \int_{A_H(x)} f_{L, \mathrm{P}}^*(x') \, dx' \biggr)
- f_{L, \mathrm{P}}^*(x) \biggr)^2
\nonumber\\
& = \mathbb{E}_{\mathrm{P}_H} \biggl( \frac{1}{\mu(A_H(x))} \int_{A_H(x)} (f_{L, \mathrm{P}}^*(x')- f_{L, \mathrm{P}}^*(x) ) \, dx' 
\biggr)^2\nonumber\\
&\leq \mathbb{E}_{\mathrm{P}_H}\big(c_L |A_H(x)|)^{2\alpha} \leq (c_L\sqrt{d}h_0)^{2\alpha} \leq c_L^2 d h_0^{2\alpha}.
\end{align}
Therefore, we have $\bigl( \mathbb{E}_{\mathrm{P}_H} ( f_{\mathrm{P},H_1}(x) ) - f_{L, \mathrm{P}}^*(x) \bigr)^2\leq c_L^2 d h_0^{2\alpha} + T^{-1} \cdot d c_L^2 h_0^2$, which completes the proof.
\end{proof}

Let $f_{\mathrm{D},\mathfrak{h}_l,\mathfrak{T}_l}^l$ be the empirical minimizer as in \eqref{equ::fdtlhl}, $f_{\mathrm{P},t}$ be as in \eqref{eq::fPtERM}, and  $\mathfrak{F}_{\mathfrak{h}_l,\mathfrak{T}_l}^l$ be the function set as in \eqref{eq::Glh}. We define the population version by 
\begin{align} \label{eq::flPhlTl}
\mathfrak{f}^l_{\mathrm{P},\mathfrak{h}_l,\mathfrak{T}_l} 
:= \mathfrak{f}_{\mathrm{P}}^l + \rho \cdot \mathfrak{f}^{l-1}_{\mathrm{D},\mathrm{B} | \mathfrak{X}_l} 
:= \frac{1}{\mathfrak{T}_l}\sum_{t=T_{l-1}+1}^{T_l} f_{\mathrm{P},t | \mathfrak{X}_l}
+ \rho \cdot \mathfrak{f}^{l-1}_{\mathrm{D},\mathrm{B} | \mathfrak{X}_l}.
\end{align}
Then we have $\mathfrak{f}^l_{\mathrm{P},\mathfrak{h}_l,\mathfrak{T}_l}  \in \mathfrak{F}_{\mathfrak{h}_l,\mathfrak{T}_l}^l$. The next proposition presents the local approximation error on the cell $A_{l,j}\subset B_l $ in \eqref{eq::Glh}.

\begin{proposition}\label{prop::secondApproxBoost}
Let $\mathfrak{X}_l$ be the residual region \eqref{eq::Xl} at the $l$-th stage of ABHT and $\{ A_{l,j}, j \in \mathfrak{J}_l \setminus \mathfrak{J}_{l,*} \}$ be the cells of $\mathfrak{X}_l$. For a fixed $j \in \mathfrak{J}_l\setminus\mathfrak{J}_{l,*}$, assume that there exists an $s \geq l$ such that $A_{l,j} \subset \Delta B_s$. Let $h_{l,j}$ and $\mathfrak{T}_{l,j}$ be the bin width and the iteration number of the cell $A_{l,j}$, respectively. For $i \in [l-1]$, let $h_{i,*}$ and $\mathfrak{T}_{i,*}$ be the optimal bin width and iteration number at the $i$-th stage as in \eqref{eq::hlstar}, respectively. Let $c := 24 \vee 3456M^2 \vee C_9$ where $C_9$ is the constant as in Proposition  \ref{prop::secondOracalBoost}. Then for any $\rho \in (0, (2c)^{-1/2})$, there exists a constant $C_7$ independent of $n$ such that
\begin{align}
& \mathbb{E}_{\mathrm{P}_H} \Bigl( \mathcal{R}_{L_{A_{l,j}},\mathrm{P}} \bigl( \mathfrak{f}_{\mathrm{P},\mathfrak{h}_l, \mathfrak{T}_l | A_{l,j}}^l \bigr) -  \mathcal{R}_{L_{A_{l,j}},\mathrm{P}}^* \Bigr)
\label{EPHApproxlAlj}
\\
& \leq C_7 \biggl( \sum_{i=1}^{l-1} \rho^{2(l-i)} \mathfrak{h}_{l-1,*}^d \bigl( \mathfrak{T}_{i,*}^{-1} \mathfrak{h}_{i,*}^2 + \mathfrak{h}_{i,*}^{2\alpha_s} \bigr) + \mathfrak{h}_{l-1,*}^d \bigl( \mathfrak{T}_{l,j}^{-1} h_{l,j}^2 +   \mathfrak{h}_{l,j}^{2\alpha_s} \bigr)
\nonumber\\
& \qquad \qquad
+ \mathfrak{h}_{l-1,*}^{\frac{d}{1+\delta}} 
\bigvee_{i=1}^{l-1}   \rho^{\frac{2\delta(l-i)}{1+\delta}} \mathfrak{h}_{i,*}^{-\frac{d}{1+\delta}} \mathfrak{T}_{i,*}^{\frac{1}{1+\delta}} n^{-\frac{1}{1+\delta}} 
+ \frac{\tau +\log(m_l/\mathfrak{h}_{l-1,*}^d)}{n} \biggr)
\nonumber
\end{align}
holds with probability $\mathrm{P}^n$ at least  $1-3l e^{-\tau}$.
\end{proposition}

\begin{proof}[Proof of  Proposition \ref{prop::secondApproxBoost}]
For any $x \in \mathcal{X}$, there holds
\begin{align}
& \eqref{EPHApproxlAlj}
= \mathbb{E}_{\mathrm{P}_X} \mathbb{E}_{\mathrm{P}_H} \bigl( \big(\mathfrak{f}_{\mathrm{P},\mathfrak{h}_l, \mathfrak{T}_l|A_{l,j}}^l(x) - f_{L,\mathrm{P}|A_{l,j}}^*(x)\big) \bigr)^2
\nonumber\\
& = \mathbb{E}_{\mathrm{P}_X} \mathbb{E}_{\mathrm{P}_H} \bigl( \rho \cdot \mathfrak{f}_{\mathrm{D},\mathrm{B} | A_{l,j}}^{l-1}(x) + \mathbb{E}_{\mathrm{P}_X} \bigl( \mathfrak{f}_{\mathrm{P}|A_{l,j}}^l(x) - f_{L,\mathrm{P}|A_{l,j}}^*(x)\big) \bigr)^2
\nonumber\\
& = \mathbb{E}_{\mathrm{P}_X} \mathbb{E}_{\mathrm{P}_H} \bigl( \rho \cdot \bigl( \mathfrak{f}_{\mathrm{D},\mathrm{B}|A_{l,j}}^{l-1}(x) -  f_{L,\mathrm{P}|A_{l,j}}^*(x) \bigr) + \mathbb{E}_{\mathrm{P}_X} \bigl( \mathfrak{f}_{\mathrm{P}|A_{l,j}}^l(x) - (1 - \rho) f_{L,\mathrm{P}|A_{l,j}}^*(x) \bigr) \bigr)^2
\nonumber\\
& \leq 2\rho^2 \mathbb{E}_{\mathrm{P}_X} \mathbb{E}_{\mathrm{P}_H} \bigl( \mathfrak{f}_{\mathrm{D},\mathrm{B}|A_{l,j}}^{l-1}(x) -  f_{L,\mathrm{P}|A_{l,j}}^*(x) \bigr)^2 
+ 2 \mathbb{E}_{\mathrm{P}_X} \mathbb{E}_{\mathrm{P}_H} \bigl( \mathfrak{f}_{\mathrm{P}|A_{l,j}}^l(x) - (1 - \rho) f_{L,\mathrm{P}|A_{l,j}}^*(x) \bigr)^2.
\label{eq::EPHAApprox}
\end{align}
For the first term in \eqref{eq::EPHAApprox}, there holds
\begin{align} \label{eq::EPXAdded}
\mathbb{E}_{\mathrm{P}_X}
\mathbb{E}_{\mathrm{P}_H}  \bigl( \mathfrak{f}_{\mathrm{D},\mathrm{B}|A_{l,j}}^{l-1}(x) -  f_{L,\mathrm{P}|A_{l,j}}^*(x) \bigr)^2 
& = \mathbb{E}_{\mathrm{P}_H}  \mathbb{E}_{\mathrm{P}_X}
\bigl( \mathfrak{f}_{\mathrm{D},\mathrm{B}|A_{l,j}}^{l-1}(x) -  f_{L,\mathrm{P}|A_{l,j}}^*(x) \bigr)^2 
\nonumber\\
& = \mathbb{E}_{\mathrm{P}_H} \bigl( \mathcal{R}_{L_{A_{l,j}},\mathrm{P}} \bigl( \mathfrak{f}_{\mathrm{D},\mathrm{B}|A_{l,j}}^{l-1} \bigr) - \mathcal{R}_{L_{A_{l,j}},\mathrm{P}}^* \bigr).
\end{align}
Using Lemma \ref{lem::coverT2} and with $|A_{l,j}| = h_{l-1,*}$, we get
\begin{align*}
\log \mathcal{N} \bigl( \mathfrak{F}_{\mathfrak{h}_{l-1}|A_{l,j}}^{l-1}, \|\cdot\|_{L_2(\mathrm{Q})}, \varepsilon \bigr)
\leq  C_9 h_{l-1,*}^d (l-1)^2\biggl( \sum_{i=1}^{l-1} \rho^{(l-1-i)\delta}  \mathfrak{T}_{i,*} \mathfrak{h}_{i,*}^{-d} \biggr) \varepsilon^{-2\delta}.
\end{align*}
Then similar arguments as in the proof of Proposition \ref{prop::secondOracalBoost} yield that
\begin{align}\label{eq::riskfDBl-1}
\mathcal{R}_{L_{A_{l,j}},\mathrm{P}} \bigl( \mathfrak{f}_{\mathrm{D},\mathrm{B}|A_{l,j}}^{l-1} \bigr) -  \mathcal{R}_{L_{A_{l,j}},\mathrm{P}}^*
& \leq 12 \bigl( \mathcal{R}_{L_{A_{l,j}},\mathrm{P}} \bigl( \mathfrak{f}_{\mathrm{P},\mathfrak{h}_l, \mathfrak{T}_l|A_{l,j}}^{l-1} \bigr) -  \mathcal{R}_{L_{A_{l,j}},\mathrm{P}}^* \bigr) + 3456 M^2 \tau / n 
\nonumber\\
& \phantom{=} 
+ C_9 \bigvee_{i=1}^{l-1}  \rho^{\frac{2\delta(l-1-i)}{1+\delta}} h_{l-1,*}^{\frac{d}{1+\delta}} \mathfrak{h}_i^{-\frac{d}{1+\delta}} \mathfrak{T}_i^{\frac{1}{1+\delta}} n^{-\frac{1}{1+\delta}}
\end{align}
holds with probability $\mathrm{P}^n$ at least $1-3e^{-\tau}$. Using \eqref{eq::EPHAApprox}, \eqref{eq::EPXAdded}, and \eqref{eq::riskfDBl-1}, we get 
\begin{align}\label{eq::inductivel}
& \eqref{EPHApproxlAlj}
\leq 2 \rho^2 \mathbb{E}_{\mathrm{P}_H} \biggl( 12 \bigl( \mathcal{R}_{L_{A_{l,j}},\mathrm{P}} \bigl( \mathfrak{f}_{\mathrm{P},\mathfrak{h}_l, \mathfrak{T}_l|A_{l,j}}^{l-1} \bigr) - \mathcal{R}_{L_{A_{l,j}},\mathrm{P}}^* \bigr) 
+ 3456 M^2 \tau / n 
\nonumber\\
& \phantom{=} 
+ C_9 \bigvee_{i=1}^{l-1} \rho^{\frac{2\delta(l-1-i)}{1+\delta}} h_{l-1,*}^{\frac{d}{1+\delta}}\mathfrak{h}_{i,*}^{-\frac{d}{1+\delta}} \mathfrak{T}_{i,*}^{\frac{1}{1+\delta}} n^{-\frac{1}{1+\delta}}  \biggr)
+ 2 \mathbb{E}_{\mathrm{P}_H} \mathbb{E}_{\mathrm{P}_X} \bigl( \mathfrak{f}_{\mathrm{P}|A_{l,j}}^l(x) - (1-\rho) f_{L,\mathrm{P}|A_{l,j}}^*(x) \bigr)^2 
\nonumber\\
& \leq c_1 \biggl( 
\rho^2  \mathbb{E}_{\mathrm{P}_H}  \bigl(\mathcal{R}_{L_{A_{l,j}},\mathrm{P}} (\mathfrak{f}_{\mathrm{P},\mathfrak{h}_l, \mathfrak{T}_l|A_{l,j}}^{l-1}) -  \mathcal{R}_{L_{A_{l,j}},\mathrm{P}}^*\bigr) 
+ \mathbb{E}_{\mathrm{P}_X} \mathbb{E}_{\mathrm{P}_H} \bigl( \mathfrak{f}_{\mathrm{P}|A_{l,j}}^l(x)- (1 - \rho) f_{L,\mathrm{P}|A_{l,j}}^*(x) \bigr)^2 
\nonumber\\
& \quad \qquad
+ \bigvee_{i=1}^{l-1}  \rho^{\frac{2\delta(l-i)}{1+\delta}} h_{l-1,*}^{\frac{d}{1+\delta}} \mathfrak{h}_{i,*}^{-\frac{d}{1+\delta}} \mathfrak{T}_{i,*}^{\frac{1}{1+\delta}} n^{-\frac{1}{1+\delta}} +  \frac{\tau}{n} \biggr),
\end{align}
where $c_1 := 24 \vee (3456M^2) \vee C_9$. Since the recursion formula \eqref{eq::inductivel} w.r.t.~$\mathfrak{f}_{\mathrm{P},\mathfrak{h}_l, \mathfrak{T}_l|A_{l,j}}^l$ and $\mathfrak{f}_{\mathrm{P}|A_{l,j}}^l$ also holds for $l-1, l-2, \ldots, 1$, with $\mathfrak{f}_{\mathrm{P},\mathfrak{h}_l, \mathfrak{T}_l|A_{l,j}}^1 = \mathfrak{f}_{\mathrm{P}|A_{l,j}}^1$ we then obtain
\begin{align}\label{eq::inductiveall}
\eqref{EPHApproxlAlj}
& \leq \sum_{i=1}^l c_1^{l-i} \rho^{2(l-i)} \mathbb{E}_{\mathrm{P}_H} \mathbb{E}_{\mathrm{P}_X} \bigl( \bigl( \mathfrak{f}_{\mathrm{P}|A_{l,j}}^i(x) - (1 - \rho) f_{L,\mathrm{P}|A_{l,j}}^*(x) \bigr)^2 \bigr) 
\nonumber\\
& \phantom{=} + 
\frac{c_1 (l-1)}{1-c_1} \cdot
\mathfrak{h}_{l-1,*}^{\frac{d}{1+\delta}} 
\bigvee_{i=1}^{l-1} \rho^{\frac{2\delta(l-i)}{1+\delta}} \mathfrak{h}_{i,*}^{-\frac{d}{1+\delta}} \mathfrak{T}_{i,*}^{\frac{1}{1+\delta}}  n^{-\frac{1}{1+\delta}} + \frac{c_1 \tau}{(1-c_1\rho^2) n}
\end{align}
with probability $\mathrm{P}^n$ at least $1-3l e^{-\tau}$. Using Proposition \ref{prop::biasterm} and Assumption \ref{def::localholderP}, we obtain 
\begin{align*}
\mathbb{E}_{\mathrm{P}_H} \bigl( \mathfrak{f}_{\mathrm{P}}^i(x) - (1 - \rho) f_{L,\mathrm{P}}^*(x) \bigr)^2 
\leq c_2 \biggl( \mathfrak{T}_{i,*}^{-1} \mathfrak{h}_{i,*}^2+ \sum_{k=1}^K \mathfrak{h}_{i,*}^{2\alpha_k} \eins_{\Delta B_k}(x) \biggr),
\qquad
i \in [l-1],
\end{align*}
and 
\begin{align*}
\mathbb{E}_{\mathrm{P}_H} \bigl( \mathfrak{f}_{\mathrm{P}}^l(x) - (1 - \rho) f_{L,\mathrm{P}}^*(x) \bigr)^2 
\leq c_2\biggl(\mathfrak{T}_{l,j}^{-1} \mathfrak{h}_{l,j}^2+ \sum_{k=1}^K \mathfrak{h}_{l,j}^{2\alpha_k} \eins_{\Delta B_k}(x) \biggr),
\end{align*}
where $c_2 := c_L^2 d$. These two inequalities together with \eqref{eq::inductiveall} and $A_{l,j} \subset \Delta B_k$ yield 
\begin{align*}
\eqref{EPHApproxlAlj}
& \leq c_2 \sum_{i=1}^{l-1} c_1^{l-i} \rho^{2(l-i)} \mathbb{E}_{\mathrm{P}_X} \biggl( \biggl( \mathfrak{T}_{i,*}^{-1} \mathfrak{h}_{i,*}^2 + \sum_{k=l}^K \mathfrak{h}_{i,*}^{2\alpha_k} \biggr) \eins_{A_{l,j}}(x) \biggr) 
+ \frac{c_1 \tau}{(1-c_1\rho^2) n}
\nonumber\\
& \phantom{=}
+ c_2 \mathbb{E}_{\mathrm{P}_X} \biggl( \mathfrak{T}_{l,j}^{-1} h_{l,j}^2 + \sum_{k=l}^K \mathfrak{h}_{l,j}^{2\alpha_k} \eins_{A_{l,j}}(x) \biggr)
+ \frac{c_1 l}{1-c_1} \cdot
\mathfrak{h}_{l-1,*}^{\frac{d}{1+\delta}} 
\bigvee_{i=1}^{l-1} 
\rho^{\frac{2\delta(l-i)}{1+\delta}} \mathfrak{h}_{i,*}^{-\frac{d}{1+\delta}}
\mathfrak{T}_{i,*}^{\frac{1}{1+\delta}} n^{-\frac{1}{1+\delta}} 
\end{align*}
holds with probability $\mathrm{P}^n$ at least $1-3l e^{-\tau}$. Thus, for all $j \in \mathfrak{J}_l \setminus \mathfrak{J}_{l,*}$ satisfying $A_{l,j} \subset \Delta B_s$ with $s \geq l$, by using the union bound, we obtain
\begin{align*}
\eqref{EPHApproxlAlj}
& \leq c_2 \sum_{i=1}^{l-1} c_1^{l-i} \rho^{2(l-i)} \mathfrak{h}_{l-1,*}^d \bigl( \mathfrak{T}_{i,*}^{-1} \mathfrak{h}_{i,*}^2  +  \mathfrak{h}_{i,*}^{2\alpha_s} \bigr)  
+ c_2 \mathfrak{h}_{l-1,*}^d \bigl( \mathfrak{T}_{l,j}^{-1} h_{l,j}^2 +   \mathfrak{h}_{l,j}^{2\alpha_s} \bigr)
\\
& \phantom{=} 
+ \frac{c_1 l}{1-c_1} \cdot \mathfrak{h}_{l-1,*}^{\frac{d}{1+\delta}} 
\bigvee_{i=1}^{l-1}  \rho^{\frac{2\delta(l-i)}{1+\delta}} \mathfrak{h}_{i,*}^{-\frac{d}{1+\delta}} \mathfrak{T}_{i,*}^{\frac{1}{1+\delta}} 
n^{-\frac{1}{1+\delta}} 
+ \frac{c_1 \tau}{(1-c_1\rho^2) n}
\end{align*}
with probability $\mathrm{P}^n$ at least $1-3l (m_l/\mathfrak{h}_{l-1,*}^d) e^{-\tau}$. Taking $\tau' := \tau - \log(m_l/\mathfrak{h}_{l-1,*}^d)$ and $\rho \leq (2c_1)^{-1/2}$, we get
\begin{align*}
\eqref{EPHApproxlAlj}
& \leq C_7 \biggl( \sum_{i=1}^{l-1} \rho^{2(l-i)} \mathfrak{h}_{l-1,*}^d \bigl( \mathfrak{T}_{i,*}^{-1} \mathfrak{h}_{i,*}^2  + \mathfrak{h}_{i,*}^{2\alpha_s} \bigr) + \mathfrak{h}_{l-1,*}^d \bigl( \mathfrak{T}_{l,j}^{-1} \mathfrak{h}_{l,j}^2 +   \mathfrak{h}_{l,j}^{2\alpha_s} \bigr)
\\
& \quad \qquad
+ \mathfrak{h}_{l-1,*}^{\frac{d}{1+\delta}} 
\bigvee_{i=1}^{l-1} \rho^{\frac{2\delta(l-i)}{1+\delta}} \mathfrak{h}_{i,*}^{-\frac{d}{1+\delta}} \mathfrak{T}_{i,*}^{\frac{1}{1+\delta}} n^{-\frac{1}{1+\delta}} 
+ \frac{\tau'+\log(m_l/\mathfrak{h}_{l-1,*}^d)}{n} \biggr),
\end{align*}
with probability $\mathrm{P}^n$ at least $1-3l e^{-\tau'}$, where $C_7 := c_2 \vee (c_1 l/ (1-c_1)) \vee (2c_1)$. This completes the proof.
\end{proof}

\begin{proof}[Proof of  Proposition \ref{prop::hl*}]
Let $A_{i,j}$, $i \in [l]$, $j \in \mathfrak{J}_i\setminus \mathfrak{J}_{i,*}$, be a cell that there exists an $s \geq i$ with $A_{i,j} \subset \Delta B_s$. According to the definition of $\mathfrak{h}_{l,*}$, it suffices to show that for $\rho$ satisfying \eqref{eq::ConditionRho}, the optimal parameters of the cell $A_{i,j}$ are of the order
\begin{align}\label{eq::optimalorder}
\mathfrak{h}_{i,j,*}
= n^{-\frac{1}{(2+2\delta)\alpha_s+d}}, 
\qquad
\mathfrak{T}_{i,j,*} = n^{0}.
\end{align}
In the following, we prove \eqref{eq::optimalorder} by induction on $l$.

Let us first consider the case $l=1$. Then for all $j \in \mathfrak{J}_1$, applying Proposition \ref{prop::secondOracalBoost} with $A := A_{1,j} \subset B_s$ for some $s \geq 1$ and using the union bound, we obtain 
\begin{align}
& \mathbb{E}_{\mathrm{P}_H} \Bigl( 
\lambda_{1,1,j} \mathfrak{h}_{1,j}^{-2d} + \lambda_{2,1,j} \mathfrak{T}_{1,j}^p  + \mathcal{R}_{L_{A_{1,j}},\mathrm{P}}(\mathfrak{f}^1_{\mathrm{D},\mathfrak{h}_{1,j},\mathfrak{T}_{1,j}}) - \mathcal{R}_{L_{A_{1,j}},\mathrm{P}}^* \Bigr)
\label{EPHArropxErrorIA}
\\
& \leq \mathbb{E}_{\mathrm{P}_H} \Bigl( 
12 a_{A_{1,j}}(\lambda_1) + C_{10} \lambda_{1,1,j}^{-\frac{p}{p-2+2p\delta}} \lambda_{2,1,j}^{-\frac{2}{p-2+2p\delta}} n^{-\frac{2p}{p-2+2p\delta}} + 3456M^2 \big(\tau+\log(m_1/\mathfrak{h}_{0}^d)\big) / n \Bigr)
\label{eq::oracleh1star}
\end{align}
with probability $\mathrm{P}^n$ not less than $1-3e^{-\tau}$. According to the definition of $a_A(\lambda_l)$ in \eqref{equ::approximationerror}, we have
\begin{align} \label{equ::approximationerrorA1j} 
a_{A_{1,j}}(\lambda_{1,j})
\leq \lambda_{1,1,j} \mathfrak{h}_{1,j}^{-2d} + \lambda_{2,1,j} \mathfrak{T}_{1,j}^p  + \mathcal{R}_{L_{A_{1,j}},\mathrm{P}}(\mathfrak{f}_{\mathrm{P}|A_{1,j}}^1) - \mathcal{R}_{L_{A_{1,j}},\mathrm{P}}^*.
\end{align}
Moreover, according to the definition of $\mathfrak{f}_{\mathrm{P}}^1$ in \eqref{eq::flPhlTl}, we have $\mathfrak{f}_{\mathrm{P} | A_{1,j}}^1 = f_{\mathrm{P},\mathrm{E} | A_{1,j}}$. Therefore, Proposition \ref{prop::biasterm} implies 
\begin{align}\label{eq::approxh1star}
\mathbb{E}_{\mathrm{P}_H} \bigl( 
\mathcal{R}_{L_{A_{1,j}},\mathrm{P}}(\mathfrak{f}_{\mathrm{P}|A_{1,j}}^1) & - \mathcal{R}_{L_{A_{1,j}},\mathrm{P}}^*
\bigr)
= \mathbb{E}_{\mathrm{P}_X} \Bigl( \mathbb{E}_{\mathrm{P}_H} \bigl( f_{\mathrm{P},\mathrm{E}}(x) - f_{L,\mathrm{P}}^*(x) \bigr)^2 \eins_{A_{1,j}}(x) \Bigr)
\nonumber\\
& 
\leq dc_L^2 \bigl( \mathfrak{h}_{1,j}^{2\alpha_s} + \mathfrak{h}_{1,j}^2 \mathfrak{T}_{1,j}^{-1} \bigr) \cdot \mathrm{P}_X(A_{1,j})
\leq d c_L^2 \bigl( \mathfrak{h}_{1,j}^{2\alpha_s} +  \mathfrak{h}_{1,j}^2 \mathfrak{T}_{1,j}^{-1} \bigr) \mathfrak{h}_0^d.
\end{align}
Using \eqref{eq::oracleh1star}, \eqref{equ::approximationerrorA1j},  \eqref{eq::approxh1star}, and $\mathfrak{h}_0 \leq 1$, we obtain that \eqref{EPHArropxErrorIA} can be upper bounded by
\begin{align*}
c_1 \Big(\lambda_{1,1,j} \mathfrak{h}_{1,j}^{-2d} + \lambda_{2,1,j} \mathfrak{T}_{1,j}^p + \mathfrak{h}_{1,j}^{2\alpha_s} + \mathfrak{T}_{1,j}^{-1} \mathfrak{h}_{1,j}^2 + \frac{\log n}{n} + \lambda_{1,1,j}^{-\frac{p}{p-2+2p\delta}} \lambda_{2,1,j}^{-\frac{2}{p-2+2p\delta}} n^{-\frac{2p}{p-2+2p\delta}}\Big)
\end{align*}
with probability $\mathrm{P}^n$ at least $1 - 3/n$, where $c_1 = C_{10} \vee (3456 M^2) \vee (12 d c_L^2)$. Minimizing this w.r.t.~$\lambda_{1,1,j}$, $\mathfrak{h}_{1,j}$, $\lambda_{2,1,j}$, and $\mathfrak{T}_{1,j}$, we obtain the minimum $6c_1 n^{-2\alpha_s/((2+2\delta)\alpha_s+d)}$, which is attained at
\begin{align*}
\lambda_{1,1,j} 
= n^{-\frac{2(d+\alpha_s)}{(2+2\delta)\alpha_s+d}},
\,
\mathfrak{h}_{1,j,*}
= n^{-\frac{1}{(2+2\delta)\alpha_s+d}}, 
\,
\lambda_{2,1,j} = n^{-\frac{2\alpha_s}{(2+2\delta)\alpha_s+d}}, 
\,
\mathfrak{T}_{1,j,*} = n^{0}.
\end{align*}

For the induction step, let us assume that \eqref{eq::optimalorder} holds for all $i \in [l-1]$.
In other words, with probability $\mathrm{P}^n$ at least $1-3e^{-\tau}$, there holds
\begin{align}\label{eq::histar}
\mathfrak{h}_{i,*} 
:= \bigvee_{j\in \mathfrak{J}_i \setminus \mathfrak{J}_{i,*}} \mathfrak{h}_{i,j,*} 
= \bigvee_{s \geq i} n^{-\frac{1}{(2+2\delta)\alpha_s+d}} 
=  n^{-\frac{1}{(2+2\delta)\alpha_i+d}},
\qquad
\mathfrak{T}_{i,*}= n^{0}.
\end{align}
Let $A_{l,j}$, $j \in \mathfrak{J}_l \setminus \mathfrak{J}_{l,*}$, be a cell that there exists an $s \geq l$ with $A_{l,j} \subset \Delta B_s$. Similarly as above, by applying Proposition \ref{prop::secondApproxBoost} and \ref{prop::secondOracalBoost} with $|A| := |A_{l,j}| = \mathfrak{h}_{l-1,*}$, we obtain 
\begin{align}
& \mathbb{E}_{\mathrm{P}_H} \Bigl( 
\lambda_{1,l,j} \mathfrak{h}_{l,j}^{-2d} + \lambda_{2,l,j} \mathfrak{T}_{l,j}^p  + \mathcal{R}_{L_{A_{l,j}},\mathrm{P}}(\mathfrak{f}_{\mathrm{D}, \mathfrak{h}_{l,j}, \mathfrak{T}_{l,j}}^l) - \mathcal{R}_{L_{A_{l,j}},\mathrm{P}}^* \Bigr)
\label{EPHArropxErrorIS}
\\
&\leq 12 C_7 \biggl( \lambda_{1,l,j} \mathfrak{h}_{l,j}^{-2d} + \lambda_{2,l,j} \mathfrak{T}_{l,j}^p + \sum_{i=1}^{l-1} \rho^{2(l-i)} \mathfrak{h}_{l-1,*}^d \bigl( \mathfrak{T}_{i,*}^{-1} \mathfrak{h}_{i,*}^2 + \mathfrak{h}_{i,*}^{2\alpha_s} \bigr) 
+ \mathfrak{h}_{l-1,*}^d \bigl( \mathfrak{T}_{l,j}^{-1} h_{l,j}^2 
+ \mathfrak{h}_{l,j}^{2\alpha_s} \bigr)
\nonumber\\
& \phantom{=} \qquad \qquad
+ \mathfrak{h}_{l-1,*}^{\frac{d}{1+\delta}} 
\bigvee_{i=1}^{l-1} \rho^{\frac{2\delta(l-i)}{1+\delta}} \mathfrak{h}_{i,*}^{-\frac{d}{1+\delta}} \mathfrak{T}_{i,*}^{\frac{1}{1+\delta}} n^{-\frac{1}{1+\delta}} 
+ \frac{\bigl( \tau + \log (m_l / \mathfrak{h}_{l-1,*}^d) \bigr)}{n} \biggr) 
+ \frac{3456 M^2 \tau}{n} 
\nonumber\\
& \phantom{=}
+ C_{10} \biggl(
\biggl(
\bigvee_{i=1}^{l-1} \rho^{\frac{2\delta(l-i)}{1+\delta}} \mathfrak{h}_{l-1,*}^{\frac{d}{1+\delta}} \mathfrak{h}_{i,*}^{-\frac{d}{1+\delta}} \mathfrak{T}_{i,*}^{\frac{1}{1+\delta}} n^{-\frac{1}{1+\delta}} 
\biggr)
\vee \biggl(
\lambda_{1,l,j}^{-\frac{p}{p-2+2p\delta}} \lambda_{2,l,j}^{-\frac{2}{p-2+2p\delta}} n^{-\frac{2p}{p-2+2p\delta}} \mathfrak{h}_{l-1,*}^{\frac{2pd}{p-2+2p\delta}} \biggr) \biggr)
\label{eq::lfiltererror}
\end{align}
with probability $\mathrm{P}^n$ at least $1-3e^{-\tau}$. Plugging \eqref{eq::histar} and \eqref{eq::histar} into \eqref{eq::lfiltererror}, we obtain 
\begin{align*}
\eqref{EPHArropxErrorIS}
& \leq 12 C_7 \biggl( \lambda_{1,l,j} \mathfrak{h}_{l,j}^{-2d} + \lambda_{2,l,j} \mathfrak{T}_{l,j}^p 
+ 2 \mathfrak{h}_{l-1,*}^d \sum_{i=1}^{l-1} \rho^{2(l-i)} n^{-\frac{2\alpha_s}{(2+2\delta)\alpha_i+d}} \biggr)
\nonumber\\
& \phantom{=}
+ 12 C_7 \mathfrak{h}_{l-1,*}^d \bigl( \mathfrak{T}_{l,j}^{-1} \mathfrak{h}_{l,j}^2 +  \mathfrak{h}_{l,j}^{2\alpha_s}  \bigr)
+ 12C_7 \lambda_{1,l,j}^{-\frac{p}{p-2+2p\delta}} \lambda_{2,l,j}^{-\frac{2}{p-2+2p\delta}} n^{-\frac{2p}{p-2+2p\delta}} \mathfrak{h}_{l-1,*}^{\frac{2pd}{p-2+2p\delta}}
\nonumber\\
& \phantom{=}
+ (12 C_7 + C_{10}) \mathfrak{h}_{l-1,*}^{\frac{d}{1+\delta}} 
\bigvee_{i=1}^{l-1} \rho^{\frac{2\delta(l-i)}{1+\delta}} n^{-\frac{2\alpha_i}{(2+2\delta)\alpha_i+d}} 
+ 3456 M^2 \tau / n
\nonumber\\
&\leq 12 C_7 \Bigl( \lambda_{1,l,j} \mathfrak{h}_{l,j}^{-2d} + \lambda_{2,l} \mathfrak{T}_{l,j}^p + \mathfrak{h}_{l-1,*}^d \bigl( \mathfrak{T}_{l,j}^{-1} \mathfrak{h}_{l,j}^2  +  \mathfrak{h}_{l,j}^{2\alpha_s} \bigr) \Bigr)
\nonumber\\
& \phantom{=}
+ (24C_7+C_{10}) \mathfrak{h}_{l-1,*}^{\frac{d}{1+\delta}} \sum_{j=1}^{l-1} \rho^{\frac{2\delta(l-j)}{1+\delta}} n^{-\frac{2\alpha_s}{(2+2\delta)\alpha_j+d}} + 3456 M^2 \tau / n
\nonumber\\
&\phantom{=} + C_{10}\lambda_{1,l,j}^{-\frac{p}{p-2+2p\delta}} \lambda_{2,l,j}^{-\frac{2}{p-2+2p\delta}} n^{-\frac{2p}{p-2+2p\delta}} \mathfrak{h}_{l-1,*}^{\frac{2pd}{p-2+2p\delta}}
\end{align*}
with probability $\mathrm{P}^n$ at least $1-3l e^{-\tau}$. The assumption on the shrinkage parameter $\rho$ in \eqref{eq::ConditionRho} implies
$\rho \leq \bigwedge_{k=1}^{l-1} \bigwedge_{s=l}^K n^{-\frac{\alpha_s(1+\delta)(2+2\delta)(\alpha_k-\alpha_s) }{\delta((2+2\delta)\alpha_k+d)((2+2\delta)\alpha_s+d)}}$
and thus we obtain
\begin{align*}
\eqref{EPHArropxErrorIS}
& \leq 12 C_7 \Bigl( \lambda_{1,l,j} \mathfrak{h}_{l,j}^{-2d} + \lambda_{2,l} \mathfrak{T}_{l,j}^p + \mathfrak{h}_{l-1,*}^d (\mathfrak{T}_{l,j}^{-1} \mathfrak{h}_{l,j}^2  +  \mathfrak{h}_{l,j}^{2\alpha_s}) \Bigr) + 3456 M^2 \tau / n
\\
& \phantom{=} + (24C_7+C_{10}) \mathfrak{h}_{l-1,*}^{\frac{d}{1+\delta}}(l-1)n^{-\frac{2\alpha_s}{(2+2\delta)\alpha_s+d}} 
+ C_{10}\lambda_{1,l,j}^{-\frac{p}{p-2+2p\delta}} \lambda_{2,l,j}^{-\frac{2}{p-2+2p\delta}} n^{-\frac{2p}{p-2+2p\delta}}\mathfrak{h}_{l-1,*}^{\frac{2pd}{p-2+2p\delta}}
\end{align*}
with probability $\mathrm{P}^n$ at least $1-3l e^{-\tau}$. By taking $\tau := \log n$ and minimizing the right-hand side w.r.t.~$\lambda_{1,l,j}$, $\mathfrak{h}_{l,j}$, $\lambda_{2,l,j}$, and $\mathfrak{T}_{l,j}$, we obtain
\begin{align*}
\eqref{EPHArropxErrorIS}
& \leq \bigl( 24 (l + 1) C_7+ l C_{10} + 3456 M^2 \bigr) \mathfrak{h}_{l-1,*}^d n^{-\frac{2\alpha_s}{(2+2\delta)\alpha_s+d}}
\end{align*}
with probability $\mathrm{P}^n$  at least $1 - 3l /n$, where the minimum is attained at 
\begin{align*}
\lambda_{1,l,j} =  n^{-\frac{2(d+\alpha_s)}{(2+2\delta)\alpha_s+d}} \mathfrak{h}_{l-1,*}^d,
\,
\mathfrak{h}_{l,j,*} = n^{-\frac{1}{(2+2\delta)\alpha_s+d}}, 
\,
\lambda_{2,l,j} = n^{-\frac{2\alpha_s}{(2+2\delta)\alpha_s+d}} \mathfrak{h}_{l-1,*}^d,
\,
\mathfrak{T}_{l,j,*} = n^{0}.
\end{align*}
Thus, we finished the induction step and \eqref{eq::optimalorder} is proved. 

According to definition of $\mathfrak{h}_{l,*}$ and using \eqref{eq::optimalorder}, we obtain
\begin{align*}
\mathfrak{h}_{l,*} 
= \bigvee_{j\in \mathfrak{J}_l\setminus \mathfrak{J}_{l,*}} \mathfrak{h}_{l,j,*}
= \bigvee_{s=l}^K n^{-1/((2+2\delta)\alpha_s+d)} 
=  n^{-1/((2+2\delta)\alpha_l+d)}
\end{align*}
and the corresponding number of iteration $\mathfrak{T}_{l,*} = n^{0}$ with probability $\mathrm{P}^n$  at least $1 - 3l /n$. This proves \eqref{eq::hlxTlx} and thus finishes the proof of Proposition \ref{prop::hl*}. 
\end{proof}

\subsubsection{Proofs Related to Section \ref{sec::residualregion}}

The next lemma presents the upper bound of the covering number of function space $\mathfrak{F}_{\mathfrak{h}_l,\mathfrak{T}_l|A}^l$ when the diameter of the hypercube $A$ is smaller than the bin width of base HT regressor in the $l$-th stage.

\begin{lemma}\label{lem::coverT2smallregion}
For a fixed $l \in [K]$, let $B_l$ be defined as in Assumption \ref{def::localholderP}. Let $\mathfrak{F}_{\mathfrak{h}_l,\mathfrak{T}_l}^l$ be the function set defined as in \eqref{eq::Glh}. Furthermore, let $\mathfrak{h}_l$ and $\mathfrak{T}_l$ be the bin width and the number of iterations in the $l$-th stage of ABHT. Suppose that $A \subset B_l$ is a hypercube satisfying $|A| \leq \mathfrak{h}_l$. Moreover, for $i \in [l-1]$, let $h_{i,*}$ be the optimal bin width in the $i$-th stage as in \eqref{eq::hlstar} and $\mathfrak{T}_{i,*}$ be the corresponding number of iteration. Then for any $\delta \in (0,1)$, $\varepsilon \in (0,1)$, and any probability measure $\mathrm{Q}$, we have
\begin{align*}
\log \mathcal{N} \bigl( \mathfrak{F}_{\mathfrak{h}_l,\mathfrak{T}_l|A}^l, \|\cdot\|_{L_2(\mathrm{Q})}, \varepsilon \bigr)  
\leq C_8 l^2 \biggl(\sum_{i=1}^{l-1} \rho^{2\delta(l-i)} \mathfrak{T}_{i,*}  + \mathfrak{T}_l  \biggr)  \varepsilon^{-2\delta},
\end{align*}
where $C_8$ is a constant only depending on $d$ and $\delta$.
\end{lemma}

\begin{proof}[Proof of  Lemma \ref{lem::coverT2smallregion}]
According to the construction of the ABHT algorithm, we have $h_t \geq \mathfrak{h}_l$ for any $t \in [T_l]$. If $|A| \leq \mathfrak{h}_l$, then we have $|A| \leq h_t$. Similar arguments as in the proof of Lemma \ref{lem::VCFn} imply that if $|A| \leq \mathfrak{h}_j$, there holds
\begin{align*}
\mathrm{VC}(\mathcal{F}_{H_t}) \leq 2^{d+1}(d+1)(\lfloor|A|\sqrt{d}/h_t\rfloor+1)^d
\leq 2^{d+2} d (2\sqrt{d})^d
=: c_d.
\end{align*}
This together with Theorem 2.6.7 in \cite{van1996weak} yields that there exists a universal constant $c_1$ such that $\mathcal{N}(\mathcal{F}_{H_t}, \|\cdot\|_{L_2(\mathrm{Q})}, \varepsilon) \leq c_1 c_d (16e)^{c_d} \varepsilon^{2 c_d-2}$. Simple algebra shows that for any $\varepsilon \in (0, 1/ (e \vee c_1))$, we have
\begin{align*}
\log \mathcal{N}(\mathcal{F}_{H_t|A}, \|\cdot\|_{L_2(\mathrm{D})}, \varepsilon) 
& \leq \log \bigl( c_1 c_d
(16 e)^{c_d}
(1/\varepsilon)^{2 c_d - 2} \bigr)
\\
& = \log c_1
+ \log c_d
+ c_d \log (16 e) 
+ 2 c_d \log (1/\varepsilon)
\leq 16 c_d \log (1/\varepsilon).
\end{align*}
Consequently, for all $\delta \in (0,1)$, we have
\begin{align}\label{eq::CoverNumFHsmall}
\sup_{\varepsilon \in (0, 1/(e \vee K))} \varepsilon^{2\delta} \log \mathcal{N}(\mathcal{F}_{H_t|A}, \|\cdot\|_{L_2(\mathrm{D})}, \varepsilon) 
\leq 16 c_d
\sup_{\varepsilon \in (0, 1)} \varepsilon^{2\delta} \log (1/\varepsilon).
\end{align}
Maximizing the right-hand side of \eqref{eq::CoverNumFHsmall} w.r.t.~$\varepsilon$, we obtain
\begin{align}\label{eq::logNFtlsmall}
\log \mathcal{N}(\mathcal{F}_{H_t|A}, \|\cdot\|_{L_2(\mathrm{D})}, \varepsilon) 
\leq 16/(2e\delta) c_d \varepsilon^{-2\delta},
\end{align}
where the maximum is attained at $\varepsilon^* = e^{-1/(2\delta)}$.

Now, similar arguments as in the proof of Lemma \ref{lem::coverT2} yield that for any probability distribution $\mathrm{Q}$, there holds
\begin{align}
\log \mathcal{N}(\mathfrak{F}_{\mathfrak{h}_l|A}^l, \|\cdot\|_{L_2(\mathrm{Q})}, 2\varepsilon) 
&\leq \log \biggl( \prod_{t=T_{l-1}+1}^{T_l} \mathcal{N}(\mathcal{F}_{H_t|A}, \|\cdot\|_{L_2(\mathrm{Q})}, \varepsilon) \biggr)
\nonumber\\
& = \log \Bigl( \mathcal{N}(\mathcal{F}_{H_{T_l}|A}, \|\cdot\|_{L_2(\mathrm{Q})}, \varepsilon)^{\mathfrak{T}_l}\Bigr)
\leq \mathfrak{T}_l \cdot 16/(2e\delta) c_d \varepsilon^{-2\delta},
\label{eq::logNFhll2}
\end{align}
where the last inequality is due to \eqref{eq::logNFtlsmall}. Then \eqref{eq::lFhlTlASSFlhlA} together with \eqref{eq::logNFhll2} yields that for any probability distribution $\mathrm{Q}$, there holds
\begin{align*}
& \log \mathcal{N} \bigl( \mathfrak{F}_{\mathfrak{h}_l,\mathfrak{T}_l|A}^l, \|\cdot\|_{L_2(\mathrm{Q})}, \varepsilon \bigr)
\\
& \leq \log \biggl( \prod_{i=1}^{l-1} \mathcal{N} \bigl( \rho^{l-i} \mathfrak{F}_{\mathfrak{h}_{i,*}|A}^i, \|\cdot\|_{L_2(\mathrm{Q})}, \varepsilon/l \bigr)\cdot  \mathcal{N} \bigl(\mathfrak{F}_{\mathfrak{h}_l|A}^l, \|\cdot\|_{L_2(\mathrm{Q})}, \varepsilon/l \bigr)\biggr)
\\
& = \sum_{i=1}^{l-1} \log\mathcal{N} \bigl( \mathfrak{F}_{\mathfrak{h}_{i,*}|A}^i, \|\cdot\|_{L_2(\mathrm{Q})}, \rho^{i-l} \varepsilon / l \bigr) + \log\mathcal{N} \bigl( \mathfrak{F}_{\mathfrak{h}_l|A}^l, \|\cdot\|_{L_2(\mathrm{Q})}, \varepsilon / l \bigr)
\\
& \leq C_8 l^2 \biggl(\sum_{i=1}^{l-1} \rho^{2\delta(l-i)} \mathfrak{T}_{i,*}  + \mathfrak{T}_l  \biggr) \varepsilon^{-2\delta},
\end{align*}
where $C_8 := 3c_d\delta^{-1}$.  Therefore, we finished the proof.
\end{proof}

The next proposition establishes the oracle inequality on a set $A$ whose diameter is smaller than the bin width $\mathfrak{h}_l$ of base HT regressors in the $l$-th stage.

\begin{proposition}\label{prop::secondOracalBoostsecond}
Let $\mathfrak{f}_{\mathrm{D},\mathfrak{h}_l, \mathfrak{T}_l}^l$ be the BHT regressor defined in \eqref{equ::fdtlhl}, $a_A(\lambda_l)$ be the corresponding approximation error defined by \eqref{equ::approximationerror}, and suppose that $|A| \leq \mathfrak{h}_l$. Then for all $\tau>0$,  there exists a constant $C_9$ independent of $n$ such that
\begin{align*}
& \lambda_{1,l,j} \mathfrak{h}_{l,j}^{-2d} +  \lambda_{2,l,j} \mathfrak{T}_{l,j}^p  + \mathcal{R}_{L_A,\mathrm{P}}(\mathfrak{f}_{\mathrm{D},\mathfrak{h}_l, \mathfrak{T}_l}^l) - \mathcal{R}_{L_A,\mathrm{P}}^*
\\
& \leq 12 a_A(\lambda_l) + \frac{3456 M^2 \tau}{n} + 3C_9 \biggl(
\biggl(
\bigvee_{i=1}^{l-1}
\rho^{\frac{2\delta(l-i)}{1+\delta}} \mathfrak{T}_{i,*}^{\frac{1}{1+\delta}} n^{-\frac{1}{1+\delta}} 
\biggr)
\vee
\biggl(
\lambda_{2,l}^{-\frac{1}{(1+\delta)p-1}} n^{-\frac{p}{(1+\delta)p-1}} 
\biggr)
\biggr)
\end{align*}
holds with probability at least $1-3e^{-\tau}$.
\end{proposition}

\begin{proof}[Proof of  Proposition \ref{prop::secondOracalBoostsecond}]
Denote $r^* := \Omega_{\lambda_l}(f) + \mathcal{R}_{L_A,\mathrm{P}}(f) - R^*_{L_A,\mathrm{P}}$, and for $r > r^*$, write
\begin{align*}
\mathcal{F}_r^l & := \{ f \in \mathfrak{F}_{\mathfrak{h}_l,\mathfrak{T}_l|A}^l : \Omega(f) + \mathcal{R}_{L_A,\mathrm{P}}(f) - \mathcal{R}^*_{L_A,\mathrm{P}} \leq r \},
\\
\mathcal{H}_r^l & := \{ L_A \circ f - L_A \circ f^*_{L,\mathrm{P}} : f \in \mathcal{F}_r^l \}.
\end{align*}
Note that for $f \in \mathcal{F}_r^l$, we have $\lambda_{2,l} \mathfrak{T}_l^p \leq r$ and $\lambda_{1,l} \mathfrak{h}_l^{-2d} \leq r$, that is,
\begin{align}\label{eq::Thcond}
\mathfrak{T}_l\leq \big(r/\lambda_{2,l}\big)^{1/p}
\quad
\text{ and }
\quad
\mathfrak{h}_l^{-d} \leq(r/\lambda_{1,l})^{1/2}.
\end{align}
Consequently, we have $\mathcal{F}_r^l \subset \mathfrak{F}_{\mathfrak{h}_l, \mathfrak{T}_l | A}^l$ with $\mathfrak{T}_l$ and $\mathfrak{h}_l$ satisfying \eqref{eq::Thcond}. Exercise 6.8 in \cite{StCh08} implies
\begin{align}
\ln \mathcal{N}(T, d, \varepsilon) < (a/\varepsilon)^q, \quad \forall \, \varepsilon > 0 
\quad \Longrightarrow 
\quad e_i(T, d) \leq 3^{1/q} a i^{-1/q}, \quad \forall \, i \geq 1.
\label{CoverEntropy}
\end{align}
This together with Lemma \ref{lem::coverT2smallregion} yields
\begin{align}\label{eq::entropyT}
e_i(\mathfrak{F}_{\mathfrak{h}_l,\mathfrak{T}_l|A}^l, d) 
\leq \biggl( 3 C_8 l^2\sum_{j=1}^{l-1} \rho^{2\delta(l-j)} \mathfrak{T}_{j,*} + \mathfrak{T}_l \biggr)^{1/2\delta} i^{-1/2\delta}, 
\quad 
\forall \, i \geq 1,
\end{align}
where $\delta\in (0,1)$. Since $L$ is Lipschitz continuous with the Lipschitz constant $|L|_1 \leq 4M$, we find
\begin{align*}
& \mathbb{E}_{\mathrm{P}^n} e_i(\mathcal{H}_r^l, L_2(\mathrm{D}))
\leq 4 M \mathbb{E}_{\mathrm{P}_X^n} e_i(\mathcal{F}_r^l, L_2(\mathrm{D}))
\leq 4 M \mathbb{E}_{\mathrm{P}_X^n} e_i(\mathfrak{F}_{\mathfrak{h}_l,\mathfrak{T}_l|A}^l, L_2(\mathrm{D}))
\\
& \leq 4 M \biggl( 3 C_8  l^2\sum_{i=1}^l \rho^{2\delta(l-i)} \mathfrak{T}_{i,*}  \biggr)^{\frac{1}{2\delta}} i^{-\frac{1}{2\delta}}
\leq 4 M (3C_8 l^2)^{\frac{1}{2\delta}} 
\biggl( \sum_{i=1}^{l-1} \rho^{2\delta(l-i)} \mathfrak{T}_{i,*}  +    (r/\lambda_{2,l})^{\frac{1}{p}} \biggr)^{\frac{1}{2\delta}} i^{-\frac{1}{2\delta}},
\end{align*}
where the second last inequality is due to \eqref{eq::entropyT} and the last inequality is due to \eqref{eq::Thcond}. Taking expectation with respect to $\mathrm{P}^n$, we get
\begin{align*}
\mathbb{E}_{\mathrm{P}_X^n} e_i(\mathcal{H}_r^l,L_2(\mathrm{D})) 
\leq c_1 \biggl( \sum_{i=1}^{l-1} \rho^{2\delta(l-i)} \mathfrak{T}_{i,*}  +  (r/\lambda_{2,l})^{\frac{1}{p}} \biggr)^{\frac{1}{2\delta}} i^{-\frac{1}{2\delta}},
\end{align*}
where $c_1 := 4 M (3C_8 l^2)^{1/2\delta}$. For least squares loss, the superemum bound $L_A(x,y,t) \leq 4 M^2$ holds for all $(x,y) \in \mathcal{X} \times \mathcal{Y}$, $t \in [-M, M]$, and the variance bound $\mathbb{E}(L_A \circ g - L_A \circ f_{L,\mathrm{P}}^*)^2 \leq V (\mathbb{E}(L_A \circ g - L_A \circ f^*_{L,\mathrm{P}}))^{\vartheta}$ holds for $V = 16 M^2$ and $\vartheta = 1$. Therefore, for $h \in \mathcal{H}_r^l$, we have $\|h\|_{\infty} \leq 8 M^2$ and $\mathbb{E}_{\mathrm{P}} h^2 \leq 16 M^2 r$. Then Theorem 7.16 in \cite{StCh08} with $a := c_1 \bigl( \sum_{i=1}^{l-1} \rho^{l-i} \mathfrak{T}_{i,*}^{1/(2\delta)} + (r/\lambda_{2,l})^{1/(2p\delta)} \bigr)$ yields that there exists a constant $c_2 > 0$ such that
\begin{align*}
\mathbb{E}_{\mathrm{P}^n} \mathrm{Rad}_D(\mathcal{H}_r^l,n)
& \leq c_2 l  
\biggl( \biggl( \bigvee_{i=1}^{l-1} \Bigl( \rho^{\delta(l-i)} \mathfrak{T}_{i,*}^{\frac{1}{2}} n^{-\frac{1}{2}} r^{\frac{1-\delta}{2}} \Bigr) \vee \Bigl( \rho^{\frac{2(l-i)}{1+\delta}} \mathfrak{T}_{i,*}^{\frac{1}{1+\delta}} n^{-\frac{1}{1+\delta}} \Bigr) \biggr)
\\
&\qquad
\vee \Bigl( r^{\frac{1+p(1-\delta)}{2p}} \lambda_{2,l}^{-1/(2p)} n^{-1/2} \Bigr) \vee \Bigl( r^{\frac{1}{p(1+\delta)}}  \lambda_{2,l}^{-\frac{1}{p(1+\delta)}} n^{-1/(1+\delta)}  \Bigr) \biggr)
=: \varphi_n(r).
\end{align*}
Simple algebra shows that the condition $\varphi_n(4r) \leq 2\sqrt{2} \varphi_n(r)$ is satisfied. Since $2\sqrt{2} < 4$, similar arguments show that there still hold the statements of the Peeling Theorem 7.7 in \cite{StCh08}. Consequently, Theorem 7.20 in \cite{StCh08} can also be applied, if the assumptions on $\varphi_n$ and $r$ are modified to $\varphi_n(4r) \leq 2\sqrt{2} \varphi_n(r)$ and $r \geq (75 \varphi_n(r)) \vee (1152 M^2\tau/n) \vee r^*$, respectively. It is easy to verify that the condition is satisfied if
\begin{align*}
r \geq  \biggl( C_9 \biggl( \bigvee_{i=1}^{l-1} \rho^{\frac{2\delta(l-i)}{1+\delta}} \mathfrak{T}_{i,*}^{\frac{1}{1+\delta}} n^{-\frac{1}{1+\delta}} \biggr) \vee C_9 \biggl( \lambda_{2,l}^{-\frac{1}{(1+\delta)p-1}} n^{-\frac{p}{(1+\delta)p-1}} \biggr) \vee \frac{1152M^2\tau}{n} \biggr),
\end{align*}
where the constant $C_9 := (75c_2 l)^2$, which yields the assertion.
\end{proof}

Let $\{ A_{l+1,j}, j \in \mathfrak{J}_{l+1} \}$ be the partition of the residual region $\mathfrak{X}_l$ at the $l$-th stage  of ABHT. The next proposition presents the local approximation error on the cell $A_{l+1,j}$ when the diameter of $A_{l+1,j}$ is smaller than the bin width $\widetilde{\mathfrak{h}}_{l,j}$.

\begin{proposition}\label{prop::firstApproxBoost}
Let $l \in [K]$ be fixed and $j \in \mathfrak{J}_{l+1}$ such that there exists an $s \geq l$ satisfying $A_{l+1,j} \subset \Delta B_s$. Furthermore, let $\widetilde{\mathfrak{h}}_{l,j}$ and $\widetilde{\mathfrak{T}}_{l,j}$ be the bin width and the iteration number of the cell $A_{l,j}$, and suppose that $|A_{l+1,j}| \leq \widetilde{\mathfrak{h}}_{l,j}$. Moreover, for $i \in [l]$, let $\mathfrak{h}_{i,*}$ and $\mathfrak{T}_{i,*}$ be the optimal bin width and iteration number as in \eqref{eq::hlstar}. Finally, let $c := 24 \vee 3456M^2 \vee 3C_8$ where $C_8$ is the constant as in Proposition \ref{lem::coverT2smallregion}. Then for any $\rho \in (0, (2c)^{-1/2})$, there exists a constant $C_{10}$ independent of $n$ such that
\begin{align*}
&\mathbb{E}_{\mathrm{P}_H} \Bigl( \mathcal{R}_{L_{A_{l+1,j}},\mathrm{P}} \bigl(\mathfrak{f}_{\mathrm{P}, \widetilde{\mathfrak{h}}_{l,j}, \widetilde{\mathfrak{T}}_{l,j} | A_{l+1,j}}^l \bigr) -  \mathcal{R}_{L_{A_{l+1,j}},\mathrm{P}}^* \Bigr)
\leq C_{10} \bigg(
\sum_{i=1}^{l-1} \rho^{2(l-i)} \mathfrak{h}_{l,*}^d \bigl( \mathfrak{T}_{i,*}^{-1} \mathfrak{h}_{i,*}^2 + \mathfrak{h}_{i,*}^{2\alpha_s} \bigr) 
\\
& \qquad \qquad \qquad \qquad \qquad
+ \mathfrak{h}_{l,*}^d \bigl( \widetilde{T}_{l,j}^{-1} \widetilde{\mathfrak{h}}_{l,j}^2 + \widetilde{\mathfrak{h}}_{l,j}^{2\alpha_s} \bigr)
+  \bigvee_{i=1}^{l-1}
\rho^{\frac{2\delta(l-i)}{1+\delta}} \mathfrak{T}_{i,*}^{\frac{1}{1+\delta}} 
n^{-\frac{1}{1+\delta}} + \frac{\tau +\log(m_l/\mathfrak{h}_{l,*}^d)}{n} \bigg)
\end{align*}
holds with probability at least $1-3le^{-\tau}$.
\end{proposition}

\begin{proof}[Proof of  Proposition \ref{prop::firstApproxBoost}]
Using the results in Proposition \ref{prop::secondOracalBoostsecond}, Proposition \ref{prop::firstApproxBoost} can be similarly proved as Proposition \ref{prop::secondApproxBoost}. Hence, we omit the proof.
\end{proof}

\begin{proof}[Proof of  Proposition \ref{prop::residualregion}]
For fixed $l \in [L]$, let $\mathfrak{h}_{l,*}$ be the optimal bin width at the $l$-th stage. The partition $\{ A_{l+1,j} \}_{j\in \mathfrak{J}_{l+1}}$ of the residual region $\mathfrak{X}_l$ has the diameter $|A_{l+1,j}| = \mathfrak{h}_{l,*}$. In order to filter out the residual region $\mathfrak{X}_{l+1}$, we need to determine the optimal bin width $\widetilde{\mathfrak{h}}_{l,j,*}$ of the cell $A_{l+1,j}$ for all $j \in \mathfrak{J}_{l+1}$.

In the following, we prove by induction on $l$ that if $\rho$ satisfies \eqref{eq::ConditionRho}, then for all $j\in \mathfrak{J}_{l+1}$ with $A_{l+1,j} \subset \Delta B_l$, we can choose
\begin{align}\label{eq::optpararegion}
\widetilde{\lambda}_{1,l,j} := 0, \quad
\widetilde{\lambda}_{2,l,j} := n^{-1}, \quad
\widetilde{\mathfrak{h}}_{l,j,*} := n^{-\frac{1}{2\alpha_l+d}}, \quad
\widetilde{\mathfrak{T}}_{l,j,*} := n^{0}.
\end{align}

Let us first consider the case $l=1$. For the cells $A_{2,j}$ with $\widetilde{\mathfrak{h}}_{1,j} \geq \mathfrak{h}_{1,*} = |A_{2,j}|$, applying Proposition \ref{prop::secondOracalBoostsecond} with $A = A_{2,j}\subset \Delta B_1$ and Proposition \ref{prop::firstApproxBoost} with $l =1$, we get 
\begin{align*}
& \mathbb{E}_{\mathrm{P}_H} \Bigl( 
\widetilde{\lambda}_{1,1,j} \widetilde{\mathfrak{h}}_{1,j}^{-2d} 
+ \widetilde{\lambda}_{2,1,j} \widetilde{\mathfrak{T}}_{1,j}^p + \mathcal{R}_{L_{A_{2,j}},\mathrm{P}} \bigl( \mathfrak{f}_{\mathrm{D}, \widetilde{\mathfrak{h}}_{1,j}, \widetilde{\mathfrak{T}}_{1,j}}^1 \bigr) - \mathcal{R}_{L_{A_{2,j}},\mathrm{P}}^* \Bigr)
\\
& \leq 12 C_{10} \Bigl( \widetilde{\lambda}_{1,1,j} \widetilde{\mathfrak{h}}_{1,j}^{-2d}  +
\widetilde{\lambda}_{2,1,j} \widetilde{\mathfrak{T}}_{1,j}^p +
\mathfrak{h}_{1,*}^d \bigl( \widetilde{\mathfrak{h}}_{1,j}^2 \widetilde{\mathfrak{T}}_{1,j}^{-1}  +  \widetilde{\mathfrak{h}}_{1,j}^{2\alpha_1} \bigr) \Bigr) 
\\
& \phantom{=}
+ 3456 M^2 \tau / n
+ 3C_9 n^{-\frac{p}{(1+\delta)p-1}}  \lambda_{2,1,j}^{-\frac{1}{(1+\delta)p-1}}
\end{align*}
with probability at least $1-3/n$. The right-hand side is minimized when choosing
\begin{align*}
\widetilde{\lambda}_{1,1,j} := 0, \quad
\widetilde{\lambda}_{2,1,j} := n^{-1}, \quad
\widetilde{\mathfrak{h}}_{1,j,*} := n^{-\frac{1}{2\alpha_1+d}}, \quad
\widetilde{\mathfrak{T}}_{1,j,*} := n^{0}.
\end{align*}
Therefore, for $j \in \mathfrak{J}_2$ with $A_{2,j} \subset \Delta B_1$, if $\widetilde{\mathfrak{h}}_{1,j} \geq \mathfrak{h}_{1,*}$, then we have
\begin{align}\label{eq::largearebad1} 
& \mathbb{E}_{\mathrm{P}_H} \Bigl( 
\widetilde{\lambda}_{1,1,j} \mathfrak{h}_{1,j,*}^{-2d} 
+ \widetilde{\lambda}_{2,1,j} \widetilde{\mathfrak{T}}_{1,j,*}^p
+ \mathcal{R}_{L_{A_{2,j}},\mathrm{P}} \bigl( \mathfrak{f}_{\mathrm{D}, \widetilde{\mathfrak{h}}_{1,j,*}, \widetilde{\mathfrak{T}}_{1,j,*}}^1 \bigr) \Bigr)
\nonumber\\
& \leq \mathbb{E}_{\mathrm{P}_H} \Bigl( \widetilde{\lambda}_{1,1,j} \widetilde{\mathfrak{h}}_{1,j}^{-2d}  +
\widetilde{\lambda}_{2,1,j} \widetilde{\mathfrak{T}}_{1,j}^p
+ \mathcal{R}_{L_{A_{2,j}},\mathrm{P}} \bigl( \mathfrak{f}_{\mathrm{D}, \widetilde{\mathfrak{h}}_{1,j}, \widetilde{\mathfrak{T}}_{1,j}}^1 \bigr) \Bigr).
\end{align}
Similar arguments as in the proof of \eqref{eq::optimalorder} in Proposition \ref{prop::hl*} with $A_{i,j} = A_{2,j} \subset \Delta B_1$ imply that \eqref{eq::largearebad1}  also holds for any $\widetilde{\mathfrak{h}}_{1,j} \leq \mathfrak{h}_{1,*}$. Therefore, we have $\widetilde{\mathfrak{h}}_{1,j,*} = \mathfrak{h}_{1,*}$.

Then we need to consider the cell $A_{2,j} \subset \Delta B_s$ with $s\geq 2$. Again, similar arguments as in the proof of \eqref{eq::optimalorder} in Proposition \ref{prop::hl*} with $A_{i,j} = A_{2,j}$ imply that the optimal bin width of $A_{2,j}$ turns out to be $\widetilde{\mathfrak{h}}_{1,j,*} = n^{- 1 / ((2+2\delta)\alpha_s+d)} \leq \mathfrak{h}_{1,*}$. Consequently, if $A_{2,j} \subset \Delta B_1$, then we have $A_{2,j} \subset \Delta  \mathfrak{X}_1$. And if  $A_{2,j} \subset B_2$, then $A_{2,j} \subset \mathfrak{X}_2$.

For the induction step, let us assume that \eqref{eq::optpararegion} holds for all $i \in [l-1]$. Let us first consider the case when the bin width $\widetilde{\mathfrak{h}}_{l,j} \geq \mathfrak{h}_{l,*} = |A_{l+1,j}|$. Applying Proposition \ref{prop::secondOracalBoostsecond} with $A := A_{l+1,j}$ and Proposition \ref{prop::firstApproxBoost}, we get
\begin{align}
& \mathbb{E}_{\mathrm{P}_H} \Bigl( \widetilde{\lambda}_{1,l,j} \widetilde{\mathfrak{h}}_{1,j}^{-2d} +
\widetilde{\lambda}_{2,l,j} \widetilde{\mathfrak{T}}_{l,j}^p + \mathcal{R}_{L_{A_{l+1,j} },\mathrm{P}} \bigl( \mathfrak{f}_{\mathrm{D}, \widetilde{\mathfrak{h}}_{l,j}, \widetilde{\mathfrak{T}}_{l,j}}^l \bigr) - \mathcal{R}_{L_{A_{l+1,j} },\mathrm{P}}^* \Bigr)
\label{eq::tildeEPHApprox}
\\
& \leq 12 C_{10} \biggl( \widetilde{\lambda}_{1,l,j} \widetilde{\mathfrak{h}}_{1,j}^{-2d} + \widetilde{\lambda}_{2,l,j} \widetilde{\mathfrak{T}}_{l,j}^p +  \sum_{i=1}^{l-1} \rho^{2(l-i)} \mathfrak{h}_{l,*}^d \bigl( \mathfrak{T}_{i,*}^{-1} \mathfrak{h}_{i,*}^2 +  \mathfrak{h}_{i,*}^{2\alpha_s} \bigr) 
\nonumber\\
& \phantom{=} \qquad \qquad 
+ \mathfrak{h}_{l,*}^d \bigl( \widetilde{\mathfrak{T}}_{l,j}^{-1} \widetilde{\mathfrak{h}}_{l,j}^2  +  \widetilde{\mathfrak{h}}_{l,j}^{2\alpha_s} \bigr) + \bigvee_{i=1}^{l-1} \rho^{\frac{2\delta(l-i)}{1+\delta}} \mathfrak{T}_{i,*}^{\frac{1}{1+\delta}}  n^{-\frac{1}{1+\delta}} \biggr) 
+ \frac{3456 M^2 \tau}{n} 
\nonumber\\
& \phantom{=} 
+ 3C_9 \biggl( \biggl( \bigvee_{i=1}^{l-1} \rho^{\frac{2\delta(l-i)}{1+\delta}} \mathfrak{T}_{i,*}^{\frac{1}{1+\delta}} n^{-\frac{1}{1+\delta}} \biggr) \vee \biggl( \widetilde{\lambda}_{2,l,j}^{-\frac{1}{(1+\delta)p-1}} n^{-\frac{p}{(1+\delta)p-1}} \biggr) \biggr).
\nonumber
\end{align}
By choosing $\rho$ satisfying \eqref{eq::ConditionRho} and taking $\{ \mathfrak{h}_{i,*} \}_{i=1}^{l-1}$ and $\{ \mathfrak{T}_{i,*} \}_{i=1}^{l-1}$ in \eqref{eq::hlxTlx}, we obtain 
\begin{align*}
\eqref{eq::tildeEPHApprox}
& \leq  12C_{10} \biggl( \widetilde{\lambda}_{2,l,j} \widetilde{\mathfrak{T}}_{l,j}^p + n^{-\frac{d}{2\alpha_l+d}} 
\bigl( \widetilde{\mathfrak{T}}_{l,j}^{-1} \widetilde{\mathfrak{h}}_{l,j}^2  + \widetilde{\mathfrak{h}}_{l,j}^{2\alpha_s} \bigr)
+  n^{-\frac{d}{2\alpha_l+d}}\sum_{j=1}^{l-1}
2 n^{-\frac{2\alpha_s}{2\alpha_s+d}} \biggr)
\\
& \phantom{=} + (12C_{10}+3C_9) n^{-1}
+ 3456 M^2 \tau / n + 3C_9\widetilde{\lambda}_{2,l,j}^{-\frac{1}{(1+\delta)p-1}} n^{-\frac{p}{(1+\delta)p-1}}.
\end{align*}
Choosing parameters as in \eqref{eq::optpararegion}, we get
\begin{align*}
\eqref{eq::tildeEPHApprox}
\leq \big(12C_{10}(l+2) + 6C_9 + 3456 M^2\big) n^{-\frac{d}{2\alpha_l+d}} n^{-\frac{2\alpha_s}{2\alpha_s+d}}.
\end{align*}
Therefore, for $j \in \mathfrak{J}_l$ with $A_{l+1,j} \subset \Delta B_l$, if $\widetilde{\mathfrak{h}}_{l,j} \geq \mathfrak{h}_{l,*}$, then we have
\begin{align*}
& \mathbb{E}_{\mathrm{P}_H} \Bigl( 
\widetilde{\lambda}_{1,l,j} \mathfrak{h}_{l,j,*}^{-2d} 
+ \widetilde{\lambda}_{2,l,j} \widetilde{\mathfrak{T}}_{l,j,*}^p
+ \mathcal{R}_{L_{A_{l+1,j}},\mathrm{P}} \bigl( \mathfrak{f}_{\mathrm{D}, \widetilde{\mathfrak{h}}_{l,j,*}, \widetilde{\mathfrak{T}}_{l,j,*}}^l \bigr) \Bigr)
\nonumber\\
& \leq \mathbb{E}_{\mathrm{P}_H} \Bigl( \widetilde{\lambda}_{1,l,j} \widetilde{\mathfrak{h}}_{l,j}^{-2d}  +
\widetilde{\lambda}_{2,l,j} \widetilde{\mathfrak{T}}_{l,j}^p
+ \mathcal{R}_{L_{A_{l+1,j}},\mathrm{P}} \bigl( \mathfrak{f}_{\mathrm{D}, \widetilde{\mathfrak{h}}_{l,j}, \widetilde{\mathfrak{T}}_{l,j}}^l \bigr) \Bigr).
\end{align*}
Similar arguments as in the proof of \eqref{eq::optimalorder} in Proposition \ref{prop::hl*} with $A_{i,j} = A_{l+1,j} \subset \Delta B_l$ imply that \eqref{eq::largearebad1}  also holds for any $\widetilde{\mathfrak{h}}_{l,j} \leq \mathfrak{h}_{l,*}$. Therefore, we have $\widetilde{\mathfrak{h}}_{l,j,*} = \mathfrak{h}_{l,*}$.

Then we need to consider the cell $A_{l+1,j} \subset \Delta B_s$ with $s\geq l+1$. Again, similar arguments as in the proof of \eqref{eq::optimalorder} in Proposition \ref{prop::hl*} with $A_{i,j} = A_{l+1,j} \subset \Delta B_s$ imply that the optimal bin width of $A_{l+1,j}$ turns out to be $\widetilde{\mathfrak{h}}_{l,j,*} = n^{- 1 / ((2+2\delta)\alpha_s+d)} \leq \mathfrak{h}_{l,*}$.  Let $j_1, j_2 \in \mathfrak{J}_{l+1}$ such that $A_{l+1,j_1} \subset \Delta B_l$ and $A_{l+1,j_2} \subset B_{l+1}$. Then we have $\mathfrak{h}_{l,*} = \widetilde{\mathfrak{h}}_{l,j_1,*} >   \widetilde{\mathfrak{h}}_{l,j_2,*}$ and consequently $A_{l+1,j_1}\subset \Delta \mathfrak{X}_l$ and $A_{l+1,j_2} \subset \mathfrak{X}_{l+1}$. For any $x\in B_{l+1}-\mathfrak{h}_{l,*}\subset B_{l+1}$, since the diameter of $A_{l+1,j}$ is $\mathfrak{h}_{l,*}$, there exists a $j_2$ such that $x \in A_{l+1,j_2}$. Thus, we have $x \in \mathfrak{X}_{l+1}$ and consequently $B_{l+1}-\mathfrak{h}_{l,*}\subset\mathfrak{X}_{l+1}$. On the other hand, let $x \in \mathfrak{X}_{l+1}$ and suppose $x \notin B_{l+1}+\mathfrak{h}_{l,*}$. Then there exists a $j_1\in \mathfrak{J}_{l+1}$ such that $x \in A_{l+1,j_1}\subset \Delta B_l$ and thus $x \in \Delta \mathfrak{X}_l$, which leads to a contradiction. Therefore, we have $x \in B_{l+1}+\mathfrak{h}_{l,*}$ and thus $\mathfrak{X}_{l+1} \subset B_{l+1}+\mathfrak{h}_{l,*}$. This finishes the proof. 
\end{proof}

\subsubsection{Proofs Related to Section \ref{sec::errordecomp}}

\begin{proof}[Proof of  Proposition \ref{prop::OracalBoost}]
For $i \in [K]$, let $\mathfrak{h}_{i,*}$ and $\mathfrak{T}_{i,*}$ be the optimal bin width and number of iteration defined as in \eqref{eq::hlxTlx}. Similar arguments as in the proof of Proposition \ref{prop::secondOracalBoost} with $A := A_{l+1,j}$ and $\tau := \log n$ yield
\begin{align}
\mathcal{R}_{L_{A_{l+1,j}},\mathrm{P}}(\mathfrak{f}_{\mathrm{D},\mathrm{B}}^l) - \mathcal{R}_{L_{A_{l+1,j}},\mathrm{P}}^*
& \leq 12 \Bigl( \mathcal{R}_{L_{A_{l+1,j}},\mathrm{P}}(\mathfrak{f}_{\mathrm{P},\mathrm{B}}^l) - \mathcal{R}_{L_{A_{l+1,j}},\mathrm{P}}^* \Bigr) + 3456 M^2 \log n / n 
\nonumber\\
& \phantom{=}
+ 3 C_{10} \mu(A_{l+1,j})^{\frac{1}{1+\delta}} \bigvee_{i=1}^l \rho^{\frac{2\delta(l-i)}{1+\delta}}  \mathfrak{h}_{i,*}^{-\frac{d}{1+\delta}} \mathfrak{T}_{i,*}^{-\frac{1}{1+\delta}} n^{-\frac{1}{1+\delta}}.
\label{eq::LocalExcessRisk}
\end{align}
Using H\"older's inequality, we get
\begin{align}\label{eq::mu/1+deltasum}
\sum_{j\in \mathfrak{J}_{\ell}\setminus \mathfrak{J}_{\ell}^*} \mu(A_{l+1,j})^{\frac{1}{1+\delta}}
\leq \biggl( \sum_{j \in \mathfrak{J}_{\ell} \setminus \mathfrak{J}_{\ell}^*}  \mu(A_{l+1,j}) \biggr)^{\frac{1}{1+\delta}} 
\biggl( \sum_{j \in \mathfrak{J}_{\ell} \setminus \mathfrak{J}_{\ell}^*} 1 \biggr)^{\frac{\delta}{1+\delta}} 
= (\Delta m_l)^{\frac{1}{1+\delta}} \#(\mathfrak{J}_{\ell}\setminus \mathfrak{J}_{\ell}^*) ^{\frac{\delta}{1+\delta}}. 
\end{align}
By Proposition \ref{prop::residualregion}, we have $\mu(\Delta \mathfrak{X}_l) \leq \Delta m_l + 2d \mathfrak{h}_{l,*} \leq \Delta m_l(1+2\mathfrak{h}_{l,*})^d \leq 2^d \Delta m_l$. Since $\mu(\Delta \mathfrak{X}_l) = \sum_{j\in \mathfrak{J}_{\ell}\setminus \mathfrak{J}_{\ell}^*} \mu(A_{l+1,j}) = \#(\mathfrak{J}_{\ell}\setminus \mathfrak{J}_{\ell}^*) \mathfrak{h}_{l,*}^d$, we have $\#(\mathfrak{J}_{\ell}\setminus \mathfrak{J}_{\ell}^*) \leq 2^d \Delta m_l\mathfrak{h}_{l,*}^{-d}$. This together with \eqref{eq::mu/1+deltasum} yields
\begin{align} \label{eq::muAl+1ji+delta}
\sum_{j\in \mathfrak{J}_{\ell}\setminus \mathfrak{J}_{\ell}^*} \mu(A_{l+1,j})^{\frac{1}{1+\delta}}
\leq 2^{-\frac{\delta d}{1+\delta}}
\mathfrak{h}_{l,*}^{-\frac{\delta d}{1+\delta}}\Delta m_l
\leq 
\mathfrak{h}_{l,*}^{-\frac{\delta d}{1+\delta}} \Delta m_l.
\end{align}
By summing up the local excess risk \eqref{eq::LocalExcessRisk} of all cells $\{ A_{l+1,j} , j \in \mathfrak{J}_l \setminus \mathfrak{J}_{l,*} \}$ on $\Delta \mathfrak{X}_l$, then using \eqref{eq::muAl+1ji+delta} and taking the order of bin width $\{ \mathfrak{h}_{i,*} \}_{i=1}^l$ in \eqref{eq::hlxTlx}, we obtain the conclusion with $C_1 := 3 C_7$.
\end{proof}

\begin{proof}[Proof of  Proposition \ref{prop::ApproxBoost}]
For $i \in [K]$, let $\mathfrak{h}_{i,*}$ and $\mathfrak{T}_{i,*}$ be the optimal bin width and number of iteration defined as in \eqref{eq::hlxTlx}. Similar as in the proof of Proposition \ref{prop::secondApproxBoost}, we can show that for any $A_{l+1,j}\subset \Delta \mathfrak{X}_l \cap \Delta B_l$, there holds
\begin{align}\label{eq::oraclecellinXl}
& \mathbb{E}_{\mathrm{P}_H} \Bigl( \mathcal{R}_{L_{A_{l+1,j}},\mathrm{P}} \bigl( \mathfrak{f}_{\mathrm{P},\mathrm{B}|A_{l+1,j}}^l \bigr) -  \mathcal{R}_{L_{A_{l+1,j}},\mathrm{P}}^* \Bigr)
\leq  C_7 \biggl( \sum_{i=1}^l \rho^{2(l-i)} \mathfrak{h}_{l,*}^d \bigl( \mathfrak{h}_{i,*}^2 \mathfrak{T}_{i,*}^{-1} + \mathfrak{h}_{i,*}^{2\alpha_l} \bigr)
\nonumber\\
& \qquad \qquad  \qquad  \qquad  
+ \mu(A_{l+1,j})^{\frac{1}{1+\delta}} \sum_{i=1}^{l-1} \rho^{\frac{2\delta(l-i)}{1+\delta}} \mathfrak{h}_{i,*}^{-\frac{d}{1+\delta}} \mathfrak{T}_{i,*}^{\frac{1}{1+\delta}} n^{-\frac{1}{1+\delta}} 
+ \frac{\tau + \log( \Delta m_l / \mathfrak{h}_{l,*}^d )}{n} \biggr)
\end{align}
with probability $\mathrm{P}^n$ at least $1-3l e^{-\tau}$. Obviously, \eqref{eq::oraclecellinXl} also holds for the cells $A_{l+1,j} \subset \Delta \mathfrak{X}_l$ satisfying $A_{l+1,j} \cap B_{l-1} \neq \emptyset$. Therefore, \eqref{eq::oraclecellinXl} holds for all $A_{l+1,j} \subset \Delta \mathfrak{X}_l$. By summing up the local approximation error \eqref{eq::oraclecellinXl} of all cells $\{ A_{l+1,j} , j \in \mathfrak{J}_l \setminus \mathfrak{J}_{l,*} \}$ on $\Delta \mathfrak{X}_l$, then using \eqref{eq::muAl+1ji+delta} and taking the order of bin width $\{ \mathfrak{h}_{i,*} \}_{i=1}^l$ in \eqref{eq::hlxTlx}, we obtain 
\begin{align*}
\mathbb{E}_{\mathrm{P}_H} \Bigl( \mathcal{R}_{L_{\Delta \mathfrak{X}_l},\mathrm{P}} \bigl( \mathfrak{f}_{\mathrm{P},\mathrm{B}}^l \bigr) -  \mathcal{R}_{L_{\Delta \mathfrak{X}_l},\mathrm{P}}^* \Bigr)
& \leq  C_7 \mathfrak{h}_{l,*}^{-\frac{\delta d}{1+\delta}}  \Delta m_l \biggl( \sum_{i=1}^l \rho^{2(l-i)} \bigl( \mathfrak{h}_{i,*}^2 \mathfrak{T}_{i,*}^{-1} + \mathfrak{h}_{i,*}^{2\alpha_l} \bigr)
\\
& \qquad 
+ \sum_{i=1}^{l-1} \rho^{\frac{2\delta(l-i)}{1+\delta}} \mathfrak{h}_{i,*}^{-\frac{d}{1+\delta}} \mathfrak{T}_{i,*}^{\frac{1}{1+\delta}} n^{-\frac{1}{1+\delta}} 
+ \frac{\tau + \log(\Delta m_l / \mathfrak{h}_{l,*}^d)}{n \mathfrak{h}_{l,*}^d}
\biggr)
\end{align*}
with probability $\mathrm{P}^n$ at least $1-3l e^{-\tau}$. Since $\Delta m_l \leq 1$ and $h_{l,*}^{-d} \leq n$, by taking $\tau = \log n$ and $C_2 = C_7$, we obtain the conclusion. 	
\end{proof}

\subsubsection{Proofs Related to Section \ref{sec::localadapt}}

\begin{proof}[Proof of  Proposition \ref{prop::adaptive}]
The result follows directly from Propositions \ref{prop::hl*} and \ref{prop::residualregion}, and the fact that $\mathfrak{X}_0 = B_0 = \mathcal{X}$.
\end{proof}

\subsubsection{Proofs Related to Section \ref{sec::upperabht}}

\begin{proof}[Proof of  Theorem \ref{thm::upperboost}]
By the definition of $f_{\mathrm{D},\mathrm{B}}$ in \eqref{eq::fDB}, we have $f_{\mathrm{D},\mathrm{B}} = \sum_{l=1}^K \mathfrak{f}_{\mathrm{D},\mathrm{B} | \Delta \mathfrak{X}_l}^l$. Since $L = \sum_{j=1}^K L_{\Delta \mathfrak{X}_l}$, we have 
\begin{align*}
& \mathbb{E}_{\mathrm{P}_H} \bigl(  \mathcal{R}_{L,\mathrm{P}}(f_{\mathrm{D},\mathrm{B}}) - \mathcal{R}_{L,\mathrm{P}}^* \bigr)
= \sum_{l=1}^{L} \mathbb{E}_{\mathrm{P}_H} \bigl(  \mathcal{R}_{L_{\Delta \mathfrak{X}_l},\mathrm{P}}(\mathfrak{f}_{\mathrm{D},\mathrm{B}}^l) - \mathcal{R}_{L_{\Delta \mathfrak{X}_l},\mathrm{P}}^*\big).
\end{align*}
Combining Propositions \ref{prop::OracalBoost} and \ref{prop::ApproxBoost}, we obtain
\begin{align*}
& \mathbb{E}_{\mathrm{P}_H} \Bigl( \mathcal{R}_{L_{\Delta \mathfrak{X}_l}, \mathrm{P}}(\mathfrak{f}_{\mathrm{D},\mathrm{B}}^l) - \mathcal{R}_{L_{\Delta \mathfrak{X}_l},\mathrm{P}}^* \Bigr)
\\
& \leq 12 C_7 C_2 \Delta m_l \mathfrak{h}_{l,*}^{-\frac{\delta d}{1+\delta}} \biggl( \sum_{j=1}^l \rho^{2(l-j)} \bigl( \mathfrak{T}_{j,*}^{-1} \mathfrak{h}_{j,*}^2 + \mathfrak{h}_{j,*}^{2\alpha_l} \bigr)
+ \sum_{j=1}^{l-1} \rho^{\frac{2(l-j)}{1+\delta}} \mathfrak{h}_{j,*}^{-\frac{d}{1+\delta}} \mathfrak{T}_{j,*}^{\frac{1}{1+\delta}} n^{-\frac{1}{1+\delta}}
+ \frac{2 \log n}{n \mathfrak{h}_{l,*}^d} \biggr) 
\\
& \phantom{=} 
+ \frac{3456 M^2 \log n}{n}  + C_1 \Delta m_l \mathfrak{h}_{l,*}^{-\frac{\delta d}{1+\delta}} \bigvee_{j=1}^l \rho^{\frac{2\delta(l-j)}{1+\delta}} \mathfrak{h}_{j,*}^{-\frac{d}{1+\delta}} \mathfrak{T}_{j,*}^{-\frac{1}{1+\delta}}  n^{-\frac{1}{1+\delta}}
\end{align*}
with probability $\mathrm{P}^n$ at least $1-3l/n$. According to Propositions \ref{prop::hl*} and \ref{prop::residualregion}, for any $l \in [K]$, we have the optimal order of $\mathfrak{h}_{l,*}$ and $\mathfrak{T}_{l,*}$ as in \eqref{eq::hlxTlx} and consequently 
\begin{align*}
\mathbb{E}_{\mathrm{P}_H} \bigl(  \mathcal{R}_{L_{\Delta \mathfrak{X}_l},\mathrm{P}}(\mathfrak{f}_{\mathrm{D},\mathrm{B}}^l) - \mathcal{R}_{L_{\Delta \mathfrak{X}_l},\mathrm{P}}^*\big)
\bigr)
\leq c_B \Delta m_l 
n^{-\frac{2\alpha_l - \delta d/(1+\delta)}{(2+2\delta)\alpha_l+d}}
\end{align*}
with probability $\mathrm{P}^n$ at least $1-3l/n$, where the constant  $c_B := 12C_7C_2(2l+3C_1) + 3456 M^2$. Summing up the above excess risk of the regions $\{ \Delta \mathfrak{X}_l, l \in [K] \}$, we obtain the assertion.
\end{proof}

\subsection{Proofs Related to PEHT}

\subsubsection{Proofs Related to Section \ref{sec::approxbag}}

\begin{proposition}\label{prop::samplesingleexp}
Let the histogram transform $H$ be defined as in \eqref{equ::HT} with bin width $h$. Then we have
\begin{align*}
\mathbb{E}_{\mathrm{P}^n}
\mathcal{R}_{L,\mathrm{P}}(f_{\mathrm{D},H})
- \mathcal{R}_{L,\mathrm{P}}(f^*_{\mathrm{P},H})
\leq  18 M^2 n^{-1} h^{-d}.
\end{align*}
\end{proposition}

\begin{proof}[Proof of  Proposition \ref{prop::samplesingleexp}]
For any fixed $j \in \mathcal{I}_H$, we define the random variable $Z_j := \sum_{i=1}^n \eins_{A_j}(X_i)$. Since the random variables $\{ \eins_{A_j}(X_i) \}_{i=1}^n$ are i.i.d.~Bernoulli distributed with parameter $\mathrm{P}(X\in A_j)$, elementary probability theory implies that the random variable $Z_j$ is Binomial distributed with parameters $n$ and $\mathrm{P}(X\in A_j)$. Therefore, for any $j \in \mathcal{I}_H$, we have $\mathbb{E}(Z_j) = n \cdot \mathrm{P}(X \in A_j)$. Moreover, the single NHT regressor $f_{\mathrm{D},H}$ can be defined by
\begin{align*}
f_{\mathrm{D},H}(x) =
\begin{cases}
\displaystyle
\frac{\sum_{i=1}^nY_i\eins_{A_j}(X_i)}{\sum_{i=1}^n\eins_{A_j}(X_i)}\eins_{A_j}(x) & \text{ if } Z_j > 0,
\\
0 & \text{ if } Z_j = 0.
\end{cases}
\end{align*}
By the law of total probability, we get 
\begin{align}
& \mathbb{E}_{\mathrm{P}_X} \bigl( f_{\mathrm{D},H}(X) - f_{\mathrm{P},H}^*(X) \bigr)^2
\nonumber\\
& = \sum_{j \in \mathcal{I}_H} \mathbb{E}_{\mathrm{P}_X}
\bigl( \bigl( f_{\mathrm{D},H}(X) - f_{\mathrm{P},H}^*(X) \bigr)^2 \big| X \in A_j \bigr)
\cdot \mathrm{P}(X\in A_j)
\nonumber\\
& = \sum_{j \in \mathcal{I}_H} \mathbb{E}_{\mathrm{P}_X}
\bigl( \bigl( f_{\mathrm{D},H}(X) - f_{\mathrm{P},H}^*(X) \bigr)^2 \big| X \in A_j, Z_j > 0 \bigr)
\cdot \mathrm{P}(Z_j > 0) \cdot \mathrm{P}(X \in A_j)
\nonumber\\
& \phantom{=}
+ \sum_{j \in \mathcal{I}_H} \mathbb{E}_{\mathrm{P}_X}
\bigl( \bigl( f_{\mathrm{D},H}(X) - f_{\mathrm{P},H}^*(X) \bigr)^2 \big| X \in A_j, Z_j = 0 \bigr)
\cdot \mathrm{P}(Z_j = 0) \cdot \mathrm{P}(X \in A_j)
\label{equ::term}.
\end{align}
For the first term in \eqref{equ::term}, there holds
\begin{align*}
& \sum_{j \in \mathcal{I}_H} \mathbb{E}_{\mathrm{P}_X}((f_{\mathrm{D},H}(X)-f_{\mathrm{P},H}^*(X))^2|X\in A_j, Z_j>0)\mathrm{P}(Z_j > 0)\mathrm{P}(X\in A_j)
\\
&=\sum_{j\in\mathcal{I}_H} \biggl(
\frac{\sum_{i=1}^nY_i\eins_{A_j}(X_i)}{\sum_{i=1}^n\eins_{A_j}(X_i)}
- \mathbb{E}(f_{L,\mathrm{P}}^*(X)|X\in A_j)
\biggr)^2\mathrm{P}(Z_j > 0)\mathrm{P}(X\in A_j)
\\
&=\sum_{j\in\mathcal{I}_H}\frac{\mathrm{P}(X\in A_j)}{(\sum_{i=1}^n\eins_{A_j}(X_i))^2}\biggl(\sum_{i=1}^n\eins_{A_j}(X_i)(Y_i-\mathbb{E}(f_{L,\mathrm{P}}^*(X)|X\in A_j))\biggr)^2\mathrm{P}(Z_j > 0)
\end{align*}
and the conditional expectation is 
\begin{align}
& \mathbb{E} \biggl( \sum_{j \in \mathcal{I}_H} \frac{\mathrm{P}(X \in A_j)}{(\sum_{i=1}^n \eins_{A_j}(X_i))^2}
\biggl( \sum_{i=1}^n \eins_{A_j}(X_i) \bigl( Y_i - \mathbb{E}(f_{L,\mathrm{P}}^*(X) | X \in A_j) \bigr) \biggr)^2
\bigg| X_i \in A_j \biggr)
\nonumber\\
& = \sum_{j \in \mathcal{I}_H}
\frac{\mathrm{P}(X \in A_j)}{(\sum_{i=1}^n \eins_{A_j}(X_i))^2} 
\sum_{i=1}^n \eins_{A_j}^2(X_i) \mathbb{E} \bigl( \bigl( Y - f_{\mathrm{P},H}^*(X) \bigr)^2 \big| X \in A_j \bigr)
\nonumber\\
& = \sum_{j \in \mathcal{I}_H}
\frac{\mathrm{P}(X \in A_j)}{\sum_{i=1}^n \eins_{A_j}(X_i)} \mathbb{E} \bigl( \bigl( Y - f_{\mathrm{P},H}^*(X))^2 \big| X \in A_j \bigr).
\label{ConditionalExpectationOnAj}
\end{align} 
Obviously, for any fixed $j \in \mathcal{I}_H$, there holds $\mathbb{E}(f_{\mathrm{P},H}^*(X) | X\in A_j) = \mathbb{E}(f_{L,\mathrm{P}}^*(X)|X\in A_j)$ and consequently we obtain
\begin{align*}
& \mathbb{E}((Y-f_{\mathrm{P},H}^*(X))^2|X\in A_j)
\\
&=\mathbb{E}((Y-f_{L,\mathrm{P}}^*(X))^2|X\in A_j)+\mathbb{E}((f_{L,\mathrm{P}}^*(X)-f_{\mathrm{P},H}^*(X))^2|X\in A_j)
\\
&=\sigma^2+\mathbb{E}((f_{L,\mathrm{P}}^*(X)-f_{\mathrm{P},H}^*(X))^2|X\in A_j).
\end{align*}
Taking expectation over both sides of \eqref{ConditionalExpectationOnAj} with respect to $\mathrm{P}^n$ and $\mathrm{P}_X$, we get
\begin{align}\label{eq::Zj>0}
& \mathbb{E}_{\mathrm{P}^n} \mathbb{E} \sum_{j \in \mathcal{I}_H} \mathbb{E}_{\mathrm{P}_X}((f_{\mathrm{D},H}(X)-f_{\mathrm{P},H}^*(X))^2|X\in A_j, Z_j>0)\mathrm{P}(Z_j > 0)\mathrm{P}(X\in A_j)
\nonumber\\
& = \bigl( \sigma^2 + \mathbb{E}(f_{L,\mathrm{P}}^*(X) - f_{\mathrm{P},H}^*(X))^2 \bigr)
\nonumber\\
&\phantom{=}\cdot
\sum_{j \in \mathcal{I}_H}
\biggl(
\mathrm{P}(X \in A_j) \mathbb{E}_{\mathrm{P}^n} \biggl(\biggl( \sum_{i=1}^n\eins_{A_j}(X_i) \biggr)^{-1}\bigg|Z_j>0\biggr) \biggr)\mathrm{P}(Z_j > 0)
\nonumber\\
& = \bigl( \sigma^2 + \mathbb{E}(f_{L,\mathrm{P}}^*(X) - f_{\mathrm{P},H}^*(X))^2 \bigr)
\nonumber\\
&\phantom{=}\cdot\sum_{j \in \mathcal{I}_H} \bigl( n^{-1} \cdot n \cdot
\mathrm{P}(X \in A_j) \mathbb{E}_{\mathrm{P}^n} (Z_j^{-1} | Z_j > 0) \bigr)\mathrm{P}(Z_j > 0)
\nonumber\\
& = n^{-1} \bigl( \sigma^2 + \mathbb{E}(f_{L,\mathrm{P}}^*(X) - f_{\mathrm{P},H}^*(X))^2 \bigr)\cdot
\sum_{j \in \mathcal{I}_H}  \bigl( \mathbb{E}(Z_j) \cdot \mathbb{E}(Z_j^{-1} | Z_j > 0) \bigr)\mathrm{P}(Z_j > 0).
\end{align}
Now we consider the term 
\begin{align*}
&\mathbb{E}(Z_j^{-1}|Z>0)\mathrm{P}(Z_j > 0)
= \sum_{l=1}^{n} \binom{n}{l} \big(\mathrm{P}(A_j)\big)^l \big(1-\mathrm{P}(A_j)\big)^{n-l} \frac{1}{l}
\\
&\leq 2 \sum_{l=1}^{n} \binom{n}{l} \big(\mathrm{P}(A_j)\big)^l \big(1-\mathrm{P}(A_j)\big)^{n-l} \frac{1}{l+1}
= \frac{2}{n+1} \sum_{l=1}^{n} \binom{n+1}{l+1} \big(\mathrm{P}(A_j)\big)^l \big(1-\mathrm{P}(A_j)\big)^{n-l}
\\
&= \frac{2}{n+1} \sum_{l=2}^{n+1} \binom{n+1}{l} \big(\mathrm{P}(A_j)\big)^{l-1} \big(1-\mathrm{P}(A_j)\big)^{n-l+1}
\\
&= \frac{2(1-\mathrm{P}(A_j))}{(n+1)\mathrm{P}(A_j)} \sum_{l=2}^{n+1} \binom{n+1}{l} \big(\mathrm{P}(A_j)\big)^l \big(1-\mathrm{P}(A_j)\big)^{n-l}
\\
&\leq \frac{2h^{-d}}{(n+1)} \sum_{l=0}^{n+1} \binom{n+1}{l} \big(\mathrm{P}(A_j)\big)^l \big(1-\mathrm{P}(A_j)\big)^{n-l}
\leq 2h^{-d}n^{-1}.
\end{align*}
Therefore, the first term in \eqref{equ::term} can be upper bounded by
\begin{align}\label{eq::boundterm1}
\mathbb{E}_{\mathrm{P}^n} \mathbb{E} \sum_{j \in \mathcal{I}_H} \mathbb{E}_{\mathrm{P}_X} & ((f_{\mathrm{D},H}(X)-f_{\mathrm{P},H}^*(X))^2|X\in A_j, Z_j>0)\mathrm{P}(Z_j > 0)
\nonumber\\
&\leq n^{-1} ( \sigma^2 + 4M^2)\cdot 
\sum_{j \in \mathcal{I}_H}  \bigl( \mathbb{E}(Z_j) \cdot 2h^{-d}n^{-1} \bigr)
\nonumber\\
&= n^{-1} ( \sigma^2 + 4M^2)\cdot 
\sum_{j \in \mathcal{I}_H}  \bigl( n h^d \cdot 2h^{-d}n^{-1} \bigr)
=  16 M^2 n^{-1} h^{-d}.
\end{align}

We now turn to estimate the second term in \eqref{equ::term}.  By the definition of $f_{\mathrm{D},H}$, we have
\begin{align}\label{eq::boundterm2}
\sum_{j \in \mathcal{I}_H} \mathbb{E}_{\mathrm{P}_X}
& \bigl( \bigl( f_{\mathrm{D},H}(X) - f_{\mathrm{P},H}^*(X))^2 \big| X \in A_j, Z_j = 0 \bigr)
\mathrm{P}(Z_j = 0) \mathrm{P}(X \in A_j)
\nonumber\\
&\leq \sum_{j \in \mathcal{I}_H} (2M)^2 (1-\mathrm{P}(A_j))^n \mathrm{P}(A_j)
\leq \sum_{j \in \mathcal{I}_H} (2M)^2 e^{-n\mathrm{P}(A_j)} \mathrm{P}(A_j)
\nonumber\\
&\leq (2M)^2 e^{-nh^d} \sum_{j \in \mathcal{I}_H} \mathrm{P}(A_j)
= (2M)^2 e^{-nh^d}.
\end{align}
Combining \eqref{eq::boundterm1} and \eqref{eq::boundterm2}, we obtain
\begin{align*}
& \mathbb{E}_{\mathrm{P}^n} \mathbb{E}_{\mathrm{P}_X} 
\bigl( f_{\mathrm{D},H}(X) - f_{\mathrm{P},H}^*(X))^2
\\
&=\sum_{j\in\mathcal{I}_H}\mathbb{E}_{\mathrm{P}_X}((f_{\mathrm{D},H}(X)-f_{\mathrm{P},H}^*(X))^2|X\in A_j, Z_j>0) \cdot \mathrm{P}(Z_j > 0) \cdot \mathrm{P}(X\in A_j)
\\
& \phantom{=}
+ \sum_{j \in \mathcal{I}_H} \mathbb{E}_{\mathrm{P}_X}((f_{\mathrm{D},H}(X)-f_{\mathrm{P},H}^*(X))^2|X\in A_j, Z_j=0) \cdot \mathrm{P}(Z_j = 0) \cdot \mathrm{P}(X\in A_j)
\\
&\leq (2M)^2 e^{- n h^d} + 16 M^2 n^{-1} h^{-d}.
\end{align*} 
Since $t \to te^{-t}$ is decreasing on $t\geq 1$, we have for any $t\geq 1$, there holds $te^{-t} \leq e^{-1}$. Obviously, we have $nh^d \geq 1$ and thus $e^{-nh^d} \leq e^{-1} n^{-1}h^{-d}$. Therefore, we obtain
\begin{align*}
\mathbb{E}_{\mathrm{P}^n} \mathbb{E}_{\mathrm{P}_X} 
\bigl( f_{\mathrm{D},H}(X) - f_{\mathrm{P},H}^*(X))^2
\leq (4e^{-1}M^2 + 16 M^2) n^{-1} h^{-d} \leq 18 M^2 n^{-1} h^{-d},
\end{align*}
which finishes the proof. 
\end{proof}

\begin{proof}[Proof of  Proposition \ref{prop::approxbag}]
First of all, let us consider the PEHT whose base learners have the same bin width $h$. According to the Proposition \ref{prop::samplesingleexp}, the sample error of single histogram transform regressor can be upper bounded by
\begin{align*}
\mathbb{E}_{\mathrm{P}^n}\mathbb{E}_{\mathrm{P}_X}|f_{\mathrm{P},t}(X) - f_{\mathrm{D},t}(X)|^2 \leq 18M^2 n^{-1} h^{-d}.
\end{align*}
Using the Cauchy-Schwarz inequality, we get
\begin{align*}
\mathbb{E}_{\mathrm{P}_H}\mathbb{E}_{\mathrm{P}^n}\mathbb{E}_{\mathrm{P}_X}|f_{\mathrm{P},\mathrm{E}}(X) - f_{\mathrm{D},\mathrm{E}}(X)|^2 
& = \mathbb{E}_{\mathrm{P}_H}\mathbb{E}_{\mathrm{P}^n} \mathbb{E}_{\mathrm{P}_X} \biggl| \frac{1}{T} \sum_{t=1}^T (f_{\mathrm{P},t}(X) - f_{\mathrm{D},t}(X)) \biggr|^2 
\\
&\leq 
\mathbb{E}_{\mathrm{P}_H}\mathbb{E}_{\mathrm{P}^n}\mathbb{E}_{\mathrm{P}_X}|f_{\mathrm{P},1}(X) - f_{\mathrm{D},1}(X)|^2 
\leq 18 M^2  n^{-1} h^{-d},
\end{align*}
which gives the upper bound for the sample error of PEHT. Moreover, Proposition \ref{prop::biasterm} implies that when fitting $f^*_{L,\mathrm{P}} \in C^{\alpha}(\mathcal{X})$ with $\alpha \in (0,1]$, the approximation error of PEHT using bin width $h$ is upper bounded by
\begin{align*}
\mathbb{E}_{\mathrm{P}_H}|f_{\mathrm{P},\mathrm{E}}(x) - f^*_{L,\mathrm{P}}(x)|^2 \leq c_L^2 h^{2\alpha} + d c_L^2 h^2/T
\leq c_L^2(d+1) h^{2\alpha},
\end{align*}
when taking $T \geq n^0$. Combining the above two estimates and choosing $h = n^{-1/(2\alpha+d)}$ and $T \geq n^0$, we obtain
$\mathbb{E}_{\mathrm{P}_H}|f_{\mathrm{D},\mathrm{E}}(x) - f^*_{L,\mathrm{P}}(x)|^2 \leq n^{- 2\alpha/(2\alpha+d)}$.
Classical nonparametric statistics tells us that this rate turns out to be minimax when fitting $f^*_{L,\mathrm{P}} \in C^{\alpha}(\mathcal{X})$. This implies that both the sample error bound and the approximation error bound are tight. In other words, there exist a target function $f^*_{L,\mathrm{P}} \in C^{\alpha}(\mathcal{X})$ such that
\begin{align}\label{eq::lowerapproxsame}
\mathbb{E}_{\mathrm{P}_H}|f_{\mathrm{P},\mathrm{E}}(x) - f^*_{L,\mathrm{P}}(x)|^2 \geq c_1 h^{2\alpha}
\end{align}
and 
$\mathbb{E}_{\mathrm{P}}|f_{\mathrm{P},\mathrm{E}}(X) - f_{\mathrm{D},\mathrm{E}}(X)|^2 \geq c_2 n^{-1} h^{-d}$,
where $c_1$ and $c_2$ are constants independent of $n$.

Next, let us consider the PEHT whose base learners have $L$ different bin widths $\mathfrak{h}_l$, $l \in [L]$. Among these $T$ base learners in PEHT, assume that there exist $T_l$ base learners with bin width $\mathfrak{h}_l$ for $l \in [L]$. Then we have $T := \sum_{l=1}^L T_l$ and define $\mathfrak{f}_{\mathrm{D},\mathrm{E}}^l := \frac{1}{T_l} \sum_{t=1}^{T_l} \mathfrak{f}_{\mathrm{D},t}^l$, where $\mathfrak{f}_{\mathrm{D},t}^l$ are the base learners with bin width $\mathfrak{h}_l$ for $t \in [T_l]$. Thus we can make the decomposition for PEHT as follows:
\begin{align}
f_{\mathrm{D},\mathrm{E}} 
& := \frac{1}{T} \sum_{t=1}^T f_{\mathrm{D},t} = \frac{1}{T} \sum_{l=1}^L \sum_{t=1}^{T_l} \mathfrak{f}_{\mathrm{D},t}^l = \frac{T_l}{T} \sum_{l=1}^L \mathfrak{f}_{\mathrm{D},\mathrm{E}}^l,
\label{eq::decomppartensemble}
\\
f_{\mathrm{P},\mathrm{E}} 
& := \frac{1}{T} \sum_{t=1}^T f_{\mathrm{P},t} = \frac{1}{T} \sum_{l=1}^L \sum_{t=1}^{T_l} \mathfrak{f}_{\mathrm{P},t}^l = \frac{T_l}{T} \sum_{k=1}^{L} \mathfrak{f}_{\mathrm{P},\mathrm{E}}^l.
\label{eq::decomppartensembleP}
\end{align}
Then we have
\begin{align}\label{equ::biasdecompensemble}
& \mathbb{E}_{\mathrm{P}_H} \bigl( f_{\mathrm{P},\mathrm{E}}(x)- f_{L, \mathrm{P}}^*(x) \bigr)^2 
\nonumber\\
& = \mathbb{E}_{\mathrm{P}_H} \bigl( 
(f_{\mathrm{P},\mathrm{E}}(x) - \mathbb{E}_{\mathrm{P}_H}(f_{\mathrm{P},\mathrm{E}}(x)) )
+ (\mathbb{E}_{\mathrm{P}_H}(f_{\mathrm{P},\mathrm{E}}(x)) - f_{L,\mathrm{P}}^*(x)) \bigr)^2
\nonumber\\
& =  \mathrm{Var}(f_{\mathrm{P},\mathrm{E}}(x))
+ (\mathbb{E}_{\mathrm{P}_H}(f_{\mathrm{P},\mathrm{E}}(x))-f_{L,\mathrm{P}}^*(x))^2
\nonumber\\
& = \sum_{l=1}^L \mathrm{Var} \bigl( (T_l/T) \mathfrak{f}_{\mathrm{P},1}^l(x) \bigr)
+ \biggl( \sum_{l=1}^L \bigl[ \mathbb{E}_{\mathrm{P}_H} \bigl( (T_l/T) \mathfrak{f}_{\mathrm{P},1}^l(x) \bigr) - (T_l/T) f_{L, \mathrm{P}}^*(x) \bigr] \biggr)^2
\nonumber\\
& = \sum_{l=1}^{L} \mathrm{Var} \bigl( (T_l/T) \mathfrak{f}_{\mathrm{P},1}^l(x) \bigr)
+  \sum_{l=1}^L \bigl[ \mathbb{E}_{\mathrm{P}_H} \bigl( (T_l/T) \mathfrak{f}_{\mathrm{P},1}^l(x) \bigr) - (T_l/T) f_{L, \mathrm{P}}^*(x) \bigr]^2
\nonumber\\
& \phantom{=} 
+ \sum_{l=1}^L \sum_{l \neq k}  
\bigl[ \mathbb{E}_{\mathrm{P}_H} \bigl( (T_k/T)  \mathfrak{f}_{\mathrm{P},1}^k(x) \bigr) - (T_k/T) f_{L, \mathrm{P}}^*(x) \bigr]
\bigl[ \mathbb{E}_{\mathrm{P}_H} \bigl( (T_l/T)  \mathfrak{f}_{\mathrm{P},1}^l(x) \bigr) - (T_l/T)  f_{L, \mathrm{P}}^*(x) \bigr]
\nonumber\\
& \geq \sum_{l=1}^L (T_l/T)^2 \mathbb{E}_{\mathrm{P}_H}  \bigl( \mathfrak{f}_{\mathrm{P},1}^l(x) - f_{L, \mathrm{P}}^*(x) \bigr)^2
\nonumber\\
& \phantom{=} 
+ \sum_{k=1}^L \sum_{l \neq k}  \bigl[ \mathbb{E}_{\mathrm{P}_H} \bigl( (T_k/T)   \mathfrak{f}_{\mathrm{P},1}^k(x) \bigr) - (T_k/T)  f_{L, \mathrm{P}}^*(x) \bigr] 
\bigl[ \mathbb{E}_{\mathrm{P}_H} \bigl( (T_l/T)  \mathfrak{f}_{\mathrm{P},1}^l(x) \bigr) - (T_l/T)  f_{L, \mathrm{P}}^*(x) \bigr].
\end{align}
For the first term in \eqref{equ::biasdecompensemble}, \eqref{eq::lowerapproxsame} implies that there exist a target function $f_{L, \mathrm{P}}^* \in C^{\alpha}(\mathcal{X})$ and $x\in\mathcal{X}$ such that  
\begin{align}\label{eq::firstbias}
\sum_{l=1}^L (T_l/T)^2 \mathbb{E}_{\mathrm{P}_H}  \bigl( \mathfrak{f}_{\mathrm{P},1}^l(x)- f_{L, \mathrm{P}}^*(x) \bigr)^2 
\geq c_1 \sum_{l=1}^L (T_l/T)^2 h_l^{2\alpha}.
\end{align}
Moreover, using \eqref{eq::PH1LP*}, we get $|\mathbb{E}_{\mathrm{P}_H}  \mathfrak{f}_{\mathrm{P},1}^k(x) - f_{L, \mathrm{P}}^*(x)| \leq c_L \sqrt{d} h_k^{\alpha}$. Then the second term in \eqref{equ::biasdecompensemble} can be upper bounded by
\begin{align*}
& \sum_{k=1}^L \sum_{l \neq k}  \bigl[ \mathbb{E}_{\mathrm{P}_H} \bigl( (T_k/T)   \mathfrak{f}_{\mathrm{P},1}^k(x) \bigr) - (T_k/T)  f_{L, \mathrm{P}}^*(x) \bigr]
\bigl[ \mathbb{E}_{\mathrm{P}_H} \bigl( (T_l/T)  f_{\mathrm{P},1}^l(x) \bigr) - (T_l/T)  f_{L, \mathrm{P}}^*(x) \bigr]
\\
& \leq c_L^2 \sum_{k=1}^L \sum_{l \neq k}  \bigl( (T_k/T)  h_k^{\alpha} \bigr) \cdot \bigl( (T_l/T)  h_l^{\alpha} \bigr).
\end{align*}
Consequently, our assumption $T_l h_l^{\alpha} \geq 4 c_1^{-1} c_L^2 L T_{l+1} h_{l+1}^{\alpha}$, $l\in [L-1]$, together with \eqref{equ::biasdecompensemble} and \eqref{eq::firstbias} yields $\mathbb{E}_{\mathrm{P}_H} \bigl( f_{\mathrm{P},\mathrm{E}}(x)- f_{L, \mathrm{P}}^*(x) \bigr)^2 \geq c_1/2 \sum_{l=1}^L (T_l/T)^2 h_l^{2\alpha}$. Therefore, there exist some probability distribution $\mathrm{P}$ in Assumption \ref{def::localholderP} such that for any $k \in [K]$, there holds
\begin{align*}
\mathbb{E}_{\mathrm{P}_H}  \bigl( f_{\mathrm{P},\mathrm{E}}(x)- f_{L, \mathrm{P}}^*(x) \bigr)^2 
\geq c_{1,k}/2\sum_{l=1}^L (T_l/T)^2 h_l^{2\alpha_k},
\qquad
x \in \Delta B_k.
\end{align*}
where $c_{1,k}$ are constants independent of $n$ and $B_{K+1} = \emptyset$. Thus, for any $x \in \mathcal{X}$, we have
\begin{align*}
\mathbb{E}_{\mathrm{P}_H}  \bigl( f_{\mathrm{P},\mathrm{E}}(x)- f_{L, \mathrm{P}}^*(x) \bigr)^2 
\geq C_3 \sum_{l=1}^L (T_l/T)^2 \sum_{k=1}^K h_l^{2\alpha_k} \eins_{\Delta B_k}(x),
\end{align*}
where $C_3 := \bigwedge_{k=1}^K c_{1,k}/2$. Taking expectation to $\mathrm{P}_X$ on both sides, we obtain
\begin{align*}
\mathbb{E}_{\mathrm{P}_H} \bigl( \mathcal{R}_{L, \mathrm{P}}(f_{\mathrm{P},\mathrm{E}})-\mathcal{R}_{L, \mathrm{P}}^* \bigr)
= \mathbb{E}_{\mathrm{P}_H} \mathbb{E}_{\mathrm{P}_X}|f_{\mathrm{P},\mathrm{E}}(X) - f^*_{L,\mathrm{P}}(X)|^2
\geq C_3 \sum_{l=1}^L (T_l/T)^2 
\sum_{k=1}^K \Delta m_k h_l^{2\alpha_k},
\end{align*}
which proves the assertion.
\end{proof}

\subsubsection{Proofs Related to Section \ref{sec::samplebag}}

\begin{proof}[Proof of  Proposition \ref{prop::samplebag}]
By the decompositions of $f_{\mathrm{D},\mathrm{E}}$ and $f_{\mathrm{P},\mathrm{E}}$ in \eqref{eq::decomppartensemble} and \eqref{eq::decomppartensembleP}, respectively, we have
\begin{align}
& \mathbb{E}_{\mathrm{P}^n}\mathbb{E}_{\mathrm{P}_X}|f_{\mathrm{P},\mathrm{E}}(X) - f_{\mathrm{D},\mathrm{E}}(X)|^2 
\nonumber\\
&= \mathbb{E}_{\mathrm{P}^n}\mathbb{E}_{\mathrm{P}_X} \biggl| \frac{1}{T} \sum_{t=1}^T (f_{\mathrm{P},\mathrm{E}}(X) - f_{\mathrm{D},\mathrm{E}}(X)) \biggr|^2 
=\mathbb{E}_{\mathrm{P}^n}\mathbb{E}_{\mathrm{P}_X} \biggl( \sum_{l=1}^L (T_l/T) \bigl( \mathfrak{f}_{\mathrm{P},\mathrm{E}}^l(X) - \mathfrak{f}_{\mathrm{D},\mathrm{E}}^l(X) \bigr) \biggr)^2
\nonumber\\
&=
\sum_{l=1}^L (T_l/T)^2 \mathbb{E}_{\mathrm{P}^n}\mathbb{E}_{\mathrm{P}_X} \bigl( \mathfrak{f}_{\mathrm{P},\mathrm{E}}^l(X) - \mathfrak{f}^l_{\mathrm{D},\mathrm{E}}(X) \bigr)^2 
\nonumber\\
& \phantom{=} 
+ \mathbb{E}_{\mathrm{P}^n}\mathbb{E}_{\mathrm{P}_X} \sum_{k=1}^L \sum_{l\neq k} \bigl[ (T_k/T) \bigl( \mathfrak{f}_{\mathrm{P},\mathrm{E}}^k(X) - \mathfrak{f}_{\mathrm{D},\mathrm{E}}^k(X) \bigr)  (T_l/T) \bigl(\mathfrak{f}_{\mathrm{P},\mathrm{E}}^l(X) - \mathfrak{f}_{\mathrm{D},\mathrm{E}}^l(X) \bigr) \bigr].
\label{eq::lowersamplediff}
\end{align}
For the first term in \eqref{eq::lowersamplediff}, since the base learners of $\mathfrak{f}^l_{\mathrm{D},\mathrm{E}}$ have the same bin width $h_l$, there holds
\begin{align}\label{eq::squareandmiddle}
&\mathbb{E}_{\mathrm{P}_H}\mathbb{E}_{\mathrm{P}^n}\mathbb{E}_{\mathrm{P}_X} \bigl( \mathfrak{f}_{\mathrm{P},\mathrm{E}}^l(X) - \mathfrak{f}^l_{\mathrm{D},\mathrm{E}}(X) \bigr)^2 
\nonumber\\
&= \frac{1}{T_l^2}\sum_{t=1}^{T_l} \mathbb{E}_{\mathrm{P}_H}\mathbb{E}_{\mathrm{P}^n}\mathbb{E}_{\mathrm{P}_X} (\mathfrak{f}_{\mathrm{P},t}^l(X) - \mathfrak{f}^l_{\mathrm{D},t}(X))^2
\nonumber\\
&\phantom{=} + \frac{1}{T_l^2}\sum_{t=1}^{T_l} \sum_{k\neq t} \mathbb{E}_{\mathrm{P}_H}\mathbb{E}_{\mathrm{P}^n}\mathbb{E}_{\mathrm{P}_X} (\mathfrak{f}_{\mathrm{P},k}^l(X) - \mathfrak{f}^l_{\mathrm{D},k}(X))(\mathfrak{f}_{\mathrm{P},t}^l(X) - \mathfrak{f}^l_{\mathrm{D},t}(X))
\nonumber\\
&= \frac{1}{T_l} \mathbb{E}_{\mathrm{P}_H}\mathbb{E}_{\mathrm{P}^n}\mathbb{E}_{\mathrm{P}_X} (\mathfrak{f}_{\mathrm{P},1}^l(X) - \mathfrak{f}^l_{\mathrm{D},1}(X))^2 + \frac{T_l-1}{T_l} \mathbb{E}_{\mathrm{P}^n}\mathbb{E}_{\mathrm{P}_X}
\big(\mathbb{E}_{\mathrm{P}_H} (\mathfrak{f}_{\mathrm{P},1}^l(X) - \mathfrak{f}^l_{\mathrm{D},1}(X))\big)^2.
\end{align}
For the first term in \eqref{eq::squareandmiddle}, combining \eqref{equ::term} and \eqref{eq::Zj>0}, we get
\begin{align*}
&\mathbb{E}_{\mathrm{P}_H}\mathbb{E}_{\mathrm{P}^n}\mathbb{E}_{\mathrm{P}_X} (\mathfrak{f}_{\mathrm{P},1}^l(X) - \mathfrak{f}^l_{\mathrm{D},1}(X))^2 
\nonumber\\
& =  n^{-1} \bigl( \sigma^2 + \mathbb{E}(f_{L,\mathrm{P}}^*(X) - f_{\mathrm{P},H}^*(X))^2 \bigr)\cdot
\sum_{j \in \mathcal{I}_H}  \bigl( \mathbb{E}(Z_j) \cdot \mathbb{E}(Z_j^{-1} | Z_j > 0) \bigr)\mathrm{P}(Z_j > 0)
\\
&\geq n^{-1}\sigma^2 \cdot
\sum_{j \in \mathcal{I}_H}  \bigl( \mathbb{E}(Z_j) \cdot \mathbb{E}(Z_j^{-1} | Z_j > 0) \bigr)\mathrm{P}(Z_j > 0).
\end{align*}
Using the binomial formula, we obtain
\begin{align*}
& \mathbb{E}(Z_j^{-1}|Z>0)\mathrm{P}(Z_j > 0)
= \sum_{l=1}^{n} \binom{n}{l} \big(\mathrm{P}(A_j)\big)^l \big(1-\mathrm{P}(A_j)\big)^{n-l} \frac{1}{l}
\\
&\geq \sum_{l=1}^{n} \binom{n}{l} \big(\mathrm{P}(A_j)\big)^l \big(1-\mathrm{P}(A_j)\big)^{n-l} \frac{1}{l+1}
= \frac{1}{n+1} \sum_{l=1}^{n} \binom{n+1}{l+1} \big(\mathrm{P}(A_j)\big)^l \big(1-\mathrm{P}(A_j)\big)^{n-l}
\\
&= \frac{1}{n+1} \sum_{l=2}^{n+1} \binom{n+1}{l} \big(\mathrm{P}(A_j)\big)^{l-1} \big(1-\mathrm{P}(A_j)\big)^{n-l+1}
\\
&= \frac{1}{(n+1)\mathrm{P}(A_j)} \sum_{l=2}^{n+1} \binom{n+1}{l} \big(\mathrm{P}(A_j)\big)^l \big(1-\mathrm{P}(A_j)\big)^{n+1-l}
\\
&= \frac{1}{(n+1)\mathrm{P}(A_j)} \Big(\sum_{l=0}^{n+1} \binom{n+1}{l} \big(\mathrm{P}(A_j)\big)^l \big(1-\mathrm{P}(A_j)\big)^{n+1-l} 
\\
&\qquad \qquad \qquad \qquad \qquad \qquad 
- \big(1-\mathrm{P}(A_j)\big)^{n+1} - (n+1)\mathrm{P}(A_j)\big(1-\mathrm{P}(A_j)\big)^n \Big)
\\
&= \frac{1}{(n+1)h^d}\big(1-(1 - h^d)^n (1+nh^d)\big)
\geq \frac{1}{(n+1)h^d}\big(1 - e^{-nh^d} (1+nh^d)\big),
\end{align*}
where the last inequality follows from the fact that $(1 - 1/x)^x \leq e^{-1}$, $x \geq 1$. Therefore, if $nh^d \geq 1$, we have 
$\mathbb{E}(Z_j^{-1}|Z>0)\mathrm{P}(Z_j > 0) \geq \frac{1}{8} n^{-1}h^{-d}$
and consequently we get
\begin{align}\label{eq::EPwai}
\frac{1}{T_l} \mathbb{E}_{\mathrm{P}_H}\mathbb{E}_{\mathrm{P}^n}\mathbb{E}_{\mathrm{P}_X} (\mathfrak{f}_{\mathrm{P},1}^l(X) - \mathfrak{f}^l_{\mathrm{D},1}(X))^2 \geq \frac{1}{8T_l} n^{-1}h^{-d} .
\end{align}

Next, we consider the second term of \eqref{eq::squareandmiddle}.  Without loss of generality, let $A_j$ be the cell containing the point $x$. Then we have  
\begin{align*}
&\mathbb{E}_{\mathrm{P}^n} \Bigl( \mathbb{E}_{\mathrm{P}_H} \big(\mathfrak{f}_{\mathrm{P},1}^l(x) - \mathfrak{f}^l_{\mathrm{D},1}(x)\big) \Bigr)^2
\\
&=\mathbb{E}_{\mathrm{P}^n} \bigg( \mathbb{E}_{\mathrm{P}_H}
\biggl(\mathbb{E}(f_{L,\mathrm{P}}^*(X)|A_x) - 
\frac{\sum_iY_i \eins_{A_j}(X_i)}{\sum_i\eins_{A_j}(X_i)}
\biggr)\bigg)^2
\\
&= \mathbb{E}_{\mathrm{P}^n} \bigg( \mathbb{E}_{\mathrm{P}_H}
\biggl(\mathbb{E}(f_{L,\mathrm{P}}^*(X)|A_j) 
- \frac{\sum_i f_{L,\mathrm{P}}^*(X_i) \eins_{A_j}(X_i)}{\sum_i\eins_{A_j}(X_i)}
+ \frac{\sum_i (f_{L,\mathrm{P}}^*(X_i) - Y_i) \eins_{A_j}(X_i)}{\sum_i\eins_{A_j}(X_i)}\bigg)
\biggr)^2
\\
&\geq \mathbb{E}_{\mathrm{P}^n} \biggl( \mathbb{E}_{\mathrm{P}_H}
\frac{\sum_i (f_{L,\mathrm{P}}^*(X_i) - Y_i) \eins_{A_j}(X_i)}{\sum_i\eins_{A_j}(X_i)} \biggr)^2 
\\
&\phantom{=} + 2\mathbb{E}_{\mathrm{P}^n}\biggl(
\mathbb{E}_{\mathrm{P}_H}
\biggl(\mathbb{E}(f_{L,\mathrm{P}}^*(X)|A_j) - \frac{\sum_i f_{L,\mathrm{P}}^*(X_i) \eins_{A_j}(X_i)}{\sum_i\eins_{A_j}(X_i)}\bigg) \cdot  \mathbb{E}_{\mathrm{P}_H} \frac{\sum_i (f_{L,\mathrm{P}}^*(X_i) - Y_i) \eins_{A_j}(X_i)}{\sum_i\eins_{A_j}(X_i)}\biggr).
\end{align*}
The linearity of the expectation operator implies 
\begin{align*}
\mathbb{E}_{\mathrm{P}_{Y|X}^n} \frac{\sum_i (f_{L,\mathrm{P}}^*(X_i) - Y_i) \eins_{A_j}(X_i)}{\sum_i\eins_{A_j}(X_i)} = 0
\end{align*}
and thus we have
\begin{align}\label{eq::onelower}
\mathbb{E}_{\mathrm{P}^n} \bigg( \mathbb{E}_{\mathrm{P}_H} \big(\mathfrak{f}_{\mathrm{P},1}^l(x) - \mathfrak{f}^l_{\mathrm{D},1}(x)\big) \bigg)^2
\geq \mathbb{E}_{\mathrm{P}^n} \biggl( \mathbb{E}_{\mathrm{P}_H}
\frac{\sum_i \big(Y_i - f_{L,\mathrm{P}}^*(X_i)\big) \eins_{A_j}(X_i)}{\sum_i\eins_{A_j}(X_i)}\bigg)^2.
\end{align}
Obviously, for any $i\neq k$, we have $\mathbb{E}_{\mathrm{P}_{Y|X}^n}\big(Y_i - f_{L,\mathrm{P}}^*(X_i)\big)\big(Y_k - f_{L,\mathrm{P}}^*(X_k)\big) = 0$ and for any $i\in[n]$, there holds $\mathbb{E}_{\mathrm{P}_{Y|X}^n} \big(Y_i - f_{L,\mathrm{P}}^*(X_i)\big)^2 = \sigma^2 > 0$. Therefore, we have 
\begin{align}\label{eq::epsilonequivalent}
& \mathbb{E}_{\mathrm{P}^n} \biggl( \mathbb{E}_{\mathrm{P}_H}
\frac{\sum_i \big(Y_i - f_{L,\mathrm{P}}^*(X_i)\big) \eins_{A_j}(X_i)}{\sum_i\eins_{A_j}(X_i)}\bigg)^2 
\nonumber\\
&= \mathbb{E}_{\mathrm{P}^n} \biggl( \sum_i\mathbb{E}_{\mathrm{P}_H} 
\frac{\big(Y_i - f_{L,\mathrm{P}}^*(X_i)\big) \eins_{A_j}(X_i)}{\sum_{k=1}^n\eins_{A_j}(X_k)}\bigg)^2
= \mathbb{E}_{\mathrm{P}^n} \sum_i \biggl( \mathbb{E}_{\mathrm{P}_H} 
\frac{\big(Y_i - f_{L,\mathrm{P}}^*(X_i)\big) \eins_{A_j}(X_i)}{\sum_{k=1}^n\eins_{A_j}(X_k)}\bigg)^2
\nonumber\\
&= n \mathbb{E}_{\mathrm{P}^n}\bigg(\mathbb{E}_{\mathrm{P}_{Y|X}^n} \big(Y_1 - f_{L,\mathrm{P}}^*(X_1)\big)^2 \cdot \bigg( \mathbb{E}_{\mathrm{P}_H} 
\frac{ \eins_{A_j}(X_1)}{\sum_{k=1}^n\eins_{A_j}(X_k)}\bigg)^2\bigg)
\nonumber\\
&= n \sigma^2 \mathbb{E}_{\mathrm{P}^n}\bigg( \mathbb{E}_{\mathrm{P}_H} 
\frac{ \eins_{A_j}(X_1)}{\sum_{k=1}^n\eins_{A_j}(X_k)}\bigg)^2.
\end{align}
For a fixed $H$, using the binomial formula, we get
\begin{align*}
\mathbb{E}_{\mathrm{P}^n} \biggl( \frac{\eins_{A_j}(X_1)}{\sum_{k=1}^n \eins_{A_j}(X_k)} \biggr)^2 
& = \mathrm{P}(X_1 \in A_j)  \mathbb{E} \Bigl( \biggl( \sum_{k=1}^n \eins_{A_j}(X_k) \biggr)^{-2} \bigg| X_1 \in A_j \biggr)
\\
& = h^d \sum_{l=0}^{n-1} \binom{n-1}{l} \mathrm{P}(A_j)^l \bigl( 1 - \mathrm{P}(A_j) \bigr)^{n-1-l} \frac{1}{(l+1)^2}
\\
&\geq \frac{h^d}{n(n+1)} \sum_{l=0}^{n-1} \binom{n+1}{l+2} \mathrm{P}(A_j)^l \bigl( 1 - \mathrm{P}(A_j) \bigr)^{n-1-l}
\\
& = \frac{h^d}{n(n+1)} \sum_{l=2}^{n+1} \binom{n+1}{l} \mathrm{P}(A_j)^{l-2} \bigl( 1 - \mathrm{P}(A_j) \bigr)^{n+1-l}
\\
& = \frac{1}{n(n+1) h^d} \biggl( \sum_{l=0}^{n+1} \binom{n+1}{l}  \mathrm{P}(A_j)^l \bigl( 1 - \mathrm{P}(A_j) \bigr)^{n+1-l} 
\\
& \qquad \qquad 
- \bigl( 1 - \mathrm{P}(A_j) \bigr)^{n+1} - (n+1) \mathrm{P}(A_j) (1-\mathrm{P}(A_j))^n \biggr)
\\
&= \frac{(1 - (1 - h^d)^n) (1 + n h^d)}{n(n+1) h^d} 
\geq \frac{1 - e^{-nh^d} (1+nh^d)}{n(n+1) h^d},
\end{align*}
where the last inequality follows from the fact that $(1 - 1/x)^x \leq e^{-1}$ for all $x \geq 1$. Therefore, if $nh^d \geq 1$, since the function $t \to 1-e^{-t}(1+t)$ is decreasing on the interval $(0, \infty)$, we have 
$\mathbb{E}_{\mathrm{P}^n} \bigl(
\eins_{A_j}(X_1) / \sum_{k=1}^n\eins_{A_j}(X_k) \bigr)^2 \geq (1/8)n^{-2}h^{-d}$.
This together with \eqref{eq::onelower} and \eqref{eq::epsilonequivalent} yields
\begin{align}\label{eq::EPnei}
\mathbb{E}_{\mathrm{P}^n} \bigl( \mathbb{E}_{\mathrm{P}_H} \bigl(\mathfrak{f}_{\mathrm{P},1}^l(x) - \mathfrak{f}^l_{\mathrm{D},1}(x)\bigr) \bigr)^2
\geq (\sigma^2/8) n^{-1} h^{-d}.
\end{align}
Combining \eqref{eq::EPwai}, \eqref{eq::EPnei} and \eqref{eq::squareandmiddle}, we obtain
\begin{align*}
\mathbb{E}_{\mathrm{P}_H}\mathbb{E}_{\mathrm{P}^n}\mathbb{E}_{\mathrm{P}_X} \bigl( \mathfrak{f}_{\mathrm{P},\mathrm{E}}^l(X) - \mathfrak{f}^l_{\mathrm{D},\mathrm{E}}(X) \bigr)^2
\geq  ((\sigma^2 \wedge 1)/8) n^{-1}h^{-d}
\end{align*}
and consequently 
\begin{align}\label{eq::squarelower}
\sum_{l=1}^L (T_l/T)^2 \mathbb{E}_{\mathrm{P}^n}\mathbb{E}_{\mathrm{P}_X} \bigl( \mathfrak{f}_{\mathrm{P},\mathrm{E}}^l(X) - \mathfrak{f}^l_{\mathrm{D},\mathrm{E}}(X) \bigr)^2 \geq \frac{\sigma^2 \wedge 1}{8} \sum_{l=1}^L (T_l/T)^2 n^{-1}h^{-d},
\end{align}
which gives the lower bound of  the first term in \eqref{eq::lowersamplediff}.

On the other hand, using the triangle inequality and the Cauchy-Schwarz inequality, the second term in \eqref{eq::lowersamplediff} can be upper bounded by 
\begin{align}\label{eq::midtermdecomp}
& \biggl| \mathbb{E}_{\mathrm{P}^n}\mathbb{E}_{\mathrm{P}_X} \sum_{k=1}^K \sum_{l\neq k}  \bigl[ (T_1/T) \bigl( \mathfrak{f}_{\mathrm{P},\mathrm{E}}^k(X) - \mathfrak{f}_{\mathrm{D},\mathrm{E}}^k(X) \bigr)  (T_l/T) \bigl( f_{\mathrm{P},\mathrm{E}}^l(X) - f_{\mathrm{D},\mathrm{E}}^l(X) \bigr) \bigr] \biggr|
\nonumber\\
&\leq \sum_{k=1}^K\sum_{l\neq k} \bigl|  \mathbb{E}_{\mathrm{P}^n} \mathbb{E}_{\mathrm{P}_X} \bigl[ (T_k/T) \bigl( \mathfrak{f}_{\mathrm{P},\mathrm{E}}^k(X) - \mathfrak{f}_{\mathrm{D},\mathrm{E}}^k(X) \bigr)  (T_l/T) \bigl( f_{\mathrm{P},\mathrm{E}}^l(X) - f_{\mathrm{D},\mathrm{E}}^l(X) \bigr) \bigr] \bigr|
\nonumber\\
&\leq \sum_{k=1}^K\sum_{l\neq k} \bigl[ \mathbb{E}_{\mathrm{P}^n} \mathbb{E}_{\mathrm{P}_X} \bigl[ (T_k/T) \bigl( \mathfrak{f}_{\mathrm{P},\mathrm{E}}^k(X) - \mathfrak{f}_{\mathrm{D},\mathrm{E}}^k(X) \bigr) \bigr]^2 \bigr]^{\frac{1}{2}}
\nonumber\\
& \qquad \qquad \qquad 
\cdot \bigl[ \mathbb{E}_{\mathrm{P}^n} \mathbb{E}_{\mathrm{P}_X} \bigl[ (T_l/T) \bigl( f_{\mathrm{P},\mathrm{E}}^l(X) - f_{\mathrm{D},\mathrm{E}}^l(X) \bigr) \bigr]^2 \bigr]^{\frac{1}{2}}
\nonumber\\
&\leq 18 M^2 \sum_{k=1}^K\sum_{l \neq k}\big((T_k/T)(T_l/T)(h_l h_k)^{-d}\big)^{1/2} n^{-1},
\end{align}
where the last inequality follows from Proposition \ref{prop::samplesingleexp}. Then our assumption $T_l h_l^{-d} \geq 512M^2 L (\sigma^2 \wedge 1)^{-1} T_{l+1} h_{l+1}^{-d}$, $l \in [L-1]$, together with \eqref{eq::lowersamplediff}, \eqref{eq::squarelower} and \eqref{eq::midtermdecomp}, yields
\begin{align*}
\mathbb{E}_{\mathrm{P}_H}\mathbb{E}_{\mathrm{P}^n}\mathbb{E}_{\mathrm{P}_X}|f_{\mathrm{P},\mathrm{E}}(X) - f_{\mathrm{D},\mathrm{E}}(X)|^2  \geq \frac{\sigma^2 \wedge 1}{16} \sum_{l=1}^L (T_l/T)^2 n^{-1} h_l^{-d},
\end{align*}
which proves the assertion with $C_4 :=(\sigma^2 \wedge 1) / 16$.
\end{proof}

\subsubsection{Proofs Related to Section \ref{sec::lowerpeht}}

\begin{proof}[Proof of  Theorem \ref{thm::lowerbag}]
Combining Propositions \ref{prop::approxbag} and \ref{prop::samplebag}, we obtain 
\begin{align}\label{eq::ensembleerror}
& \mathbb{E}_{\mathrm{P}_H} \mathbb{E}_{\mathrm{P}^n} \mathcal{R}_{L,\mathrm{P}}(f_{\mathrm{D},\mathrm{E}}) - \mathcal{R}_{L,\mathrm{P}}^* 
\nonumber\\
&= \mathbb{E}_{\mathrm{P}_H} \big( \mathcal{R}_{L,\mathrm{P}}(f_{\mathrm{P},\mathrm{E}}) - \mathcal{R}_{L,\mathrm{P}}^*\big) + \mathbb{E}_{\mathrm{P}_H} \mathbb{E}_{\mathrm{P}^n} \mathbb{E}_{\mathrm{P}_X} |f_{\mathrm{P},\mathrm{E}}(X) - f_{\mathrm{D},\mathrm{E}}(X)|^2  
\nonumber\\
& \geq c_1\bigg(\sum_{l=1}^L (T_l/T)^2 n^{-1} h_l^{-d}
+ \sum_{l=1}^L (T_l/T)^2 
\sum_{k=1}^K \Delta m_k h_l^{2\alpha_k} \bigg)
\nonumber\\
& = c_1\sum_{l=1}^L  (T_l/T)^2 \biggl( n^{-1} h_l^{-d} +  \sum_{k=1}^K \Delta m_k h_l^{2\alpha_k} \biggr),
\end{align}
where $c_1 := C_4 \wedge C_3$ with constants $C_3$ and $C_4$ defined as in Propositions \ref{prop::approxbag} and \ref{prop::samplebag}, respectively. Let $h_*$ be the bandwidth which minimizes $n^{-1} h^{-d} +  \sum_{k=1}^K \Delta m_k h^{2\alpha_k}$. Using Cauchy-Schwarz inequality and $\sum_{l=1}^{L}T_l = T$, we have $\sum_{l=1}^L  T_l^2 \geq \frac{1}{L}  \bigl( \sum_{l=1}^L  T_l \bigr)^2 = T^2 / L$. Consequently, we get
\begin{align*}
\inf_{l \in [L]} \sum_{l=1}^L  (T_l/T)^2 \biggl( n^{-1} h_l^{-d} +  \sum_{k=1}^K \Delta m_k h_l^{2\alpha_k} \biggr) 
&\geq \inf_{l \in [L]} \sum_{l=1}^L  (T_l/T)^2 
\inf_{l \in [L]}  \biggl( n^{-1} h_l^{-d} +  \sum_{k=1}^K \Delta m_k h_l^{2\alpha_k} \biggr)
\\
&\geq \frac{1}{L} \inf_{h} \biggl( n^{-1} h^{-d} +  \sum_{k=1}^K \Delta m_k h^{2\alpha_k} \biggr).
\end{align*}
This together with \eqref{eq::ensembleerror} yields 
\begin{align}\label{eq::errorh1}
\inf_{f_{\mathrm{D},\mathrm{E}}} \sup_{\mathrm{P} \in \mathcal{P}} \mathbb{E}_{\mathrm{P}_H} \mathbb{E}_{\mathrm{P}^n} \mathcal{R}_{L,\mathrm{P}}(f_{\mathrm{D},\mathrm{E}}) - \mathcal{R}_{L,\mathrm{P}}^* 
= \frac{c_1}{L} \inf_{h} \biggl( n^{-1} h^{-d} +  \sum_{k=1}^K \Delta m_k h^{2\alpha_k} \biggr),
\end{align}
which yields the assertion with $c_E := c_1/L$. 
\end{proof}

\subsection{Proofs Related to Section \ref{sec::thmcompare}}

\begin{proof}[Proof of  Theorem \ref{thm::finiten}]
Let us first consider the excess risk of PHBT. For the lower bound in the right hand side of \eqref{eq::LBbag}, we have 
\begin{align*}
\inf_{h} \biggl( n^{-1} h^{-d} + \sum_{k=1}^K \Delta m_k  h^{2\alpha_k} \biggr) 
& \geq \inf_{h} \biggl( n^{-1} h^{-d} + \bigvee_{k=1}^K \Delta m_k h^{2\alpha_k} \biggr) 
\\
& \geq \bigvee_{k=1}^K \inf_{h} \bigl( n^{-1} h^{-d} + \Delta m_k h^{2\alpha_k} \bigr).
\end{align*}
By taking $h_* := \bigl( n \Delta m_k \bigr)^{-1/(2\alpha_k+d)}$ and $T_1 = n^0$, we obtain
\begin{align*}
\inf_{h} \bigl( n^{-1} h^{-d} +  \Delta m_k h^{2\alpha_k} \bigr) 
= \Delta m_k^{\frac{d}{2\alpha_k+d}} n^{-\frac{2\alpha_k}{2\alpha_k+d}}.
\end{align*}
This together with \eqref{eq::errorh1} implies
\begin{align}\label{eq::risksum}
\mathbb{E}_{\mathrm{P}_H} \mathbb{E}_{\mathrm{P}^n} \mathcal{R}_{L,\mathrm{P}}(f_{\mathrm{D},\mathrm{E}}) - \mathcal{R}_{L,\mathrm{P}}^* 
\geq c_E \bigvee_{k=1}^K \Delta m_k^{\frac{d}{2\alpha_k+d}} n^{-\frac{2\alpha_k}{2\alpha_k+d}}
= c_E \Delta m_{k'}^{\frac{d}{2\alpha_{k'}+d}} n^{-\frac{2\alpha_{k'}}{2\alpha_{k'}+d}},
\end{align}
where $k' = \argmax_{k\in[K]} \Delta m_k^{d/(2\alpha_k+d)} n^{-2\alpha_k/(2\alpha_k+d)}$, which implies
\begin{align}\label{eq::maxcondition}
\Delta m_{k'} = \bigvee_{k=1}^K n^{\frac{2\alpha_{k'}-2\alpha_k}{2\alpha_k+d}} \Delta m_k^{\frac{2\alpha_{k'}+d}{2\alpha_k+d}}.
\end{align}
Combining \eqref{eq::risksum} and \eqref{eq::maxcondition}, we obtain
\begin{align}\label{eq::riskdominant}
\mathbb{E}_{\mathrm{P}_H} \mathbb{E}_{\mathrm{P}^n} \mathcal{R}_{L,\mathrm{P}}(f_{\mathrm{D},\mathrm{E}}) - \mathcal{R}_{L,\mathrm{P}}^* \geq c_E \Delta m_{k'}^{\frac{d}{2\alpha_{k'}+d}} n^{-\frac{2\alpha_{k'}}{2\alpha_{k'}+d}}.
\end{align}

Next, let us consider the excess risk of ABHT. Let $k^* \in [K]$ be defined as in \eqref{eq::kstar}. By Theorem \ref{thm::upperboost}, we have
\begin{align}\label{eq::orderboost}
\mathbb{E}_{\mathrm{P}_H} \bigl(  \mathcal{R}_{L,\mathrm{P}}(\mathfrak{f}_{\mathrm{D},\mathrm{B}}) - \mathcal{R}_{L,\mathrm{P}}^*\bigr)
\leq c_B \sum_{k=1}^{K} \Delta m_k n^{-\frac{2\alpha_k - \delta d/(1+\delta)}{(2+2\delta)\alpha_k+d}} 
\leq c_B K \Delta m_{k^*}  n^{-\frac{2\alpha_{k^*} - \delta d/(1+\delta)}{(2+2\delta)\alpha_{k^*}+d}}
\end{align}
with probability $\mathrm{P}^n$ at least $1 - 3K/n$. 
It is easy to verify that for $N(\delta)$ satisfying \eqref{eq::n0}, we have 
\begin{align*}
\Delta m_{k^*}^{-1} 
\cdot (Kc_B/c_E)^{-\frac{2\alpha_{k^*}+d}{2\alpha_{k^*}}} = 
N(\delta)^{\frac{10d^2\delta/\alpha_{k^*}}{2\alpha_{k^*}+d}}.
\end{align*}
Consequently, for any $n \leq N(\delta)$, there holds
\begin{align*}
\Delta m_{k^*}^{-1} 
= (Kc_B/c_E)^{\frac{2\alpha_{k^*}+d}{2\alpha_{k^*}}} N(\delta)^{\frac{10d^2\delta/\alpha_{k^*}}{2\alpha_{k^*}+d}}
\geq (Kc_B/c_E)^{\frac{2\alpha_{k^*}+d}{2\alpha_{k^*}}}
n^{\frac{10d^2\delta/\alpha_{k^*}}{2\alpha_{k^*}+d}},
\end{align*}
which is equivalent to 
\begin{align*}
\Delta m_{k^*} \leq (Kc_Bc_E^{-1})^{-\frac{2\alpha_{k^*}+d}{2\alpha_{k^*}}}
n^{-\frac{10d^2\delta/\alpha_{k^*}}{2\alpha_{k^*}+d}}.
\end{align*}
Since $\alpha_k \leq 1$, $k \in [K]$, and $d \geq 1$, some simple calculations yield
\begin{align*}
n^{\frac{2\alpha_{k'}-2\alpha_{k^*}}{2\alpha_{k^*}+d}} \Delta m_{k^*}^{\frac{2\alpha_{k'}+d}{2\alpha_{k^*}+d}} 
\geq \Bigl( Kc_Bc_E^{-1} n^{\frac{10d^2\delta}{(2\alpha_{k^*}+d)^2}} \Delta m_{k^*} \Bigr)^{\frac{2\alpha_{k'}+d}{d}} n^{-\frac{2\alpha_{k^*}-2\alpha_{k'} - (4\alpha_{k'}\alpha_{k^*}/d+2\alpha_{k'}+d) \delta}{2\alpha_{k^*}+d}}.
\end{align*}
This together with \eqref{eq::maxcondition} implies
\begin{align*}
\Delta m_{k'} \geq \Bigl( Kc_Bc_E^{-1} n^{\frac{10d^2\delta}{(2\alpha_{k^*}+d)^2}} \Delta m_{k^*} \Bigr)^{\frac{2\alpha_{k'}+d}{d}} n^{-\frac{2\alpha_{k^*}-2\alpha_{k'} - 4\alpha_{k'}\alpha_{k^*}\delta/d
	-(2\alpha_{k'}+d)\delta /(1+\delta)}{2\alpha_{k^*}+d}},
\end{align*}
which is equivalent to
\begin{align*}
n^{\frac{10d^2\delta}{(2\alpha_{k^*}+d)^2}} c_B K\Delta m_{k^*}  n^{-\frac{2\alpha_{k^*} -\delta d/(1+\delta)
}{(2+2\delta)\alpha_{k^*}+d}} 
\leq c_E  \Delta m_{k'}^{\frac{d}{2\alpha_{k'}+d}} n^{-\frac{2\alpha_{k'}}{2\alpha_{k'}+d}}.
\end{align*}
This together with \eqref{eq::orderboost} and \eqref{eq::riskdominant} yields
the assertion.
\end{proof}

\section{Conclusion}\label{sec::Conclusion}

In this paper, we propose an adaptive boosting algorithm with the histogram transforms as base learners, called \textit{adaptive boosting histogram transform} (\textit{ABHT}). By assuming that the target function lies in a locally H\"{o}lder continuous space, we prove that ABHT can well recognize the regions with different local H\"{o}lder exponents. This enables us to prove that the ABHT converges strictly faster than PEHT, a parallel ensemble of histogram transforms, by comparing the upper bound for the excess risk of ABHT and the lower bound for that of PEHT. Moreover, we conduct numerical experiments to further verify the theoretical results. 

The study in this paper is originally motivated by pursuing some further understanding of the advantages of sequential learning algorithms \cite{cai2020boosted} over parallel learning algorithms \cite{hang2021histogram}. It turns out that the study conducted in this paper brings us some new theoretical perspectives and a deeper understanding of the sequential learning algorithm in terms of the adaptivity under local smoothness assumption. Our theory has the potential of distinguishing a broad variety of locally adaptive algorithms, from the perspective of fitting locally smooth target functions. For example, with similar arguments, we could show the advantage of gradient boosting over other algorithms such as support vector regressors (SVR) which cannot be adaptive to locally smooth functions.

\bibliographystyle{plain}
\small{\bibliography{ABHTR}}
\end{document}